\documentclass{article} 
\usepackage{iclr2027_conference,times}

\usepackage{ifxetex}\ifxetex\usepackage{fontspec}\fi

\usepackage{amsmath,amsfonts,bm}

\def\eqref#1{equation~\ref{#1}}

\def\1{\bm{1}}

\DeclareMathAlphabet{\mathsfit}{\encodingdefault}{\sfdefault}{m}{sl}
\SetMathAlphabet{\mathsfit}{bold}{\encodingdefault}{\sfdefault}{bx}{n}

\usepackage{hyperref}
\usepackage{url}
\usepackage{booktabs}
\usepackage{graphicx}
\usepackage{amsmath,amssymb,amsthm}
\usepackage{multirow}
\usepackage{xcolor}

\newtheorem{proposition}{Proposition}
\newtheorem{lemma}{Lemma}
\newtheorem{observation}{Observation}

\newcommand{\Vth}{V_\theta}
\newcommand{\Vbar}{\bar{V}}
\newcommand{\Dpref}{\mathcal{D}_{\mathrm{pref}}}
\newcommand{\Doff}{\mathcal{D}}
\newcommand{\segA}{\sigma^{+}}
\newcommand{\segB}{\sigma^{-}}


\title{Certified Safety Curation: Distribution-Free Guarantees for Safe Offline Reinforcement Learning}

\author{Adam Haroon\thanks{Corresponding author: \texttt{aharoon@iastate.edu}}, \ Cody Fleming \\
Iowa State University\\
Ames, IA, USA\\
\texttt{\{aharoon,flemingc\}@iastate.edu}}

\iclrfinalcopy

\begin{document}

\maketitle
\lhead{Preprint. Under review.}

\begin{abstract}
Safe offline reinforcement learning assumes a cost function on every transition. We ask what remains possible when safety can be judged only by comparing short clips and occasionally asking whether an episode exceeded its budget. Certified safety curation answers with a filter-then-clone pipeline: a state-only value trained from segment comparisons scores whole trajectories, Learn-then-Test calibration certifies a selection threshold under a distribution-free $(\alpha, \delta)$ bound on the unsafe fraction of the selection, and behavior cloning follows. What is certified is the training set, not the policy, whose cost we report rather than bound. Where no threshold attains the target, the procedure refuses. We are not aware of prior work certifying the composition of a training set for offline RL or imitation. Oracle controls justify the design: reweighting individual transitions fails even with an exact value, so the value selects whole trajectories. The policies satisfy the cost budget on twelve of fifteen DSRL tasks, matching a clone of the ground-truth safe subset, which needs a label on every trajectory. Retrained on the certified selection, the strongest full-label method becomes safe where no setting of its own cost target rescues it. Refusal is predictable: the certificate's probability has a closed form in the purity of the pool's top quantile, which the calibration sample estimates and through which the scorer enters.
\end{abstract}

\section{Introduction}
\label{sec:intro}

Offline reinforcement learning promises policy improvement from logged data alone; its safety-constrained variant additionally requires the learned policy to keep expected cost below a budget \citep{liu2024dsrl}. Existing safe offline RL methods require a cost function evaluated on every transition of the dataset \citep{liu2023cdt, zheng2024fisor, gong2025trac}. On benchmarks it is given; the difficulty is upstream, in writing that function down, which is the problem reward engineering has proven to be and why preference-based learning exists \citep{christiano2017deep}. Where no trustworthy cost function can be specified, a practitioner may still be able to say which of two short clips looks safer, and to judge occasionally whether an episode went over budget. This paper asks what safe offline RL can still deliver from those two signals alone. Recent work has begun importing preferences into safe offline RL \citep{gong2026presa, burnwal2026osil}. None of it, however, carries a statistical guarantee.

\begin{figure}[t]
\begin{center}
\includegraphics[width=\linewidth]{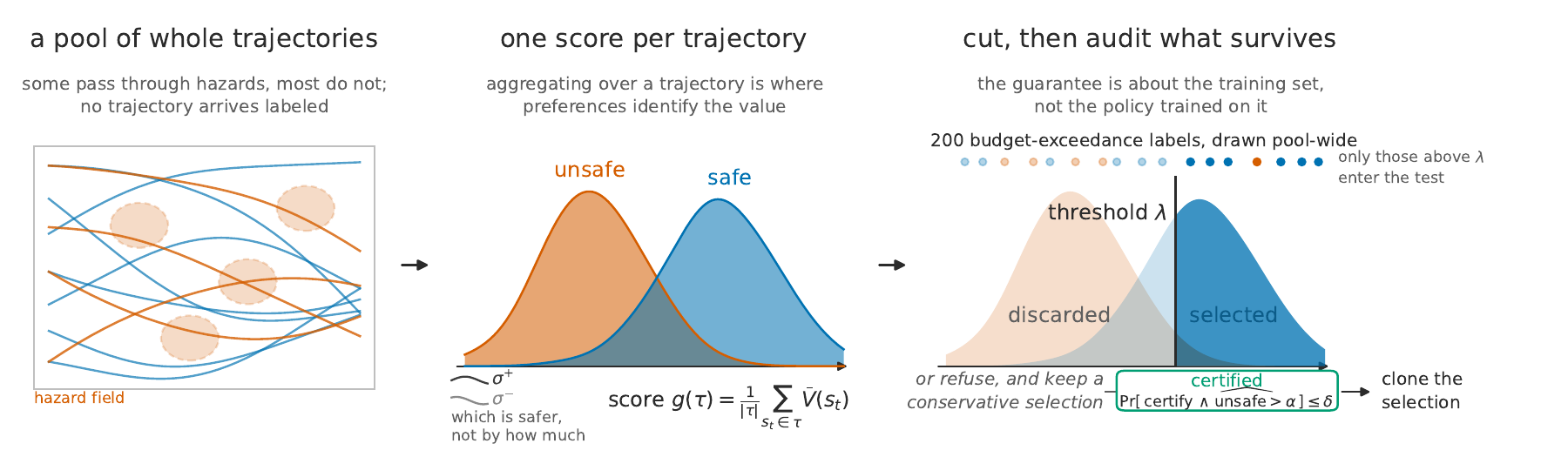}
\end{center}
\caption{Certified safety curation. A state-only value trained from segment
comparisons, which carry an ordering and no magnitudes, scores whole
trajectories. A threshold cuts the score axis and $200$ episodic budget labels
audit what survives, certifying the selection at $(\alpha, \delta)$ or refusing
and returning a conservative selection reported as uncertified; either way the
selection is cloned, and a certified selection may first be curated for reward
(Section~\ref{sec:gating}). Paths and densities are illustrative, not measured.}
\label{fig:concept}
\end{figure}

The question we study is whether a \emph{state-only} safety value, one that never sees actions and is trained only from comparisons, can identify safe training data well enough to be certified. Our answer is certified safety curation (Figure~\ref{fig:concept}), a filter-then-clone method whose deliverable is a statistical guarantee on the data it trains on: the value scores whole trajectories, a threshold selects, and the threshold is chosen by Learn-then-Test calibration \citep{angelopoulos2025ltt} against a small labeled sample. With probability at least $1-\delta$ over that draw, the unsafe-trajectory fraction of the selected training set is at most $\alpha$. The guarantee consumes 200 budget-exceedance judgments, one bit per sampled trajectory, where competing methods require a cost value on every transition.

Our contributions are three. First, a granularity analysis with oracle controls establishing that the readout scale, not the supervision, is the binding choice (Section~\ref{sec:analysis}); the intermediate segment scale fails even with full labels, BC-Safe-Seg being safe on two of fifteen tasks against twelve at trajectory scale (Section~\ref{sec:results}). Second, certified curation: Learn-then-Test calibration of the selection threshold against 200 budget-exceedance labels, with an $(\alpha, \delta)$ guarantee on training-set composition that refuses wherever the score cannot concentrate safety below $\alpha$, and whose certification probability has a closed form the calibration sample estimates (Section~\ref{sec:calibration-method}). Third, a measured account of what the supervision buys. CDT violates on eight Goal and Button tasks at every cost target down to a fifth of the budget, yet is safe on the certified selection. A classifier trained on the calibration labels alone comes within one safe task of the calibrated filter, eleven against twelve at matched budget, and replacing its bits with numeric costs does not improve it. What separates methods is certification, not scorer accuracy or supervision richness (Section~\ref{sec:experiments}).

\section{Related Work}
\label{sec:related}

\paragraph{Filtered and weighted behavior cloning.}
Selecting or reweighting data before cloning is well established for reward: percentile behavior cloning ranks trajectories by return \citep{chen2021dt, emmons2022rvs}, BAIL selects transitions near a return envelope \citep{chen2020bail}, CRR and AWR weight transitions by learned advantages \citep{wang2020crr, peng2019awr}, and concurrent work scores demonstrations with learned feedback models before cloning \citep{liang2026robometer, beck2025sfbc}. This line targets task success; none addresses a safety constraint, and every threshold in it is hand-tuned.

\paragraph{Safe offline RL.}
Methods with full cost labels include constrained sequence modeling \citep{liu2023cdt}, feasibility-guided diffusion \citep{zheng2024fisor}, and trajectory classification \citep{gong2025trac}. The DSRL benchmark's BC-Safe baseline clones trajectories under the cost budget, a label on every trajectory \citep{liu2024dsrl}. Preference and weak-label variants are emerging: PReSa consumes binary segment safety labels inside a Lagrangian objective \citep{gong2026presa}; OSIL learns a per-step cost model from dataset-level preferences and applies it as a cloning penalty \citep{burnwal2026osil}; PREFINE fine-tunes implicit reward and cost models from preferences for safety alignment \citep{prefine2025}. None selects training data with a preference-learned value, and none carries a statistical guarantee.

\paragraph{Distribution-free risk control.}
Learn-then-Test and risk-controlling prediction sets calibrate model configurations with finite-sample guarantees \citep{angelopoulos2025ltt, bates2021rcps, angelopoulos2024crc}. In RL and control, published applications attach to a policy: conformal off-policy intervals \citep{taufiq2022copp}, runtime monitors \citep{feldman2025conformal}, and high-confidence policy improvement \citep{thomas2015hcpi, laroche2019spibb}. The nearest statistical relative is conformal selection \citep{jin2023conformal}, which screens candidate units under false discovery rate control; FDR is an expectation bound, weaker than the $(\alpha, \delta)$ high-probability bound we calibrate. Our calibration takes Learn-then-Test's fixed sequence but supplies an exact finite-population p-value, proved super-uniform here (Lemma~\ref{lem:superunif}), and a closed form for how often it certifies, which validity alone does not give (Proposition~\ref{prop:rate}); the contribution is that and the object certified, the composition of a training set.

\section{Preliminaries}
\label{sec:prelim}

A constrained MDP $(\mathcal{S}, \mathcal{A}, P, r, c, b)$ maximizes return subject to $J_c(\pi) \le b$; the offline setting provides a static dataset $\Doff = \{\tau_i\}_{i=1}^{N}$. We call a trajectory unsafe when its episodic cost, the sum of its per-step
costs, exceeds the budget and write $u(\tau) \in \{0, 1\}$ for this indicator. A constrained MDP supplies $c$ by definition and benchmarks apply it to every transition; our method never reads it, using two weaker signals instead. The first is a preference dataset $\Dpref$ of pairs of length-$H$ segments $(\segA, \segB)$, each labeled only with which is safer, an ordering rather than a magnitude. The second, used only for calibration, is a \emph{cost label}: the binary indicator $u(\tau)$ on each of $n$ trajectories drawn uniformly without replacement from $\Doff$, which makes the finite-population test of Section~\ref{sec:calibration-method} exact. We call whoever supplies these signals the \emph{labeler}; in our experiments it is the benchmark's own cost function, which lets us study the algorithm without confounding it with labeler error. We report raw reward and episodic cost against each task's budget; a policy is safe when its expected cost is within budget (Appendix~\ref{app:baselines}).

\section{Method: Certified Safety Curation}
\label{sec:method}

The method has four stages: value learning, trajectory scoring, calibrated selection, and cloning; pseudocode is in Appendix~\ref{app:algo}.

\subsection{Safety value from segment preferences}
\label{sec:value}
We learn a state-only value $\Vth: \mathcal{S} \to \mathbb{R}$ by maximum likelihood under a Bradley-Terry model \citep{bradley1952rank} over segment sums,
\begin{equation}
\label{eq:bt}
\mathcal{L}(\theta) = -\mathbb{E}_{(\segA, \segB) \sim \Dpref}\Big[\log \mathrm{logistic}\Big(\textstyle\sum_{t=1}^{H} \Vth(s_t^{+}) - \sum_{t=1}^{H} \Vth(s_t^{-})\Big)\Big],
\end{equation}
where $\mathrm{logistic}(x) = 1/(1+e^{-x})$. We train an ensemble of $K$ independently initialized networks and write $\Vbar$ for their mean (Appendix~\ref{app:hyper} fixes $K$ and reports the sensitivity). The ensemble is not what makes selection work, and single members match its filter precision to within $0.006$ (Appendix~\ref{app:hyper}); we keep it because member disagreement is the diagnostic Section~\ref{sec:analysis} uses to locate where the value is and is not identified. Equation~\ref{eq:bt} supervises only segment sums: any perturbation whose sums agree across every preference pair is invisible to the loss, so per-state values are underdetermined and aggregation is the value's reliable regime, which Section~\ref{sec:analysis} quantifies.

\subsection{Trajectory scoring and selection}
\label{sec:scoring}
Each trajectory receives the score $g(\tau) = \frac{1}{|\tau|} \sum_{s_t \in \tau} \Vbar(s_t)$, and the selection at threshold $\lambda^{*}$ is $S(\lambda^{*}) = \{\tau_i : g(\tau_i) \ge \lambda^{*}\}$. The policy is obtained by behavior cloning on $S(\lambda^{*})$. Averaging over hundreds of states cancels the per-state noise that Equation~\ref{eq:bt} leaves unconstrained, and selection at the trajectory level preserves whole multi-step avoidance behaviors that per-transition reweighting fragments.

\subsection{Calibrated threshold with a composition guarantee}
\label{sec:calibration-method}
The threshold is the method's only free parameter, selected by Learn-then-Test (LTT) \citep{angelopoulos2025ltt}. Let $\mathcal{C}$ be a calibration sample of $n$ trajectories drawn uniformly without replacement from $\Doff$, each carrying its binary label $u(\tau)$. For a decreasing grid of candidate thresholds $\lambda_1 > \lambda_2 > \cdots$ (score quantiles from $0.85$ down to $0.30$), let $N_j = |S(\lambda_j)|$ be the selection size, and let $m_j$ and $k_j$ be the number of calibration trajectories above $\lambda_j$ and the number of those that are unsafe. Because the pool is finite and the selection sizes are known, the null hypothesis that the selection contains more than $\alpha N_j$ unsafe trajectories admits an exact test: under the least-favorable null the unsafe calibration count is hypergeometric, and $p_j = F_{\mathrm{hyp}}\big(k_j;\, N_j,\, \lfloor \alpha N_j \rfloor + 1,\, m_j\big)$ is a super-uniform p-value, where $F_{\mathrm{hyp}}$ is the hypergeometric CDF. Strict fixed-sequence testing tests the hypotheses in order, each at level $\delta$, stops at the first failure to reject, and returns the last rejected threshold; if the first test fails, no threshold is certified.

\begin{proposition}
\label{prop:guarantee}
The returned threshold $\hat\lambda$ satisfies
\begin{equation}
\label{eq:guarantee}
\Pr\Big[\,\textnormal{a threshold is certified}\ \wedge\ \mathrm{unsafe}\big(S(\hat\lambda)\big) > \alpha\,\Big] \;\le\; \delta,
\end{equation}
where $\mathrm{unsafe}(S) = \frac{1}{|S|}\sum_{\tau \in S} u(\tau)$ is the realized unsafe fraction of the selection under the true labels on $\Doff$, and the probability is over the uniform calibration draw. A proof is given in Appendix~\ref{app:proof}.
\end{proposition}

Validity needs no assumption that the unsafe fraction decreases as the threshold rises: a true null is never tested after a rejection, and monotonicity affects only \emph{power}, that is, how deep into the grid the walk can reach before it stops. The grid and the score are fixed before the calibration sample is drawn. The guarantee is transductive over the finite pool and holds per calibration run, one pass of the Appendix~\ref{app:algo}
procedure, with no joint claim across tasks or seeds. It is also unconditional: it bounds the joint probability of certifying and violating, not the violation probability given a certificate, which can exceed $\delta$ where certification is rare (Appendix~\ref{app:proof}, Section~\ref{sec:guarantee-validation}). The grid starts at the $0.85$ quantile because deeper thresholds are statistically uncertifiable at practical $n$ and sit in the score tail the Bradley-Terry null space leaves unidentified (Section~\ref{sec:analysis}); when no threshold certifies, the procedure returns an explicitly uncertified fallback: the top fraction given by a one-sided Clopper-Pearson lower bound on pool safe mass, from the calibration labels already spent.

\subsection{Certification-gated reward-aware curation}
\label{sec:gating}
A safety-only filter discards the dataset's reward information, so within a selection we optionally retain the top half by episodic return before cloning. Because return and cost correlate on contaminated selections, this is safe exactly when the selection is clean, so we gate it on the certificate, which is both a guarantee on composition and a decision signal for how aggressively to curate. The hard rule is the extreme point of a family we also instantiate softly: \emph{certified return-weighted cloning} resamples the certified selection with probability proportional to $\exp(\mathrm{clip}(z(R_i), \pm \kappa))$, where $z(R_i)$ standardizes episodic return within the selection and the clip holds it in $[-\kappa, \kappa]$, setting the return-seeking intensity. This is the operator the granularity analysis prescribes: the exponentiated weighting of advantage-weighted regression, moved to the trajectory scale where selection composes, meaning the selection can be handed to a different learner and still keep it within budget, and run only behind the certificate. Its safety on certified selections is robust to $\kappa$, and beyond the certified composition it is not (Section~\ref{sec:results}, Appendix~\ref{app:composea40}).

\section{Why Trajectory Scale: an Oracle-Controlled Analysis}
\label{sec:analysis}

The algorithm of Appendix~\ref{app:algo} clones a \emph{selection}: the value picks trajectories and behavior cloning does the rest. The standard alternative reweights individual transitions instead; everything analyzed in this section is that alternative, not a component of our method. It is advantage-weighted regression \citep{peng2019awr}: form a per-transition advantage from the value, weight each transition by $\exp(A/\beta)$, and fit a policy to the reweighted data. The value is a scoring signal in both cases; only the granularity of its use differs. The analysis uses nine DSRL tasks \citep{liu2024dsrl}, the five velocity and four Goal-family tasks, called the \emph{analysis tasks} throughout; Section~\ref{sec:experiments} evaluates the full fifteen. Labels are exact, isolating algorithmic questions from labeler noise, and segment-level labeling makes the value's supervision genuinely sub-trajectory (Appendix~\ref{app:hyper}).

The analysis isolates one obstacle, which concerns how the value is used and holds even when the value is exact, and separates it from an identification caveat. Throughout, \emph{per-transition} use means weighting each state-action pair by its own advantage, while \emph{trajectory-scale} use means aggregating the value over a whole trajectory and selecting on that aggregate.

\paragraph{The obstacle: per-transition reweighting.}
We substitute two ground-truth values into the identical advantage-weighted extraction on all nine tasks: the per-step oracle $V = -c$, the exact generator of the preferences, and the discounted cost-to-go oracle ($\gamma = 0.99$). Neither improves on the learned value in any cell (Table~\ref{tab:oracle}), and on all five velocity tasks the learned value is safe while at least one oracle violates, reaching cost $235$ and $169$ against budget $20$ on HalfCheetah and Ant. The learned value's generalization, spreading sparse cost events onto neighboring states, is therefore essential signal rather than approximation error; widening the advantage window degrades safety monotonically, and the same reweighting destroys BC-Safe's safety (Appendix~\ref{app:extraction}). The mechanism is compositional: exponentiated per-step weights reweight actions the dataset already contains at each state, while safety is carried by whole multi-step avoidance behaviors that reweighting fragments, and Appendix~\ref{app:proof} proves a case where the failure is unavoidable (Proposition~\ref{prop:corridor}). The obstacle is scoped rather than absolute: on PointGoal1 five of the eleven per-transition configurations reach the budget, at higher reward than our clone but far thinner margin, while on CarGoal2 and PointGoal2 none of nine does (Appendices~\ref{app:extraction} and~\ref{app:t32t26}).

\paragraph{Identification at per-transition scale.}
Preference training additionally underdetermines the learned value at per-transition scale, which Table~\ref{tab:obstacle2} quantifies at both scales. Trajectory-level ranking is reliable everywhere, per-transition agreement is not: the top-one-percent overlap varies widely across tasks and seeds ($0.39$ to $0.92$), so the tail exponential weighting concentrates on is seed-specific, consistent with the Bradley-Terry null space. This is a reason the per-transition scale is unprincipled, not a predictor of its failure: a matched control with dense per-step supervision disagrees comparably, and neither held-out accuracy nor per-transition reliability predicts the policy. The two can even invert, the most seed-consistent tasks being among those that fail (Appendix~\ref{app:extraction}). What limits the per-transition route is the operator, not the value it is given.

\begin{observation}[Granularity principle]
\label{obs:granularity}
Reweighting transitions does not produce trajectory-level safety even under exact values (the obstacle), and the preference-trained value is in any case identified only at trajectory scale. Trajectory-scale selection uses the value where preference training identifies it, and it keeps whole behaviors intact rather than reweighting their parts.
\end{observation}

The positive half has a direct witness: behavior cloning on the ground-truth safe subset satisfies the budget on eight of nine tasks, including PointGoal1, which no per-transition operator moved. Our method replaces that full labeling with the preference-learned score and adds the certificate.

\paragraph{What the value encodes.}
Inspection corroborates the principle: the input dimension the value is most sensitive to, ranked by the gradient of $\Vbar$ with respect to each observation coordinate, is the cost-defining one on four of five velocity tasks. On every task $\Vbar$ declines as violations approach, which explains its advantage over the exact but sparse per-step oracle, and on navigation $\Vbar$ binned over the workspace recovers a hazard field it was never given. The diagnostics flatten precisely
where the method trails the strongest baseline (Figures~\ref{fig:interp} and~\ref{fig:landscape-layouts}, Appendix~\ref{app:landscape}).

\section{Experiments}
\label{sec:experiments}

\subsection{Setup}
We evaluate on fifteen DSRL tasks \citep{liu2024dsrl}: five velocity-constrained locomotion tasks and ten SafetyGymnasium navigation tasks spanning the Goal, Button, and Circle families for the Point and Car robots, the union of our comparators' task sets at their budgets (20 velocity, 25 navigation). Datasets are the standard DSRL releases, one dataset and environment pair per task; all methods share architectures, the DSRL evaluation protocol, and identical preference data (Appendix~\ref{app:hyper}). We compare against BC-All, the full-label filters BC-Safe and BC-Safe-Seg \citep{liu2024dsrl, gong2026presa}, random-, return-, and bottom-return-selection controls at matched selection size (the pool's ground-truth safe count, so selection signal is compared at equal budget), the full-label methods CDT \citep{liu2023cdt}, CPQ \citep{xu2022cpq} and COptiDICE \citep{lee2022coptidice} trained in-codebase at matched cost limits (Appendix~\ref{app:baselines}), and CPL \citep{hejna2024cpl} consuming exactly our preference data. Our own configurations are the V-filter at the calibration-estimated fraction (ground-truth-matched and fixed top-quartile fractions retained as diagnostics), the calibrated filter at $\alpha = 0.25$, $\delta = 0.1$, $n = 200$, its certification-gated reward-aware variant, and certified return-weighted cloning. Headline configurations and controls report mean and 95 percent bootstrap CIs over five seeds; sweeps use three. A five-task BulletSafetyGym suite, held out as a transfer test with every hyperparameter frozen, is reported in Section~\ref{sec:results} and Appendix~\ref{app:bullet}.

\subsection{Main results}
\label{sec:results}
\begin{table}[t]
\caption{Raw reward and cost on fifteen DSRL tasks (five seeds, mean $\pm$ 95\% CI; budgets 20/25; bold cost within budget, where within budget means mean cost at most the budget, the convention used throughout; BC-All is a single training with 100-episode evaluation; V-filter is the uncertified filter at the calibration-estimated fraction; Calibrated is the full procedure of Section~\ref{sec:calibration-method}, certifying where it can and otherwise returning the conservative selection with an explicit refusal). $^\dagger$: on the four certified tasks the calibrated column applies the certification-gated sub-selection; elsewhere the gated and ungated variants coincide.}
\label{tab:main}
\begin{center}
\scriptsize
\setlength{\tabcolsep}{2.5pt}
\resizebox{\textwidth}{!}{%
\begin{tabular}{l rr rr rr rr}
\toprule
& \multicolumn{2}{c}{BC-All} & \multicolumn{2}{c}{BC-Safe} & \multicolumn{2}{c}{V-filter} & \multicolumn{2}{c}{Calibrated} \\
\cmidrule(lr){2-3}\cmidrule(lr){4-5}\cmidrule(lr){6-7}\cmidrule(lr){8-9}
Task & $R$ & $C$ & $R$ & $C$ & $R$ & $C$ & $R$ & $C$ \\
\midrule
HalfCheetah & 2596 & 106 & 2170\,\scriptsize$\pm$265 & \textbf{3\,\scriptsize$\pm$3} & 1866\,\scriptsize$\pm$385 & \textbf{1\,\scriptsize$\pm$1} & 2124\,\scriptsize$\pm$99$^\dagger$ & \textbf{3\,\scriptsize$\pm$3} \\
Walker2d & 2698 & \textbf{10} & 2700\,\scriptsize$\pm$36 & \textbf{5\,\scriptsize$\pm$6} & 2658\,\scriptsize$\pm$51 & \textbf{4\,\scriptsize$\pm$5} & 2661\,\scriptsize$\pm$57$^\dagger$ & \textbf{1\,\scriptsize$\pm$1} \\
Ant & 2876 & 175 & 2871\,\scriptsize$\pm$29 & \textbf{6\,\scriptsize$\pm$4} & 2465\,\scriptsize$\pm$38 & \textbf{2\,\scriptsize$\pm$1} & 2499\,\scriptsize$\pm$23 & \textbf{1\,\scriptsize$\pm$0} \\
Hopper & 993 & 131 & 365\,\scriptsize$\pm$152 & \textbf{18\,\scriptsize$\pm$14} & 1238\,\scriptsize$\pm$257 & 33\,\scriptsize$\pm$31 & 998\,\scriptsize$\pm$323 & \textbf{3\,\scriptsize$\pm$3} \\
Swimmer & 120 & 92 & 119\,\scriptsize$\pm$10 & 37\,\scriptsize$\pm$31 & 106\,\scriptsize$\pm$8 & \textbf{15\,\scriptsize$\pm$8} & 116\,\scriptsize$\pm$11 & 59\,\scriptsize$\pm$62 \\
CarGoal1 & 15.59 & \textbf{22.44} & 10.12\,\scriptsize$\pm$0.35 & \textbf{8.98\,\scriptsize$\pm$1.34} & 8.64\,\scriptsize$\pm$0.74 & \textbf{8.44\,\scriptsize$\pm$0.85} & 8.45\,\scriptsize$\pm$0.56$^\dagger$ & \textbf{7.85\,\scriptsize$\pm$1.04} \\
CarGoal2 & 6.00 & 38.49 & 3.71\,\scriptsize$\pm$0.32 & \textbf{16.40\,\scriptsize$\pm$2.45} & 3.77\,\scriptsize$\pm$0.33 & \textbf{17.38\,\scriptsize$\pm$2.17} & 3.65\,\scriptsize$\pm$0.36 & \textbf{19.35\,\scriptsize$\pm$5.36} \\
PointGoal1 & 17.46 & 36.94 & 10.44\,\scriptsize$\pm$0.36 & \textbf{12.60\,\scriptsize$\pm$1.24} & 8.27\,\scriptsize$\pm$0.25 & \textbf{11.55\,\scriptsize$\pm$0.31} & 7.19\,\scriptsize$\pm$0.71$^\dagger$ & \textbf{8.70\,\scriptsize$\pm$1.84} \\
PointGoal2 & 14.97 & 76.58 & 7.09\,\scriptsize$\pm$0.65 & \textbf{21.36\,\scriptsize$\pm$2.24} & 6.51\,\scriptsize$\pm$0.54 & \textbf{19.73\,\scriptsize$\pm$1.33} & 6.17\,\scriptsize$\pm$0.55 & \textbf{20.78\,\scriptsize$\pm$1.36} \\
PointButton1 & 7.45 & 50.90 & 1.59\,\scriptsize$\pm$0.41 & \textbf{19.92\,\scriptsize$\pm$4.30} & 0.68\,\scriptsize$\pm$0.50 & \textbf{18.79\,\scriptsize$\pm$5.17} & -0.10\,\scriptsize$\pm$0.56 & \textbf{14.26\,\scriptsize$\pm$3.47} \\
PointButton2 & 11.68 & 66.56 & 5.95\,\scriptsize$\pm$0.54 & 32.89\,\scriptsize$\pm$2.56 & 4.67\,\scriptsize$\pm$0.42 & 30.23\,\scriptsize$\pm$2.92 & 4.38\,\scriptsize$\pm$0.66 & 28.36\,\scriptsize$\pm$3.04 \\
CarButton1 & -3.11 & 43.41 & 2.10\,\scriptsize$\pm$0.27 & \textbf{22.33\,\scriptsize$\pm$2.76} & 0.88\,\scriptsize$\pm$0.26 & \textbf{13.60\,\scriptsize$\pm$2.17} & 0.99\,\scriptsize$\pm$0.39 & \textbf{14.61\,\scriptsize$\pm$2.01} \\
CarButton2 & -4.10 & 45.15 & -1.09\,\scriptsize$\pm$0.41 & \textbf{22.26\,\scriptsize$\pm$1.79} & -0.89\,\scriptsize$\pm$0.15 & \textbf{18.50\,\scriptsize$\pm$3.21} & -0.91\,\scriptsize$\pm$0.18 & \textbf{19.07\,\scriptsize$\pm$1.60} \\
PointCircle1 & 52.5 & 156.0 & 36.0\,\scriptsize$\pm$2.8 & \textbf{4.7\,\scriptsize$\pm$3.9} & 34.9\,\scriptsize$\pm$2.3 & \textbf{10.2\,\scriptsize$\pm$2.8} & 33.2\,\scriptsize$\pm$4.5 & \textbf{17.5\,\scriptsize$\pm$17.0} \\
PointCircle2 & 42.0 & 158.8 & 36.8\,\scriptsize$\pm$1.1 & 44.2\,\scriptsize$\pm$20.9 & 35.3\,\scriptsize$\pm$1.1 & 82.4\,\scriptsize$\pm$31.1 & 35.8\,\scriptsize$\pm$0.9 & 87.2\,\scriptsize$\pm$30.5 \\
 
\bottomrule
\end{tabular}}
\end{center}
\end{table}

We judge methods by the constrained objective itself: credit first for satisfying the cost budget, and only then for the reward achieved while doing so. Bold cells in every table mark budget satisfaction. Table~\ref{tab:main} presents the main comparison. Raw cost is primary, since thresholds and budgets are stated in it; Table~\ref{tab:normcost} restates cost in normalized units for comparison with the literature.

\paragraph{The calibrated filter.} One fixed procedure across tasks, consuming preferences plus 200 calibration labels, is safe on twelve of fifteen, a count that includes its uncertified fallback. The certificate itself fires on four tasks at $\alpha = 0.25$ and on eight of the nine analysis tasks at $\alpha = 0.40$ (Section~\ref{sec:ablations}), and on uncertified tasks the deliverable is the conservative selection with an explicit no-certificate report. Its failures are Swimmer, where only the estimated-fraction V-filter is within budget, and PointButton2 and PointCircle2, which have no selectable safe subset, where every method fails and the calibration refuses to certify.

\paragraph{The value against ground truth.} The V-filter at the calibration-estimated fraction, which reads no ground-truth quantity, is safe on twelve of fifteen, as is BC-Safe, which needs a budget label on every trajectory. Their failure sets differ rather than nest: ours violates on Hopper, which BC-Safe handles; BC-Safe violates on Swimmer, which ours handles; both violate on PointButton2 and PointCircle2. On the eleven where both are safe ours runs at lower mean cost on nine, mean normalized cost $0.46$ against BC-Safe's $0.53$ (Appendix~\ref{app:baselines}). The mechanism is graded versus binary signal: the value ranks trajectories by distance from the budget while the label admits everything up to it, and the clone inherits that boundary.

\paragraph{The certificate as a gate.} The gated reward-aware variant is certification-safe. On the four certified tasks it keeps cost within budget in every cell with reward held within confidence intervals, and uncertified tasks keep their conservative selections; ungated, the same sub-selection on Hopper's contaminated selection multiplies mean cost to 82, and the gate refuses there (Appendix~\ref{app:extended}). Sub-selection discards trajectories but never adds them, so reward rises only where the certified selection is reward-rich.

\paragraph{Transfer.} The unchanged pipeline transfers to five BulletSafetyGym tasks with new embodiments and a harder budget, satisfying it on all five with the guarantee holding in every validation cell (Appendix~\ref{app:bullet}). Across both suites the certificate fires on 6 of 20 tasks at $\alpha = 0.25$, and yield is governed by the purity margin of Section~\ref{sec:guarantee-validation} rather than raw safe mass.

\paragraph{External baselines.} Against full-label offline safe RL and direct preference-based policy learning (Table~\ref{tab:baselines}, Figure~\ref{fig:pareto}, Appendix~\ref{app:baselines}), CDT is safe on four of the five velocity tasks and on PointCircle1, with reward above ours wherever it is safe, but violates at \emph{every} cost target on the eight Goal and Button tasks, its realized cost barely responding to the cost target. Supervision does not substitute for readout scale: at their budget-matched cost limits CDT, CPQ and COptiDICE read a cost on every transition and keep five, four and two of fifteen within budget, and CDT's safest per-task target reaches seven, while a filter consuming $200$ bits keeps twelve (Appendix~\ref{app:cdtsweep}). CPL, consuming exactly our preference data, keeps four: the same supervision that curates safely fails as a direct policy-optimization signal. Every \emph{trajectory-scale} filter-then-clone arm in our study keeps more tasks within budget than every optimization-based method, eleven to thirteen against two to five; the segment-scale filter, safe on two of fifteen despite full labels, and the no-information selection controls do not, which is the granularity principle at benchmark scale rather than an exception to it.

\paragraph{Composability.} Because the certificate describes the data rather than the policy, the certified selection can be handed to any offline learner. Retraining CDT on the certified selections of the four tasks certifying at $\alpha = 0.25$ makes every run safe, 36 of 36 across three seeds and three distinct selections per task, three of them out of specification, meaning a realized unsafe fraction above $\alpha$ (Tables~\ref{tab:compose} and~\ref{tab:a25draws}): on both navigation tasks, where full-data CDT violates by nearly 50 percent, it runs within budget, and on HalfCheetah it reaches CDT's own full-data reward on the deepest selection at cost below 1. Only four tasks appear because only four certify, which is refusal rather than a choice of favorable cases. Attribution is controlled twice: full-data CDT without augmentation remains severely unsafe on navigation, and no cost target down to a fifth of the budget rescues it (Appendix~\ref{app:cdtsweep}). CPL, a weaker learner, follows on three of the four, within budget on all 27 of their runs, so curation is not tuned to one consumer. Walker2d is the exception and the informative one: full-data CPL already sits at the budget there, and the certified selections, a fifth the size of the pool, push it beyond. Curation cannot rescue a learner whose violations are not data-driven, and it cannot supply the volume a weak learner needs. At the looser $\alpha = 0.40$ the three diverge sharply: our clone stays safe on 22 of 24 task-selections, CPL on 17, and CDT violates on every navigation task on at least one (Appendix~\ref{app:composea40}). Averaging learners outvote the unsafe tail while return-conditioned learners seek exactly it, so composing with them requires the tighter certificate, giving $\alpha = 0.25$ an empirical justification beyond conservatism. Certified curation is thus a supervision-agnostic front end for stronger learners, and one transfer task marks the boundary where an operator's own violations put it out of reach (Appendix~\ref{app:bullet}).

\paragraph{The certified return-weighted operator.} The certificate's value depends on its consumer, so we derive one from the pipeline itself: certified return-weighted cloning (Section~\ref{sec:gating}), run only on certified selections. Across the four certifying tasks' twelve selections, the CarRun echo and four intensities, 39 of the 40 within-specification cells are safe; the exception, CarRun's hard top-half rule, shows that resampling rather than return-seeking carries the safety. The operator converts certification into reward: on HalfCheetah's deployed selection it lifts the clone's 2104 to 2271 at cost 0.03, above CDT retrained on that same selection. Beyond the certified composition the same knob is not safe, the worst tail selection violating once the clip is loosened to $\kappa = 3$, so the certificate is what decouples the operator's safety from its aggressiveness (Table~\ref{tab:operator}).

\subsection{Guarantee validation}
\label{sec:guarantee-validation}
If safety can be decoded from state, when can the decoding be trusted? We validate Proposition~\ref{prop:guarantee} on the nine analysis tasks and the five transfer tasks across 2000 fresh calibration draws per task, seed, and size $n \in \{50, 100, 200, 400\}$ (Figure~\ref{fig:coverage}, Appendix~\ref{app:certprops}). The unconditional false-certification rate, the probability of certifying a threshold whose selection is out of specification, stays below $\delta = 0.1$ in every cell, the worst $0.078$ (Figure~\ref{fig:coverage}, left, and Appendix~\ref{app:certprops}). Certification rates increase with $n$ and with achievable purity, and Hopper, Ant and PointGoal2 certify never or almost never, correctly, since their top-quantile selections genuinely exceed $\alpha$. Refusal is the guarantee working, and predictable: a pool's \emph{purity margin} is $\alpha$ minus the smallest unsafe fraction any threshold on the grid can achieve. Proposition~\ref{prop:rate} below makes the rate a deterministic function of this margin, and the resampled rates verify it. Rate tracks margin almost exactly across all fourteen (Spearman $0.996$), and again when segment length rather than task is varied (Spearman $0.994$ over 27 cells). It behaves as a threshold rather than a trend: tasks with a negative margin certify essentially never, positive margins yield rates rising with the margin, and varying the margin continuously on constructed pools reproduces the transition at zero (Figure~\ref{fig:contamination}, Appendix~\ref{app:proof}).

\begin{proposition}[Certification dichotomy]
\label{prop:dichotomy}
Write $u_j = U_j / N_j$ for the unsafe fraction of the selection at grid threshold $\lambda_j$, and let $u_{\min} = \min_j u_j$.
\begin{enumerate}
\item If $u_{\min} > \alpha$, then $\Pr[\textnormal{certify}] \le \delta$ for every calibration size $n$.
\item If $u_1 < \alpha$, then $\Pr[\textnormal{certify}] \to 1$ as $n \to N$.
\end{enumerate}
Proof in Appendix~\ref{app:proof}.
\end{proposition}

\begin{proposition}[Certification rate]
\label{prop:rate}
Let $N_1 = |S(\lambda_1)|$ contain $U_1$ unsafe trajectories, $k^{*} = \lfloor \alpha N_1 \rfloor + 1$, $m \sim \mathrm{Hyp}(N, N_1, n)$ the number of calibration trajectories falling in $S(\lambda_1)$, and $k \mid m \sim \mathrm{Hyp}(N_1, U_1, m)$ the unsafe count among them. Then, with $F_{\mathrm{hyp}}$ the hypergeometric CDF,
$\Pr[\textnormal{certify}] = \sum_{m} \Pr[m] \sum_{k : F_{\mathrm{hyp}}(k; N_1, k^{*}, m) \le \delta} \Pr[k \mid m]$.
Proof in Appendix~\ref{app:proof}.
\end{proposition}

Part (i) is the formal content of refusal: on such a pool no threshold can purify below $\alpha$, and no labeling budget changes that. Part (ii) makes the label cost forecastable, and Proposition~\ref{prop:rate} makes it exact: the strict fixed sequence certifies if and only if it rejects at $\lambda_1$, so the rate is one hypergeometric event in which the scorer enters only through $U_1$, and the expression inverts to the budget $n$ that reaches a target rate. It matches a $200$-draw resample's rates within $0.019$ on average (Appendix~\ref{app:proof}). The $2000$-draw resample yields the deployer-facing conditional rate $\Pr[\text{violate} \mid \text{certified}]$, small where certification is common and large where it is rare (Ant $1.00$, certifying in under one percent of draws), so the certification rate is the trust signal to read alongside the certificate (Table~\ref{tab:condviol}). A deployer need not run that resample: the calibration counts already spent estimate the rate through Proposition~\ref{prop:rate}, and $\delta / \Pr[\text{certify}]$ turns it into a bound on the conditional rate, so the procedure returns a triple: threshold, estimated rate, and bound, validated in Appendix~\ref{app:certprops} with its one failure mode.

Two boundaries qualify the claim. Certification is stricter than policy safety requires: the binary indicator counts cost 26 against budget 25 as fully unsafe, so selections can exceed $\alpha$ while their clones stay within budget; the gap between twelve safe tasks and four certifying is this conservatism. Composition is also not sufficient: BC-Safe's training set has unsafe fraction exactly zero, an implicit certificate at $\alpha = 0$, yet its clone violates on three of twenty tasks, from the label's boundary near budget and from cloning error in the large ones (Swimmer at 36.8), visible per task in Appendix~\ref{app:baselines}. The certificate audits the data; the favorable mapping to policy behavior is carried by margin-ranked selection, and we report the two separately. Because that mapping is empirical, we also certify the deployed policy from the same one bit per episode: Clopper-Pearson on episodes pooled across the fifteen tasks puts $\Pr[\text{cost} > \text{budget}]$ at $0.224$ for the calibrated selection against $0.627$ for cloning everything and $0.243$ for the full-label oracle. An empirical-Bernstein bound on $\mathbb{E}[\text{cost}]$ certifies ten of fifteen tasks on every seed, the shortfall against twelve safe tasks being the price of a finite-sample bound. The deployment certificate is the stronger statement about one policy; the offline certificate is the one worth having first, because it is computable before any policy exists, needs no environment interaction, and audits a reusable artifact every downstream learner inherits (Appendix~\ref{app:policycert}).

The three tasks the procedure does not solve divide cleanly. On PointButton2 and PointCircle2 no selection works at all: cloning the ground-truth safe subset already violates ($32.9$ and $44.2$ against budget $25$), so the failure is in what the pool contains rather than in which part of it we keep, and the certificate refuses on every seed. Swimmer is the opposite case and the one that locates the data-to-policy gap. Its deployed selection is $14.5$ percent unsafe, comfortably inside $\alpha = 0.25$, and the clone still runs at $58.9$ against budget $20$; the uncertified V-filter, which keeps the larger selection the estimated fraction allows rather than the fallback's lower bound on it, lands at $14.6$ (Appendix~\ref{app:certprops}). A clean training set is therefore not a safe policy, and on this benchmark that gap is visible on exactly one task. Neither mode is a failure of the guarantee: the first is a pool with no safe subset to certify, reported as a refusal; the second is a certified composition whose clone we measure rather than promise.

\subsection{Ablations}
\label{sec:ablations}

\paragraph{Selection signal.} Matched-size random- and return-selection controls isolate the safety value's contribution (Table~\ref{tab:controls}). Return selection, the percentile-BC analogue, is unsafe on all fifteen tasks; random selection on fourteen. Bottom-return selection, the strongest form of the concern that the score merely anti-ranks return, is safe on eight of fifteen yet catastrophically unsafe on four of five velocity tasks even though its selected trajectories are low-cost. The score tracks cost on every task (Spearman $-0.55$ to $-0.93$) and anti-correlates with return in proportion to how tightly cost
and return are coupled (Pearson $0.72$), as any safety-tracking score must (Table~\ref{tab:scorecorr}).

\paragraph{Calibration size and guarantee level.} Certification rate falls gracefully with $n$: at $n = 50$ the procedure certifies rarely and the guarantee still holds, so a small budget costs conservatism, never silent failure. Sweeping $\alpha$ maps the operating curve: $0.05$ certifies nowhere, $0.10$ almost nowhere, $0.25$ is the practical operating point, and $0.40$ certifies broadly except on Hopper, whose rate is $0.002$ there and zero at every tighter level, correctly (Figures~\ref{fig:paretotargets}, \ref{fig:coverage} and~\ref{fig:alphacurve}).

\paragraph{Labels-only filtering.} The calibration sample could train a filter instead of auditing one. A supervised classifier on only the $n$ labels, with our architecture, scoring rule and fractions, is safe on 11 to 13 of fifteen across budgets $n \in \{50, 100, 200, 400\}$, and splitting the budget between training and calibration certifies at a rate comparable to ours. At matched budget preferences buy one further safe task, twelve against eleven, and no additional certification; what they buy is the audit, since labels-only spends its judgments on training while the preference route holds all 200 back. Proposition~\ref{prop:rate} explains the certification parity: the scorer enters the certification rate only through the purity its top quantile attains, so scorers of equal precision certify at equal rates (Appendices~\ref{app:extended} and~\ref{app:labelsonly}).

\paragraph{Further controls.} Five controls in Appendix~\ref{app:extended} bound what the calibration sample and the score contribute: held-out preferences in place of labeled calibration are safe on four of fifteen; label noise, a Boltzmann labeler and a halved pair budget degrade precision without ever certifying falsely; numeric costs in place of the $200$ bits buy nothing; and $Q(s,a)$ forfeits policy-agnostic scoring.

\section{Limitations}
\label{sec:limits}
The certificate covers training-set composition, not the resulting policy's cost; a starved selection can underfit into an unsafe policy, mitigated but not eliminated by the fallback's minimum size and addressed from the other side in Section~\ref{sec:guarantee-validation}. The guarantee consumes 200 labeled trajectories, which the preference-only variant removes at the cost of the certificate. All preferences are synthesized from ground-truth episodic costs, so the weak-oracle regime is modeled, under label noise and a Boltzmann labeler, rather than entered with human labelers. The extreme tail of the trajectory score is unreliable, which the capped grid works around; tail-robust scoring remains open. The return-weighted operator must not be run on uncertified selections: on the worst observed out-of-specification selection it stays within budget at both milder clips and under the hard rule, and violates once the clip is loosened to $\kappa = 3$. Three scoping notes close: a labels-only classifier comes within one safe task of the comparison route, so that route's advantage is mostly the held-out audit rather than coverage; the action-conditioned ablation covers five tasks; and composability spans four tasks at $\alpha = 0.25$ and eight at $\alpha = 0.40$.

\section{Conclusion}
\label{sec:conclusion}
We presented certified safety curation: a state-only value learned from segment comparisons selects whole trajectories, Learn-then-Test calibrates the threshold, and the certified selection is cloned under a distribution-free guarantee on its composition. Safety in these domains is decodable at trajectory scale from supervision far weaker than a cost function; what separates methods is the certificate and its refusals. The selection also transfers: the strongest full-label baseline retrained on it becomes safe where it otherwise violates. Offline, the data bounds what any method can achieve, and certifying it is a guarantee available before a policy is trained or run.

\subsubsection*{AI use statement}
We used generative AI tools as a coding and editing assistant: for writing and
refactoring experiment-orchestration and analysis scripts, for regenerating
tables and figures from archived records, and for copy-editing prose. We did not
use generative AI to originate the research question, the method, the
theoretical results, or their proofs, and no experimental result was produced or
altered by an AI tool outside the logged pipeline described in the paper. All
AI-assisted code was reviewed by the authors and its outputs cross-checked
against the archived evaluation records; every numerical claim
in the text is either read from a generated table or bound to its producing expression by the
auditing scripts accompanying the paper.
We take responsibility for the final content of this work, including all text,
claims, and artifacts.

\subsubsection*{Ethics statement}
This work studies safety constraint satisfaction in offline reinforcement
learning using public simulated benchmarks; it involves no human subjects, no
personal data, and no new data collection. The preference labels are synthesized
from benchmark cost functions rather than elicited from people, which we state
as a limitation rather than a feature. One dual-use consideration is specific to this paper: the certificate bounds the
composition of a training set, not the behavior of the resulting policy. We
report certification and policy outcomes separately throughout, and
Section~\ref{sec:limits} states that boundary explicitly.

\subsubsection*{Reproducibility statement}
All experiments use public DSRL datasets. Every table regenerates from archived evaluation records via the scripts accompanying the paper, and the guarantee figures regenerate from archived calibration
statistics; the one exception is the CarRun composability echo, whose original certified selection was lost and re-derived (Appendix~\ref{app:bullet}). Hyperparameters are listed in Appendix~\ref{app:hyper}; no per-task tuning is used for the calibrated method. Predictions and decision rules the text calls registered, including the transfer suite's do-not-filter test and its composability echo (Appendix~\ref{app:bullet}), were committed in writing before the runs they govern, and are reported whether or not the prediction survived.

\bibliography{references}
\bibliographystyle{iclr2027_conference}

\appendix

\section{Algorithm}
\label{app:algo}
\begin{center}
\fbox{\parbox{0.92\linewidth}{
\textbf{Certified safety curation.}
\begin{enumerate}
\item Train $K$ value networks on $\Dpref$ by Equation~\ref{eq:bt}; set $\Vbar$ to the ensemble mean.
\item Score every trajectory: $g(\tau) = \frac{1}{|\tau|}\sum_{s_t \in \tau} \Vbar(s_t)$.
\item Draw $n$ calibration trajectories uniformly; label episodic costs.
\item Fixed-sequence LTT over score quantiles $0.85$ down to $0.30$ with exact hypergeometric tests at level $\delta$ against risk target $\alpha$; return the last certified threshold, or an uncertified selection at the calibration-estimated safe-mass fraction with an explicit no-certificate report.
\item If certified, optionally curate for reward, retaining the top half by return or resampling trajectories with probability proportional to $\exp(\mathrm{clip}(z(R), \pm \kappa))$.
\item Behavior-clone the selection.
\end{enumerate}}}
\end{center}

\section{Proofs}
\label{app:proof}

Propositions~\ref{prop:guarantee}, \ref{prop:dichotomy} and~\ref{prop:rate} are proved in
turn, followed by one result stated only here: a construction separating the two operator
classes the analysis contrasts, per-transition reweighting and whole-trajectory selection,
of Section~\ref{sec:analysis}.

Fix the threshold grid $\lambda_1 > \lambda_2 > \cdots > \lambda_J$ in advance. For each $j$ write $S_j = S(\lambda_j) \subseteq \Doff$, $N_j = |S_j|$, and let $U_j$ denote the number of unsafe trajectories in $S_j$. The null hypothesis at step $j$ is $\mathcal{H}_j: U_j > \alpha N_j$, equivalently $U_j \ge U_j^{*}$ with $U_j^{*} = \lfloor \alpha N_j \rfloor + 1$. Throughout the proofs,
$\binom{a}{b}$ denotes the binomial coefficient, the number of $b$-element subsets of an
$a$-element set, not a column vector. $\mathrm{Hyp}(\cdot, \cdot, \cdot)$ is the hypergeometric
distribution, its arguments the population size, the number of successes in it, and the
number drawn.
The composition guarantee of Proposition~\ref{prop:guarantee} rests on two lemmas.

\begin{lemma}[Conditional distribution]
\label{lem:conddist}
Let $\mathcal{C}$ be drawn uniformly without replacement from $\Doff$, and condition on $m_j = |\mathcal{C} \cap S_j|$. Then $\mathcal{C} \cap S_j$ is a uniform without-replacement sample of size $m_j$ from $S_j$, so $k_j \mid m_j \sim \mathrm{Hyp}(N_j, U_j, m_j)$.
\end{lemma}

\begin{proof}
Let $A \subseteq S_j$ with $|A| = m_j$ and $B \subseteq \Doff \setminus S_j$ with
$|B| = n - m_j$. Since $\mathcal{C}$ is a uniformly random $n$-subset of $\Doff$,
$\Pr[\mathcal{C} = A \cup B] = \binom{N}{n}^{-1}$, which does not depend on $A$ or $B$.
Summing over the admissible $B$ gives
\[
  \Pr\big[\mathcal{C} \cap S_j = A \mid m_j\big]
  \;=\; \binom{N - N_j}{n - m_j} \Big/ \Big(\binom{N_j}{m_j}\binom{N - N_j}{n - m_j}\Big)
  \;=\; \binom{N_j}{m_j}^{-1},
\]
the same value for every such $A$. Hence $\mathcal{C} \cap S_j$
is a uniform $m_j$-subset of $S_j$, and counting its unsafe elements gives $k_j \mid m_j
\sim \mathrm{Hyp}(N_j, U_j, m_j)$.
\end{proof}

\begin{lemma}[Super-uniformity]
\label{lem:superunif}
Define $p_j = F_{\mathrm{hyp}}(k_j; N_j, U_j^{*}, m_j)$. If $\mathcal{H}_j$ holds, then $\Pr[p_j \le \delta \mid m_j] \le \delta$ for every $m_j$, and hence marginally.
\end{lemma}

\begin{proof}
The argument has three steps.

\emph{Stochastic monotonicity.} If $U \le U'$ then $\mathrm{Hyp}(N, U, m) \preceq_{\mathrm{st}}
\mathrm{Hyp}(N, U', m)$, where $\preceq_{\mathrm{st}}$ is stochastic dominance. Couple the two by drawing a single uniform $m$-subset $M$ of
$\{1, \dots, N\}$ and declaring the first $U$ indices successes in one marginal and the
first $U'$ in the other. Both marginals are correct, and $|M \cap \{1, \dots, U\}| \le
|M \cap \{1, \dots, U'\}|$ pointwise, which is stochastic dominance.

\emph{Dominance of the $p$-value.} Under $\mathcal{H}_j$ we have $U_j \ge U_j^{*}$, so by
Lemma~\ref{lem:conddist} and the previous step $k_j \succeq_{\mathrm{st}} X$ for $X \sim
\mathrm{Hyp}(N_j, U_j^{*}, m_j)$. The map $k \mapsto F_{\mathrm{hyp}}(k; N_j, U_j^{*},
m_j)$ is non-decreasing, and a non-decreasing map preserves stochastic order, so $p_j
\succeq_{\mathrm{st}} F_{\mathrm{hyp}}(X; N_j, U_j^{*}, m_j)$.

\emph{Super-uniformity.} For any random variable $X$ with CDF $F$ and any $\delta \in
(0,1)$, $\Pr[F(X) \le \delta] \le \delta$: if no point of the support satisfies $F(x)
\le \delta$ the probability is zero, and otherwise, writing $x^{*}$ for the largest such
point, $F$ non-decreasing gives $\{F(X) \le \delta\} = \{X \le x^{*}\}$, whose
probability is $F(x^{*}) \le \delta$. Applying this to $X$ above and combining with the
previous step, 
\[
  \Pr[p_j \le \delta \mid m_j] \;\le\; \Pr[F_{\mathrm{hyp}}(X) \le \delta] \;\le\; \delta
\]
 for every $m_j$, and averaging over $m_j$ gives the marginal statement.
\end{proof}

\begin{proof}[Proof of Proposition~\ref{prop:guarantee}]
Let $t$ be the index of the first true null in the fixed order, if any exists; if none exists, no certification can be false and the claim is vacuous. The strict stopping rule tests $\mathcal{H}_1, \mathcal{H}_2, \ldots$ in order and stops at the first failure to reject. Reaching $\mathcal{H}_t$ requires rejecting $\mathcal{H}_1, \ldots, \mathcal{H}_{t-1}$, all of which are false nulls, so no error is committed before step $t$. Suppose the procedure certifies and the returned threshold is $\hat\lambda =
\lambda_{\hat\jmath}$ with $\mathrm{unsafe}(S(\lambda_{\hat\jmath})) > \alpha$, the
violation event of Equation~\ref{eq:guarantee}; that is, $\mathcal{H}_{\hat\jmath}$ is true. Since $t$ is the least index with a true
null, $\hat\jmath \ge t$. The strict stopping rule returns $\lambda_{\hat\jmath}$ only if
it rejected $\mathcal{H}_1, \ldots, \mathcal{H}_{\hat\jmath}$ in order, so in particular
it tested and rejected $\mathcal{H}_t$. The false-certification event is therefore contained in
\[
  \{\mathcal{H}_t \text{ tested and rejected}\} \;\subseteq\; \{p_t \le \delta\},
\]
whose probability is at most $\delta$ by Lemma~\ref{lem:superunif}. The bound holds without any monotonicity assumption on $j \mapsto U_j / N_j$; monotonicity only affects which hypotheses are reachable, that is, power.
\end{proof}

One dependence deserves an explicit statement. The score $g$ is trained on preferences synthesized from the episodic costs of the same pool $\Doff$, so $g$, and hence every $S_j$, is a function of the pool's labels. Validity is unaffected: $g$ and the sets $S_j$ are fixed before $\mathcal{C}$ is drawn and do not depend on which trajectories land in $\mathcal{C}$, so the exchangeability argument of Lemma~\ref{lem:conddist} applies verbatim. What the construction does forfeit is independent test data for the score itself, which is why the guarantee is on the selection's composition and not on any property of $g$. As a further remark, calibration trajectories whose labels reveal them unsafe could be excluded from the selection, which can only improve its composition; the reported results do not apply this and are conservative in that respect.

A corollary transfers the composition bound to the return-weighted operator of Section~\ref{sec:gating}. Resampling the certified selection with weights $w_i \in [e^{-\kappa}, e^{\kappa}]$ multiplies its unsafe-to-safe odds by at most $e^{2\kappa}$, so the certified bound $\alpha$ on the unsafe fraction becomes a bound of $e^{2\kappa}\alpha / (e^{2\kappa}\alpha + 1 - \alpha)$ on the unsafe mass of the resampling distribution. The bound is informative only at small $\kappa$ (at $\alpha = 0.25$ it already exceeds one half by $\kappa = 1$), so the intensity-robustness observed in Section~\ref{sec:results} is an empirical property of the score-return geometry of certified selections rather than a consequence of the guarantee; the guarantee's role is to pin the composition the operator starts from. Certifying the operator's weighted risk directly, rather than through this corollary, does not help at our label budget. The resampling weights are a function of returns and of the selection, so the normalizer is known exactly and only the weighted unsafe mass needs estimating, but that quantity ranges over $[e^{-\kappa}, e^{\kappa}]$, and at $\kappa = 2$ a distribution-free bound needs roughly $10^3$ labeled trajectories \emph{inside} the selection, about $6700$ calibration labels at a 15 percent selection, against the 200 we spend. The composition risk is certifiable at this budget precisely because it is a bounded count rather than a weighted average.

\begin{proof}[Proof of Proposition~\ref{prop:dichotomy}]
(i) If $u_{\min} > \alpha$ then $u_j > \alpha$ for every $j$, so any returned threshold has $\mathrm{unsafe}(S(\hat\lambda)) > \alpha$. Certification therefore implies the violation event, so for any $n$
\[
  \Pr[\textnormal{certify}]
  \;=\; \Pr\big[\textnormal{certify} \wedge \mathrm{unsafe} > \alpha\big]
  \;\le\; \delta
\]
by Proposition~\ref{prop:guarantee}. (ii) At $n = N$ the calibration sample is the whole pool, so $m_1 = N_1$ and $k_1 = U_1$ deterministically, and $p_1 = F_{\mathrm{hyp}}(U_1; N_1, U_1^{*}, N_1)$ is $0$ when $U_1 < U_1^{*} = \lfloor \alpha N_1 \rfloor + 1$, which is implied by $u_1 < \alpha$. The first hypothesis is then rejected with certainty and a threshold is returned. For $n < N$ the same event gives a lower bound: whenever the calibration draw happens to contain all of $S(\lambda_1)$ we have $m_1 = N_1$ and the argument above applies verbatim, so
\[
  \Pr[\textnormal{certify}] \;\ge\; \Pr[m_1 = N_1] \;=\; \binom{N - N_1}{n - N_1} \Big/ \binom{N}{n},
\]
which tends to one as $n \to N$.
\end{proof}

\paragraph{Discussion of Proposition~\ref{prop:dichotomy}.}
The two halves are not symmetric: part (i) is a corollary of Proposition~\ref{prop:guarantee} rather than new machinery, while part (ii) concerns the first tested threshold, since the fixed-sequence walk stops at its first failure; when the score ranks well the two coincide, because $u_j$ is smallest at the top quantile.

Refusal is therefore a correctness property, not a conservative heuristic. The second half is a question of labels, and its cost is forecastable: the budget at which a task first certifies in half of draws falls off as a power of the margin, shallower than a fixed-threshold test would give because the walk reaches deeper thresholds as $n$ grows (Appendix~\ref{app:certprops}). Since the margin is estimable from the calibration sample, a practitioner can forecast before spending labels whether a certificate is attainable and what it would cost. A rare certificate on a mostly-refusing task is weak evidence; the conditional rates the same resample yields are reported in Section~\ref{sec:guarantee-validation} and Table~\ref{tab:condviol}. The validation resamples calibration draws with the score fixed, the theorem's own setting, so it checks the implementation and finite-sample constants rather than testing an assumption.

\paragraph{A construction separating the operator classes.} The following minimal CMDP makes the operator content of the obstacle provable rather than only measured. Consider a deterministic $T$-step corridor whose state is the time index, so successor states are action-independent and the state records no action history; at every step two actions are available, one safe (per-step cost 0) and one risky (cost 1), and the budget is 0. The behavior distribution is a mixture: with probability $q \in [0, 1)$ a trajectory comes from an all-safe expert, and otherwise from a policy that takes the safe action with probability $p_{\mathrm{s}} \in (0, 1)$ independently at each step, so the per-step marginal safe frequency at every state is $\bar p = q + (1 - q) p_{\mathrm{s}}$. We state the population version; empirical frequencies concentrate around $\bar p$ at each step by Hoeffding's inequality.

\begin{proposition}
\label{prop:corridor}
In the corridor CMDP, per-transition weighted cloning with weights $w(s, a)$ and safe-to-risky weight ratio $W$ at every state yields the product policy with per-step safe probability
\[
  p'_{\mathrm{s}} \;=\; \frac{\bar p W}{\bar p W + 1 - \bar p},
\]
whatever the expert mass $q$. Consequently: (i) any state-only weighting gives $W = 1$, because successors are action-independent, so the cloned policy has all-safe probability $\bar p^{\,T}$, which tends to zero even though a fraction $q$ of the data is all-safe; (ii) reaching all-safe probability $1 - \epsilon$ requires
\[
  W \;\ge\; \frac{1-\bar p}{\bar p}\,\big((1-\epsilon)^{-1/T} - 1\big)^{-1}
  \;\longrightarrow\; \frac{1-\bar p}{\bar p} \cdot \frac{T}{\ln\frac{1}{1-\epsilon}},
\]
linear in the horizon, so weights bounded by any fixed clip fail beyond a horizon threshold; (iii) trajectory-level selection of the zero-cost trajectories followed by unweighted cloning yields the exactly safe policy with probability at least $1 - (1 - q)^{N}$ over a dataset of $N$ trajectories, hence with probability one as $N \to \infty$ for any $q > 0$.
\end{proposition}

\begin{proof}
Weighted cloning at state $s$ maximizes $\sum_a c_a \log \pi(a \mid s)$ over the
simplex, where $c_a = \mu(a) w(s, a) > 0$ is the population weight of action $a$ and $\mu$ is the mixture's per-step action marginal, $\mu(\text{safe}) = \bar p$. The
objective is strictly concave on the relative interior and tends to $-\infty$ at the
boundary, so it has a unique maximizer there; the stationarity condition of
\[
  \sum_a c_a \log \pi_a - \eta\Big(\sum_a \pi_a - 1\Big)
\]
is $c_a / \pi_a = \eta$ for every $a$, and normalizing gives
\[
  \pi_a \;=\; c_a \Big/ \sum_b c_b,
  \qquad\text{that is}\qquad
  \pi(a \mid s) \;\propto\; \mu(a)\, w(s, a).
\] With two actions of weights in ratio $W$ this is $p'_{\mathrm{s}}$ as
stated; rollout steps are independent because the state carries no action history, so the mixture's across-step correlation is not representable by any per-state policy and enters only through $\bar p$. For (i), a weight that depends on the state alone takes the same value on both actions at each state, since both lead to the same successor, so $W = 1$ and $p'_{\mathrm{s}} = \bar p$. For (ii), $(p'_{\mathrm{s}})^{T} \ge 1 - \epsilon$ forces $p'_{\mathrm{s}} \ge r$ with $r = (1-\epsilon)^{1/T}$, and $\bar p W/(\bar p W + 1 - \bar p) \ge r$ rearranges to
\[
  W \;\ge\; \frac{1-\bar p}{\bar p} \cdot \frac{r}{1-r}
  \;=\; \frac{1-\bar p}{\bar p}\big((1-\epsilon)^{-1/T} - 1\big)^{-1}
  \;\longrightarrow\; \frac{1-\bar p}{\bar p} \cdot \frac{T}{\ln\frac{1}{1-\epsilon}}.
\] For (iii), every transition in a zero-cost trajectory takes the safe action, so cloning the selection reproduces it deterministically; the selection is nonempty unless no trajectory is zero-cost, and each expert trajectory is zero-cost, which bounds the failure probability by $(1-q)^{N}$.
\end{proof}

The construction is faithful to the paper's setting twice over. Segment-sum preference training receives zero signal here, since every trajectory visits the identical state sequence, so a state-only value necessarily lands in case (i), exactly where action-independent successors leave segment preferences uninformative as well; and advantage weights built from observed per-step costs have the fixed ratio $W = e^{1/\beta}$, which case (ii) defeats at long horizons, the deployed clip of 20 included. The two operator classes thus pay different prices. Reweighting needs a weight range growing with the horizon regardless of how much all-safe data the mixture contains, while selection needs only that the safe subset exist, which any positive expert mass guarantees. The benchmarks supply safe trajectories at 9 to 52 percent of each pool precisely because real behavior is correlated across steps rather than independent, and that correlation is what whole-trajectory selection preserves and per-transition reweighting fragments. The corridor isolates the operator's structural limitation; the benchmark-scale claim rests on the empirical sweep (Appendix~\ref{app:t32t26}), since real states partially record consequences, which is why one navigation task escapes at a sharp enough temperature. Figure~\ref{fig:corridor} simulates the construction at $q = 0.2$: unweighted cloning collapses immediately, the deployed clip of 20 collapses by the benchmarks' horizon of 1000, a clip of 400 is already bending, and selection holds at one throughout because the expert subset always exists, with theory (dashed) matching simulation (markers).

\begin{figure}[h]
\begin{center}
\includegraphics[width=0.6\linewidth]{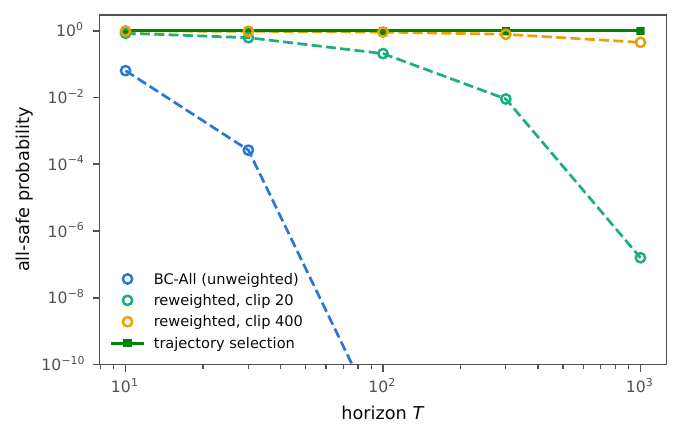}
\end{center}
\caption{The corridor construction, simulated ($p_{\mathrm{s}} = 0.7$, expert mass $q = 0.2$, $N = 10^4$, 50 replicates): all-safe probability of the cloned policy against horizon. Open markers are simulation; dashed lines are the closed forms of Proposition~\ref{prop:corridor} for the three per-transition weight clips, which the simulation tracks closely. Trajectory selection (solid, filled) sits at its closed form $1 - (1-q)^{N}$, indistinguishable from one at this $N$.}
\label{fig:corridor}
\end{figure}

\paragraph{Certification probability in closed form.} Proposition~\ref{prop:dichotomy}
separates the two regimes but leaves the rate between them asymptotic. Because the walk
is a strict fixed sequence, stopping at the first failure to reject, a certificate is
returned if and only if the first and most selective threshold $\lambda_1$ is rejected.
That collapses the whole procedure to one hypergeometric event and makes the rate exact.

\begin{proof}[Proof of Proposition~\ref{prop:rate}]
In the notation of the statement, the claim is
\begin{equation}
\label{eq:rate}
\Pr[\textnormal{certify}]
 = \sum_{m} \Pr[m] \sum_{k \,:\, F_{\mathrm{hyp}}(k; N_1, k^{*}, m) \le \delta} \Pr[k \mid m].
\end{equation}
The walk tests $H_1, H_2, \dots$ in order and halts at the first hypothesis it fails to
reject, so it returns a certificate if and only if it rejects $H_1$; every deeper
threshold is reached only after $H_1$ has already been rejected. Rejection at
$\lambda_1$ occurs exactly when the exact test's $p$-value $F_{\mathrm{hyp}}(k; N_1, k^{*}, m)$ is at
most $\delta$. The calibration sample is a uniform draw of $n$ from the $N$
trajectories of $\Doff$, of which $N_1$ lie in $S(\lambda_1)$, so the number $m$ falling
inside the selection is $\mathrm{Hyp}(N, N_1, n)$. Conditioned on $m$, those $m$
trajectories are a uniform draw from $S(\lambda_1)$, which contains $U_1$ unsafe ones,
so $k \sim \mathrm{Hyp}(N_1, U_1, m)$. Summing the rejection indicator against
these two distributions gives Equation~\ref{eq:rate}. Both halves of
Proposition~\ref{prop:dichotomy} follow: if $u_1 > \alpha$ then $U_1 \ge k^{*}$ and the
inner sum is the level of the test, at most $\delta$; and if $u_1 < \alpha$ the rejection
event at $\lambda_1$ has probability tending to one, as shown in the proof of
Proposition~\ref{prop:dichotomy}. The first half in fact needs only $u_1 > \alpha$ rather
than $u_{\min} > \alpha$, since certification turns on $\mathcal{H}_1$ alone.
\end{proof}

Three things follow, the first two already stated in the main text. Proposition~\ref{prop:dichotomy} is recovered as a corollary: when
$u_1 > \alpha$ the inner sum is bounded by $\delta$, and when $u_1 < \alpha$ it tends to
one as $n \to N$. The scorer enters Equation~\ref{eq:rate} only through $U_1$. And the expression inverts: given a target rate, it returns the
calibration budget that achieves it, so a practitioner can price the labels before
spending them. On PointGoal1 it asks for $n \ge 360$ to reach a rate of $0.9$, where the
deployed $n = 200$ gives $0.66$.

Table~\ref{tab:certratepred} compares Equation~\ref{eq:rate} against the resampled
rates. Across the $100$ cells the mean absolute error is $0.019$ and the correlation is
$0.994$, with the residuals consistent with the Monte-Carlo error of a $200$-draw
estimate. Equation~\ref{eq:rate} is the operative form; a readable relaxation of it is not.
Writing $\epsilon = \alpha - u_1$ for the purity margin, three Hoeffding steps for
sampling without replacement give

\[
  \Pr[\textnormal{certify} \mid m] \;\ge\; 1 - \exp(-m\epsilon^{2}/2)
\]
 whenever $m \ge 2\ln(1/\delta)/\epsilon^{2}$.
That bound is vacuous on every cell here, because the first grid threshold keeps only
$15$ percent of the pool and so receives about $30$ of the $200$ calibration
trajectories, far short of what the condition asks. It is reported because it names the
governing quantity, $m\epsilon^{2}$: a wide margin on a thin selection certifies rarely,
which is why deepening the grid start does not help.

\begin{table}[h]
\caption{Certification rate, Equation~\ref{eq:rate} against the $200$-draw resample
(predicted/observed, averaged over three pipeline seeds). Across the $100$ cells with a first-threshold record
(of $108$; Hopper's seeds 1 and 2 have none at any budget) the mean absolute error is
$0.019$ and the Pearson correlation $0.994$.}
\label{tab:certratepred}
\begin{center}
\scriptsize
\begin{tabular}{l rrrr}
\toprule
Task & $n = 50$ & $n = 100$ & $n = 200$ & $n = 400$ \\
\midrule
HalfCheetah & 0.07/0.05 & 0.14/0.13 & 0.23/0.18 & 0.40/0.38 \\
Walker2d & 0.12/0.14 & 0.28/0.27 & 0.50/0.48 & 0.81/0.77 \\
Ant & 0.01/0.02 & 0.01/0.01 & 0.00/0.00 & 0.00/0.00 \\
Hopper & 0.00/0.00 & 0.00/0.00 & 0.00/0.00 & 0.00/0.00 \\
Swimmer & 0.04/0.05 & 0.07/0.10 & 0.10/0.08 & 0.13/0.14 \\
CarGoal1 & 0.08/0.09 & 0.18/0.18 & 0.34/0.31 & 0.59/0.67 \\
CarGoal2 & 0.05/0.05 & 0.10/0.08 & 0.15/0.16 & 0.24/0.27 \\
PointGoal1 & 0.14/0.14 & 0.36/0.36 & 0.65/0.73 & 0.93/0.94 \\
PointGoal2 & 0.03/0.01 & 0.05/0.04 & 0.05/0.04 & 0.05/0.03 \\
 
\bottomrule
\end{tabular}
\end{center}
\end{table}

\paragraph{Controlled contamination.} The nine pools offer one purity margin each, and
none sits near zero, so they cannot show whether the dichotomy is a threshold or merely
a trend. Figure~\ref{fig:contamination} varies the margin continuously by subsampling
each pool to eight target unsafe fractions, dropping unsafe trajectories to purify it or
safe ones to dirty it, recomputing the grid quantiles on the constructed pool as the
procedure would on a dataset of that shape, and running the identical calibration walk
on it with scores fixed. Across $198$ constructed pools spanning margins from $-0.35$ to
$+0.25$, the certification rate is essentially zero at every non-positive margin, with a
maximum of $0.08$ against $\delta = 0.1$, and rises through a sigmoid to one by a margin
of about $0.2$. Equation~\ref{eq:rate} tracks every point, with mean absolute error
$0.010$ and correlation $0.999$, and residuals consistent with the Monte-Carlo error of
the $200$-draw estimate. Eight of the nine tasks cross the threshold in both directions;
Swimmer's pool is too clean to be dirtied below zero without falling under the minimum
pool size and contributes positive margins only. Nothing is retrained: the sweep reuses
the deployed scorers and changes only what is in the pool.

\begin{figure}[h]
\begin{center}
\includegraphics[width=0.72\linewidth]{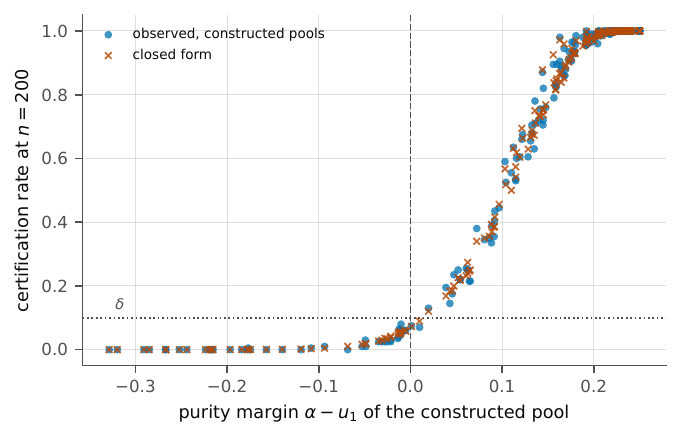}
\end{center}
\caption{Certification against controlled contamination. Each pool is subsampled to eight
target unsafe fractions and the full calibration walk is rerun on the result ($200$
draws, $n = 200$, scores fixed); the horizontal axis is the constructed pool's purity
margin $\alpha - u_1$. Crosses are Equation~\ref{eq:rate} evaluated on that pool's own
$(N, N_1, U_1)$. The rate is at most $0.08$ at every non-positive margin (dotted line at
$\delta$) and rises to one by a margin of about $0.2$; $198$ pools over nine tasks.}
\label{fig:contamination}
\end{figure}

\section{Full extraction sweep}
\label{app:extraction}

This appendix reports the sweeps behind Section~\ref{sec:analysis} in full: the value's
agreement with held-out preferences at two scales, the oracle-value substitutions, and
the per-transition configurations tested.

\begin{table}[h]
\caption{The value at two scales on the nine analysis tasks, three seeds each.
Accuracy ranges over tasks and is the fraction of held-out pairs the value
orders the same way as their ground-truth summed costs, on pairs drawn from the
bottom and top cost quartiles. The per-transition rows range over every
seed pair of every task, on 5000 matched transitions.}
\label{tab:obstacle2}
\begin{center}
\small
\begin{tabular}{llr}
\toprule
Scale & Quantity & Range \\
\midrule
Trajectory & held-out segment-pair accuracy & $0.798$ to $1.000$ \\
Per-transition & cross-seed advantage rank correlation $\rho$ & $0.79$ to $0.99$ \\
Per-transition & top-1-percent transition overlap (Jaccard) & $0.39$ to $0.92$ \\
 
\bottomrule
\end{tabular}
\end{center}
\end{table}

\begin{table}[t]
\caption{Oracle-value ablation on all nine analysis tasks with the BC-All reference (oracles: three seeds; learned value: five seeds; bold marks mean cost within budget).}
\label{tab:oracle}
\begin{center}
\scriptsize
\setlength{\tabcolsep}{3pt}
\resizebox{\textwidth}{!}{%
\begin{tabular}{l rr rr rr rr}
\toprule
& \multicolumn{2}{c}{BC-All} & \multicolumn{2}{c}{Learned $\Vbar$} & \multicolumn{2}{c}{Oracle, per-step} & \multicolumn{2}{c}{Oracle, cost-to-go} \\
\cmidrule(lr){2-3}\cmidrule(lr){4-5}\cmidrule(lr){6-7}\cmidrule(lr){8-9}
Task & $R$ & $C$ & $R$ & $C$ & $R$ & $C$ & $R$ & $C$ \\
\midrule
HalfCheetah & 2596 & 106 & 544\,\scriptsize$\pm$685 & \textbf{0\,\scriptsize$\pm$0} & 2721\,\scriptsize$\pm$68 & 235\,\scriptsize$\pm$192 & 2734\,\scriptsize$\pm$73 & 297\,\scriptsize$\pm$201 \\
Walker2d & 2698 & \textbf{10} & 2724\,\scriptsize$\pm$65 & \textbf{7\,\scriptsize$\pm$5} & 2699\,\scriptsize$\pm$52 & 24\,\scriptsize$\pm$11 & 2684\,\scriptsize$\pm$52 & 50\,\scriptsize$\pm$59 \\
Ant & 2876 & 175 & 2293\,\scriptsize$\pm$84 & \textbf{0\,\scriptsize$\pm$0} & 2929\,\scriptsize$\pm$30 & 169\,\scriptsize$\pm$67 & 2948\,\scriptsize$\pm$28 & 251\,\scriptsize$\pm$77 \\
Hopper & 993 & 131 & 76\,\scriptsize$\pm$1 & \textbf{18\,\scriptsize$\pm$1} & 817\,\scriptsize$\pm$267 & 85\,\scriptsize$\pm$3 & 968\,\scriptsize$\pm$166 & 64\,\scriptsize$\pm$32 \\
Swimmer & 120 & 92 & 15\,\scriptsize$\pm$9 & \textbf{1\,\scriptsize$\pm$1} & 99\,\scriptsize$\pm$1 & \textbf{19\,\scriptsize$\pm$1} & 92\,\scriptsize$\pm$14 & 56\,\scriptsize$\pm$34 \\
CarGoal1 & 15.59 & \textbf{22.44} & 10.19\,\scriptsize$\pm$1.72 & \textbf{14.00\,\scriptsize$\pm$0.11} & 16.16\,\scriptsize$\pm$1.34 & \textbf{22.90\,\scriptsize$\pm$2.11} & 16.02\,\scriptsize$\pm$1.02 & \textbf{19.43\,\scriptsize$\pm$0.52} \\
CarGoal2 & 6.00 & 38.49 & 6.33\,\scriptsize$\pm$0.06 & 29.40\,\scriptsize$\pm$1.25 & 7.66\,\scriptsize$\pm$0.22 & 37.66\,\scriptsize$\pm$3.21 & 8.09\,\scriptsize$\pm$0.61 & 35.65\,\scriptsize$\pm$0.46 \\
PointGoal1 & 17.46 & 36.94 & 13.72\,\scriptsize$\pm$1.78 & \textbf{22.11\,\scriptsize$\pm$4.63} & 18.42\,\scriptsize$\pm$0.59 & 34.98\,\scriptsize$\pm$2.24 & 18.11\,\scriptsize$\pm$1.07 & 32.81\,\scriptsize$\pm$3.59 \\
PointGoal2 & 14.97 & 76.58 & 6.66\,\scriptsize$\pm$0.50 & 50.02\,\scriptsize$\pm$4.94 & 15.35\,\scriptsize$\pm$0.88 & 83.69\,\scriptsize$\pm$0.62 & 14.49\,\scriptsize$\pm$0.95 & 73.62\,\scriptsize$\pm$7.18 \\
 
\bottomrule
\end{tabular}}
\end{center}
\end{table}

The full oracle-value ablation is in Table~\ref{tab:oracle}. No cell has an oracle improving on the learned value: the closest is CarGoal1, where the discounted cost-to-go oracle reaches cost $19.43$ against the learned value's $14.00$, both within budget. On Swimmer the learned value is the safest configuration in the table ($1$ against $19$ and $56$ for the two oracles). Table~\ref{tab:extraction} reports the per-transition operator sweep on the nine analysis tasks (three seeds; the exp/$h{=}1$ reference appears in Table~\ref{tab:oracle}). Columns: rank and binary weight transforms; counterfactual scoring with the learned value (per-state-normalized softmax); and the action-conditioned oracle, a value distilled from full ground-truth cost-to-go and scored counterfactually through the dynamics ensemble. On the failing navigation tasks (CarGoal2, PointGoal1, PointGoal2), no weight transform escapes the cost plateau, and neither does
the action-conditioned oracle, which consumes dense cost labels unavailable under preference
supervision: it stays over budget on all three ($39.5$, $31.5$ and $49.6$ against $25$) while
raising reward above every weight transform. Varying the extraction temperature rather than the transform does reach the budget, on PointGoal1 and on neither of the other two (Appendix~\ref{app:t32t26}). On the velocity tasks the transforms trade reward for cost without dominating the exponential reference, and on CarGoal1, the one navigation task where per-transition extraction is cost-safe throughout, the weighted variants attain somewhat higher reward at two to three times the selection route's cost.

\paragraph{Matched-supervision disagreement control.} To separate the Bradley-Terry null space from ordinary seed variance, we compute the same cross-seed advantage statistics for cost-to-go ensembles trained with dense per-step supervision (three seeds, identical architecture, 5000 fixed transitions per task). Disagreement is not specific to preference training. On all nine analysis tasks the supervised ensembles agree less by advantage rank correlation than the preference ensembles on the identical sampled transitions (task-mean $\rho$ of $0.43$ to $0.82$ against $0.79$ to $0.99$), and their top-one-percent overlaps are lower throughout ($0.31$ to $0.67$ against $0.42$ to $0.90$). The null-space argument still explains why preference training cannot identify per-transition values; the control shows that dense supervision does not identify them either at these sample sizes, so the granular scale is unreliable under both supervisions, which is the fact the design responds to. The identification deficit is quantifiable. The 1000 training pairs constrain 1.4 to 4.2 percent of each pool's states, the drawn segments overlapping on 1 to 10 percent of the slots they cover, so roughly 97 percent of visited states receive no direct constraint and, within the covered states, each appears in about one segment sum; the score's behavior everywhere else is generalization, which is exactly what the oracle comparison of Section~\ref{sec:analysis} finds to be the essential signal.

\begin{table}[h]
\caption{Extraction operator sweep on the nine analysis tasks (three seeds, mean $\pm$ 95\% CI; bold marks mean cost within budget). The oracle substitutions and the exponential reference appear in Table~\ref{tab:oracle}.}
\label{tab:extraction}
\begin{center}
\scriptsize
\setlength{\tabcolsep}{3pt}
\resizebox{\textwidth}{!}{%
\begin{tabular}{l rr rr rr rr}
\toprule
& \multicolumn{2}{c}{Rank} & \multicolumn{2}{c}{Binary} & \multicolumn{2}{c}{CF (learned)} & \multicolumn{2}{c}{CF (AC oracle)} \\
\cmidrule(lr){2-3}\cmidrule(lr){4-5}\cmidrule(lr){6-7}\cmidrule(lr){8-9}
Task & $R$ & $C$ & $R$ & $C$ & $R$ & $C$ & $R$ & $C$ \\
\midrule
HalfCheetah & 2197\,\scriptsize$\pm$315 & \textbf{15\,\scriptsize$\pm$19} & 1038\,\scriptsize$\pm$503 & \textbf{0\,\scriptsize$\pm$1} & 1598\,\scriptsize$\pm$371 & \textbf{2\,\scriptsize$\pm$3} & 1889\,\scriptsize$\pm$11 & \textbf{0\,\scriptsize$\pm$0} \\
Walker2d & 2662\,\scriptsize$\pm$44 & 21\,\scriptsize$\pm$10 & 2302\,\scriptsize$\pm$352 & 55\,\scriptsize$\pm$13 & 2715\,\scriptsize$\pm$28 & 21\,\scriptsize$\pm$25 & 2683\,\scriptsize$\pm$23 & \textbf{0\,\scriptsize$\pm$0} \\
Ant & 2316\,\scriptsize$\pm$96 & \textbf{1\,\scriptsize$\pm$0} & 2017\,\scriptsize$\pm$341 & \textbf{0\,\scriptsize$\pm$0} & 2553\,\scriptsize$\pm$133 & \textbf{2\,\scriptsize$\pm$1} & 2881\,\scriptsize$\pm$15 & \textbf{5\,\scriptsize$\pm$2} \\
Hopper & 97\,\scriptsize$\pm$26 & 22\,\scriptsize$\pm$3 & 118\,\scriptsize$\pm$57 & 23\,\scriptsize$\pm$6 & 153\,\scriptsize$\pm$62 & 29\,\scriptsize$\pm$8 & 370\,\scriptsize$\pm$122 & \textbf{14\,\scriptsize$\pm$11} \\
Swimmer & 12\,\scriptsize$\pm$8 & \textbf{1\,\scriptsize$\pm$1} & 11\,\scriptsize$\pm$3 & \textbf{1\,\scriptsize$\pm$1} & 27\,\scriptsize$\pm$7 & \textbf{2\,\scriptsize$\pm$0} & 89\,\scriptsize$\pm$5 & 32\,\scriptsize$\pm$16 \\
CarGoal1 & 10.73\,\scriptsize$\pm$2.27 & \textbf{15.33\,\scriptsize$\pm$0.61} & 12.34\,\scriptsize$\pm$0.27 & \textbf{17.59\,\scriptsize$\pm$0.90} & 16.48\,\scriptsize$\pm$0.79 & \textbf{21.93\,\scriptsize$\pm$1.99} & 17.28\,\scriptsize$\pm$0.36 & \textbf{18.98\,\scriptsize$\pm$1.33} \\
CarGoal2 & 7.09\,\scriptsize$\pm$1.14 & 33.74\,\scriptsize$\pm$5.88 & 6.02\,\scriptsize$\pm$0.54 & 26.85\,\scriptsize$\pm$3.34 & 8.53\,\scriptsize$\pm$0.03 & 39.97\,\scriptsize$\pm$3.45 & 9.70\,\scriptsize$\pm$0.30 & 39.46\,\scriptsize$\pm$0.67 \\
PointGoal1 & 16.42\,\scriptsize$\pm$0.42 & 29.85\,\scriptsize$\pm$3.88 & 15.48\,\scriptsize$\pm$0.33 & 28.06\,\scriptsize$\pm$2.67 & 16.71\,\scriptsize$\pm$1.71 & 31.76\,\scriptsize$\pm$6.54 & 19.08\,\scriptsize$\pm$0.61 & 31.54\,\scriptsize$\pm$1.60 \\
PointGoal2 & 8.47\,\scriptsize$\pm$1.02 & 53.88\,\scriptsize$\pm$4.53 & 6.99\,\scriptsize$\pm$1.09 & 60.75\,\scriptsize$\pm$7.70 & 8.20\,\scriptsize$\pm$0.23 & 58.68\,\scriptsize$\pm$7.08 & 12.15\,\scriptsize$\pm$1.48 & 49.57\,\scriptsize$\pm$5.59 \\
 
\bottomrule
\end{tabular}}
\end{center}
\end{table}

\paragraph{Probes behind the obstacle.} Two probes support the granularity reading in Section~\ref{sec:analysis}. Multi-step advantages, formed over an $h$-step window, degrade safety monotonically in $h$: at $h = 5$ cost worsens on eight of the nine analysis tasks relative to $h = 1$, HalfCheetah moving from $0.3$ to $15.6$. Applying the same exponentiated reweighting on top of the ground-truth-filtered BC-Safe policy destroys its safety on five of the seven tasks where that policy satisfies the budget, so the failure is not a deficiency of the learned value: it survives substituting a policy trained on exactly the safe data. On the failing navigation tasks every value source, learned or oracle, sits in the same cost band well above budget.
\paragraph{The value at two scales.} Table~\ref{tab:obstacle2} gives the ranges behind the identification caveat of Section~\ref{sec:analysis}. Trajectory-level held-out accuracy runs from $0.798$ to $1.000$, while cross-seed advantage rank correlation spans $0.79$ to $0.99$ and top-one-percent overlap $0.39$ to $0.92$: ranking whole trajectories is reliable everywhere, and the tail that exponential weighting concentrates on is not. The diagnostics do not predict the policy, and the two extremes make the point. Swimmer has the strongest cross-seed agreement of the nine ($\rho = 0.99$, overlap $0.90$) and held-out accuracy $0.996$, yet its deployed clone is the worst of the suite at cost $58.9$ against budget $20$; PointGoal2 has the weakest agreement ($\rho = 0.79$) and stays within budget at $20.8$ against $25$, as does CarGoal2, whose held-out accuracy is the lowest of the nine at $0.798$, at $19.3$. Across the nine tasks the deployed filter is safe on six of the seven with $\rho \geq 0.85$ and on both below it, so agreement neither separates the safe tasks from the unsafe ones nor orders them.

\section{Extraction temperature, weight clip, and score aggregation}
\label{app:t32t26}
Two sweeps were registered with committed predictions before they were run. Both
are reported here, whether or not the prediction survived.

\paragraph{Extraction temperature and weight clip.} Table~\ref{tab:t32sweep}
varies the advantage-weighted extraction's temperature $\beta$ over
$\{0.1, 0.3, 3\}$ at clip $20$, and the weight clip over $\{5, 100\}$ at
$\beta = 1$, on the nine analysis tasks at three seeds. All six cells share one
trainer and one seed protocol, so the columns are directly comparable; the
deployed arm elsewhere in the paper uses an ensemble value and five seeds and is
not a controlled comparison against them. The registered
prediction was that no setting reaches the budget on the three failing
navigation tasks, since the weight range the corridor argument requires grows
with the horizon and the sweep's largest ratio is far below it. Sharpening the
weights trades cost against reward monotonically, and the prediction holds on
two of the three. On CarGoal2 no setting falls below $29.0$ and on PointGoal2
none below $41.8$, against a budget of $25$. On PointGoal1 it fails. Cost falls monotonically as the weights sharpen, from
$30.53 \pm 1.67$ at $\beta = 3$ to $22.05 \pm 1.33$ at $\beta = 0.3$ and
$16.40 \pm 0.71$ at $\beta = 0.1$, the last two inside the budget. The safety is
paid for in reward, which falls from $17.53$ to $10.80$ across that range. We therefore scope the obstacle rather than
state it without qualification. Per-transition reweighting reaches the budget at
no swept setting on two of the three failing navigation tasks, and on the third
only by sharpening far enough to give up a third of the reward, where it still
runs at nearly twice the selection route's cost ($16.40$ against $8.80$) though
at higher reward ($10.80$ against $7.47$). This is the behavior the corridor
construction anticipates once real states partially record consequences, which
is where the benchmark departs from the idealized corridor.

\begin{table}[h]
\caption{Extraction temperature and weight clip on the nine analysis tasks
(three seeds, mean $\pm$ 95\% CI; bold marks mean cost within budget). These are
the five registered cells, all from one trainer at one seed protocol.}
\label{tab:t32sweep}
\begin{center}
\scriptsize
\setlength{\tabcolsep}{3pt}
\resizebox{\textwidth}{!}{%
\begin{tabular}{l rr rr rr rr rr}
\toprule
& \multicolumn{2}{c}{$\beta{=}0.1$} & \multicolumn{2}{c}{$\beta{=}0.3$}
& \multicolumn{2}{c}{$\beta{=}3$}
& \multicolumn{2}{c}{clip $5$} & \multicolumn{2}{c}{clip $100$} \\
\cmidrule(lr){2-3}\cmidrule(lr){4-5}\cmidrule(lr){6-7}\cmidrule(lr){8-9}\cmidrule(lr){10-11}
Task & $R$ & $C$ & $R$ & $C$ & $R$ & $C$ & $R$ & $C$ & $R$ & $C$ \\
\midrule
HalfCheetah & 714\,\scriptsize$\pm$95 & \textbf{0\,\scriptsize$\pm$0} & 1514\,\scriptsize$\pm$583 & \textbf{8\,\scriptsize$\pm$12} & 2731\,\scriptsize$\pm$13 & 173\,\scriptsize$\pm$67 & 2645\,\scriptsize$\pm$34 & 177\,\scriptsize$\pm$199 & 2636\,\scriptsize$\pm$40 & 155\,\scriptsize$\pm$175 \\
Walker2d & 2459\,\scriptsize$\pm$333 & 39\,\scriptsize$\pm$45 & 2750\,\scriptsize$\pm$2 & 22\,\scriptsize$\pm$32 & 2637\,\scriptsize$\pm$46 & 30\,\scriptsize$\pm$32 & 2678\,\scriptsize$\pm$69 & \textbf{12\,\scriptsize$\pm$12} & 2665\,\scriptsize$\pm$34 & \textbf{15\,\scriptsize$\pm$16} \\
Ant & 2429\,\scriptsize$\pm$117 & \textbf{1\,\scriptsize$\pm$1} & 2936\,\scriptsize$\pm$18 & 46\,\scriptsize$\pm$24 & 2914\,\scriptsize$\pm$49 & 126\,\scriptsize$\pm$19 & 2931\,\scriptsize$\pm$21 & 91\,\scriptsize$\pm$25 & 2919\,\scriptsize$\pm$16 & 95\,\scriptsize$\pm$28 \\
Hopper & 245\,\scriptsize$\pm$263 & 41\,\scriptsize$\pm$35 & 306\,\scriptsize$\pm$83 & 42\,\scriptsize$\pm$4 & 638\,\scriptsize$\pm$197 & 64\,\scriptsize$\pm$11 & 741\,\scriptsize$\pm$190 & 74\,\scriptsize$\pm$8 & 845\,\scriptsize$\pm$345 & 72\,\scriptsize$\pm$22 \\
Swimmer & 35\,\scriptsize$\pm$12 & \textbf{2\,\scriptsize$\pm$2} & 81\,\scriptsize$\pm$12 & \textbf{1\,\scriptsize$\pm$1} & 113\,\scriptsize$\pm$11 & 43\,\scriptsize$\pm$45 & 111\,\scriptsize$\pm$6 & 24\,\scriptsize$\pm$20 & 115\,\scriptsize$\pm$6 & \textbf{14\,\scriptsize$\pm$9} \\
CarGoal1 & 10.95\,\scriptsize$\pm$0.80 & \textbf{15.39\,\scriptsize$\pm$2.56} & 11.75\,\scriptsize$\pm$0.97 & \textbf{14.61\,\scriptsize$\pm$3.75} & 15.99\,\scriptsize$\pm$0.53 & \textbf{22.90\,\scriptsize$\pm$2.31} & 14.05\,\scriptsize$\pm$0.38 & \textbf{16.51\,\scriptsize$\pm$1.02} & 15.42\,\scriptsize$\pm$0.62 & \textbf{19.59\,\scriptsize$\pm$1.76} \\
CarGoal2 & 7.34\,\scriptsize$\pm$0.53 & 32.18\,\scriptsize$\pm$0.46 & 6.84\,\scriptsize$\pm$0.57 & 29.04\,\scriptsize$\pm$0.73 & 7.02\,\scriptsize$\pm$0.19 & 31.59\,\scriptsize$\pm$3.04 & 7.22\,\scriptsize$\pm$0.78 & 35.02\,\scriptsize$\pm$5.44 & 7.39\,\scriptsize$\pm$0.66 & 32.08\,\scriptsize$\pm$3.07 \\
PointGoal1 & 10.80\,\scriptsize$\pm$1.03 & \textbf{16.40\,\scriptsize$\pm$0.71} & 14.31\,\scriptsize$\pm$0.96 & \textbf{22.05\,\scriptsize$\pm$1.33} & 17.53\,\scriptsize$\pm$0.33 & 30.53\,\scriptsize$\pm$1.67 & 16.79\,\scriptsize$\pm$0.75 & 30.33\,\scriptsize$\pm$2.92 & 17.26\,\scriptsize$\pm$0.98 & 29.67\,\scriptsize$\pm$1.81 \\
PointGoal2 & 7.40\,\scriptsize$\pm$0.26 & 41.84\,\scriptsize$\pm$12.44 & 10.31\,\scriptsize$\pm$0.72 & 49.59\,\scriptsize$\pm$6.69 & 14.43\,\scriptsize$\pm$0.96 & 71.62\,\scriptsize$\pm$9.50 & 13.23\,\scriptsize$\pm$0.98 & 66.82\,\scriptsize$\pm$10.09 & 13.88\,\scriptsize$\pm$0.71 & 71.98\,\scriptsize$\pm$4.04 \\
 
\bottomrule
\end{tabular}}
\end{center}
\end{table}

\paragraph{Score aggregation.} Table~\ref{tab:t26agg} replaces the deployed mean
aggregator with the minimum and the tenth percentile of the per-state value
along a trajectory, filtering at the calibration-estimated fraction. Both
registered predictions hold. The minimum inherits the score's unreliable extreme
tail and is safe on four of the nine tasks against the mean's eight. The tenth
percentile tracks the mean closely everywhere and matches its count of eight, on a
different set: it rescues Hopper, at cost $19.7$ against the mean's $33.0$ on a budget
of $20$, the task the registration named in advance as one where the mean-aggregated
filter fails, and it loses Swimmer.
That is a pre-registered improvement on a single task rather than a sweep-wide
gain, and it softens rather than removes the open problem of tail-robust
scoring.

\begin{table}[h]
\caption{Score aggregators on the nine analysis tasks, V-filter at the
calibration-estimated fraction (three seeds, mean $\pm$ 95\% CI; bold marks mean
cost within budget). The mean column is the deployed aggregator.}
\label{tab:t26agg}
\begin{center}
\scriptsize
\begin{tabular}{l rr rr rr}
\toprule
& \multicolumn{2}{c}{min} & \multicolumn{2}{c}{10th pct} & \multicolumn{2}{c}{mean (deployed)} \\
\cmidrule(lr){2-3}\cmidrule(lr){4-5}\cmidrule(lr){6-7}
Task & $R$ & $C$ & $R$ & $C$ & $R$ & $C$ \\
\midrule
HalfCheetah & 2795\,\scriptsize$\pm$81 & 454\,\scriptsize$\pm$186 & 2188\,\scriptsize$\pm$127 & \textbf{9\,\scriptsize$\pm$8} & 1866\,\scriptsize$\pm$385 & \textbf{1\,\scriptsize$\pm$1} \\
Walker2d & 2462\,\scriptsize$\pm$178 & \textbf{16\,\scriptsize$\pm$10} & 2680\,\scriptsize$\pm$14 & \textbf{0\,\scriptsize$\pm$0} & 2658\,\scriptsize$\pm$51 & \textbf{4\,\scriptsize$\pm$5} \\
Ant & 2857\,\scriptsize$\pm$89 & \textbf{20\,\scriptsize$\pm$13} & 2484\,\scriptsize$\pm$40 & \textbf{1\,\scriptsize$\pm$1} & 2465\,\scriptsize$\pm$38 & \textbf{2\,\scriptsize$\pm$1} \\
Hopper & 1130\,\scriptsize$\pm$382 & 126\,\scriptsize$\pm$65 & 607\,\scriptsize$\pm$136 & \textbf{20\,\scriptsize$\pm$14} & 1238\,\scriptsize$\pm$257 & 33\,\scriptsize$\pm$31 \\
Swimmer & 130\,\scriptsize$\pm$8 & 40\,\scriptsize$\pm$36 & 55\,\scriptsize$\pm$11 & 32\,\scriptsize$\pm$9 & 106\,\scriptsize$\pm$8 & \textbf{15\,\scriptsize$\pm$8} \\
CarGoal1 & 11.13\,\scriptsize$\pm$0.11 & \textbf{12.17\,\scriptsize$\pm$0.89} & 9.33\,\scriptsize$\pm$0.36 & \textbf{6.95\,\scriptsize$\pm$0.65} & 8.64\,\scriptsize$\pm$0.74 & \textbf{8.44\,\scriptsize$\pm$0.85} \\
CarGoal2 & 6.19\,\scriptsize$\pm$0.80 & 27.37\,\scriptsize$\pm$4.07 & 4.04\,\scriptsize$\pm$0.56 & \textbf{18.89\,\scriptsize$\pm$0.79} & 3.77\,\scriptsize$\pm$0.33 & \textbf{17.38\,\scriptsize$\pm$2.17} \\
PointGoal1 & 13.11\,\scriptsize$\pm$1.81 & \textbf{18.72\,\scriptsize$\pm$4.80} & 8.02\,\scriptsize$\pm$0.19 & \textbf{8.67\,\scriptsize$\pm$1.09} & 8.27\,\scriptsize$\pm$0.25 & \textbf{11.55\,\scriptsize$\pm$0.31} \\
PointGoal2 & 11.23\,\scriptsize$\pm$0.75 & 42.30\,\scriptsize$\pm$2.23 & 6.64\,\scriptsize$\pm$0.63 & \textbf{19.80\,\scriptsize$\pm$2.63} & 6.51\,\scriptsize$\pm$0.54 & \textbf{19.73\,\scriptsize$\pm$1.33} \\
 
\bottomrule
\end{tabular}
\end{center}
\end{table}

\section{Value diagnostics: anticipation, traces, and the safety landscape}
\label{app:landscape}
Figure~\ref{fig:interp} reports the two temporal diagnostics behind the inspection paragraph of Section~\ref{sec:analysis}. Binned by steps until the next violation, the ensemble mean falls by two to four within-task standard deviations between the most distant bin and imminent violation on all nine tasks; individual bins reverse by up to 0.9 standard deviations, so the signal is a graded early warning rather than a strictly ordered countdown. In the single-trajectory trace, per-state values swing across several units while their rolling mean moves stably through the episode's cost-event bands, which is Observation~\ref{obs:granularity} at the resolution of one episode. The same network is noisy per state and reliable in aggregate.

\begin{figure}[h]
\begin{center}
\includegraphics[width=0.98\linewidth]{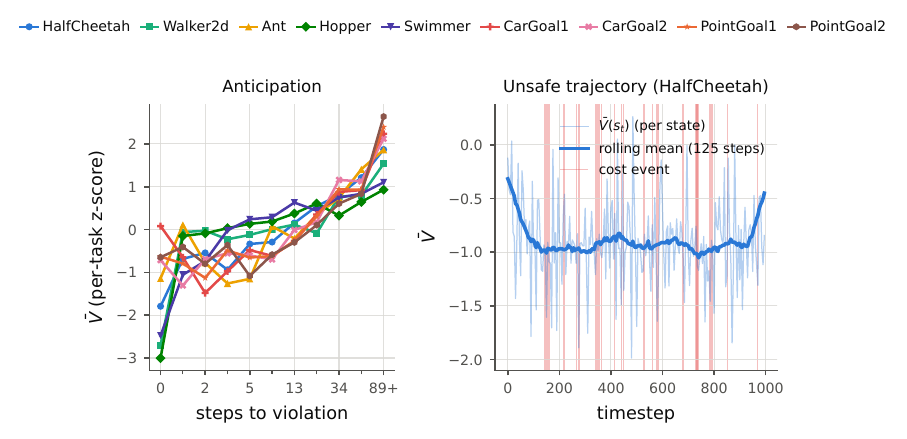}
\end{center}
\caption{Left: $\Vbar$ declines as violations approach (per-task z-scores). Right: per-state values along one unsafe trajectory (episodic cost 40, budget 20) are noisy while their rolling mean is stable.}
\label{fig:interp}
\end{figure}

Figure~\ref{fig:landscape-layouts} renders the learned value over the workspace on three independently sampled PointGoal1 layouts with the same value network (100 rollouts per layout; BC-All and BC-Safe policies with action-noise episodes for coverage). The low-value basin tracks the hazard field in each layout, with state-level Spearman correlations of $0.43$, $0.44$, and $0.20$ between $\Vbar$ and distance to the nearest hazard, the values labeled in the figure. The alignment largely collapses when the same states are frozen at zero velocity ($\rho$ of $0.09$, $0.27$, and $0.15$ on layouts 7, 11 and 23), so the value encodes proximity and motion jointly on the visited distribution rather than position alone. Correlations computed on cell means rather than states are weaker ($0.17$, $0.29$, $0.22$): binning discards the visit density that concentrates states in the region where the value varies most, so each cell counts equally regardless of how well the behavior data supports it. Two secondary structures are visible. Hazard clusters produce deeper basins than isolated hazards, which the behavior data rarely visits. And in the second layout the goal itself lies inside the hazard field, since reaching it requires transiting low-value states, the geometric signature of tasks where any safety filter must either sacrifice reward or refuse.

\begin{figure}[h]
\begin{center}
\includegraphics[width=0.98\linewidth]{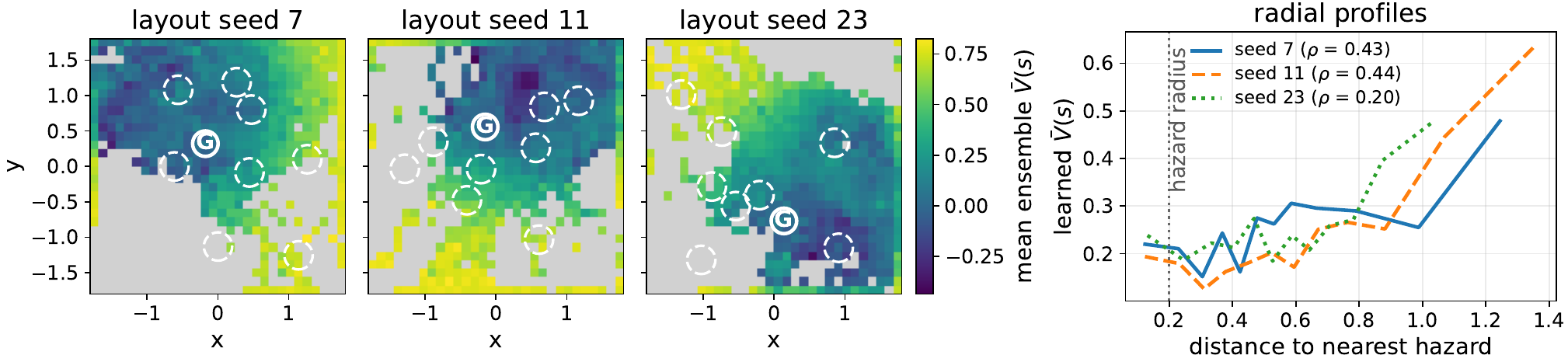}
\end{center}
\caption{Safety landscape on three independently sampled PointGoal1 layouts (reset seeds 7, 11, 23), same value network. Mean ensemble $\Vbar$ over the states visited by 100 rollouts per layout, alternating the BC-All and BC-Safe behavior policies with action noise so that coverage spans both the hazard field and the safe region; cells with fewer than 25 visits are left gray. Dashed circles are ground-truth hazards, which the model never sees, and G marks the goal. Right: $\Vbar$ in distance bins, labeled with the state-level Spearman correlation per layout.}
\label{fig:landscape-layouts}
\end{figure}

\section{External baselines and labeling cost}
\label{app:baselines}

This appendix reports the external baselines per task, in normalized and in raw units,
and places every method by the supervision it consumes.

\begin{table}[t]
\caption{Normalized cost ($C_n = C/\text{budget}$; safe when $C_n \le 1$, bold) across all methods in the study, with CDT at its budget-matched cost target (its best target per task, two further safe tasks, is in Appendix~\ref{app:cdtsweep}), for at-a-glance comparison in the benchmark's standard units; rewards and intervals appear in Tables~\ref{tab:main} and~\ref{tab:baselines}.}
\label{tab:normcost}
\begin{center}
\scriptsize
\setlength{\tabcolsep}{3pt}
\resizebox{\textwidth}{!}{%
\begin{tabular}{l rrrrr rrr r}
\toprule
Task & BC-All & BC-Safe & Safe-Seg & V-filter & Calib. & CDT & CPQ & COpt. & CPL \\
\midrule
HalfCheetah & 5.31 & \textbf{0.17} & \textbf{0.56} & \textbf{0.05} & \textbf{0.13} & \textbf{0.27} & 2.08 & \textbf{0.00} & \textbf{0.02} \\
Walker2d & \textbf{0.50} & \textbf{0.26} & 1.64 & \textbf{0.21} & \textbf{0.05} & \textbf{0.05} & \textbf{0.31} & 2.09 & \textbf{0.87} \\
Ant & 8.73 & \textbf{0.29} & 1.29 & \textbf{0.08} & \textbf{0.07} & \textbf{0.61} & \textbf{0.00} & 5.48 & \textbf{0.28} \\
Hopper & 6.54 & \textbf{0.89} & \textbf{0.73} & 1.65 & \textbf{0.17} & \textbf{0.70} & 4.66 & 4.13 & 1.89 \\
Swimmer & 4.60 & 1.84 & 1.33 & \textbf{0.73} & 2.94 & 1.16 & \textbf{0.01} & 10.50 & 1.69 \\
CarGoal1 & \textbf{0.90} & \textbf{0.36} & 1.23 & \textbf{0.34} & \textbf{0.31} & 1.47 & 1.44 & \textbf{0.64} & \textbf{0.96} \\
CarGoal2 & 1.54 & \textbf{0.66} & 1.85 & \textbf{0.70} & \textbf{0.77} & 2.61 & 3.70 & 1.29 & 2.05 \\
PointGoal1 & 1.48 & \textbf{0.50} & 1.55 & \textbf{0.46} & \textbf{0.35} & 1.47 & \textbf{0.46} & 2.70 & 1.63 \\
PointGoal2 & 3.06 & \textbf{0.85} & 3.23 & \textbf{0.79} & \textbf{0.83} & 2.40 & 1.25 & 2.50 & 3.08 \\
PointButton1 & 2.04 & \textbf{0.80} & 2.12 & \textbf{0.75} & \textbf{0.57} & 4.66 & 4.60 & 1.40 & 1.97 \\
PointButton2 & 2.66 & 1.32 & 2.91 & 1.21 & 1.13 & 3.70 & 5.20 & 1.81 & 2.57 \\
CarButton1 & 1.74 & \textbf{0.89} & 3.20 & \textbf{0.54} & \textbf{0.58} & 4.19 & 15.60 & 1.13 & 2.47 \\
CarButton2 & 1.81 & \textbf{0.89} & 2.54 & \textbf{0.74} & \textbf{0.76} & 3.91 & 18.41 & 1.85 & 2.38 \\
PointCircle1 & 6.24 & \textbf{0.19} & 7.19 & \textbf{0.41} & \textbf{0.70} & \textbf{0.24} & 1.22 & 7.69 & 6.27 \\
PointCircle2 & 6.35 & 1.77 & 6.51 & 3.30 & 3.49 & 1.12 & 4.45 & 11.65 & 8.99 \\
 
\bottomrule
\end{tabular}}
\end{center}
\end{table}
Table~\ref{tab:baselines} reports the full-label baselines and CPL per task; Figure~\ref{fig:pareto} summarizes the study by the kind of supervisor each method assumes against tasks kept within budget. Two details support the main-text reading. CDT's navigation costs run from roughly 1.5 times budget on the Goal tasks to nearly 5 times on the Button tasks, and Appendix~\ref{app:cdtsweep} shows they stay there at every conditioning target. CPQ's safety is largely degenerate, pairing near-zero cost with collapsed reward on three of its four safe cells (Walker2d, Ant, Swimmer); PointGoal1 is its one safe cell with intact reward. COptiDICE's two safe cells, HalfCheetah and CarGoal1, give up 23 to 36 percent of BC-All's reward. In the figure, full-label methods are placed at their best configuration, CDT at its safest target per task; no external baseline exceeds the calibrated filter's twelve safe tasks, CDT reaching seven of fifteen, while BC-Safe reaches twelve on ground-truth labels. BC-Safe-Seg, which clones the ground-truth-safe \emph{segments} rather than trajectories, is safe on two of fifteen despite full labels. Selection at segment scale fragments the multi-step avoidance behaviors that carry safety, a selection-level counterpart of the granularity principle at the scale between per-transition reweighting and whole-trajectory selection. Table~\ref{tab:normcost} gives the normalized-cost comparison across all methods; normalized reward $R_n$ uses dataset return extrema and is recoverable from the raw values reported here.

\begin{table}[h]
\caption{External baselines: full-label offline safe RL trained in-codebase (three seeds; CDT is reported at its budget-matched cost target, and its full cost-target sweep, including its safest target per task, is in Appendix~\ref{app:cdtsweep}), CPL on our exact preference data, and BC-Safe-Seg, the segment-level full-label clone (five seeds). Bold cost cells are within budget.}
\label{tab:baselines}
\begin{center}
\scriptsize
\setlength{\tabcolsep}{3.5pt}
\resizebox{\textwidth}{!}{%
\begin{tabular}{l rr rr rr rr rr}
\toprule
& \multicolumn{2}{c}{CDT} & \multicolumn{2}{c}{CPQ} & \multicolumn{2}{c}{COptiDICE} & \multicolumn{2}{c}{CPL} & \multicolumn{2}{c}{BC-Safe-Seg} \\
\cmidrule(lr){2-3}\cmidrule(lr){4-5}\cmidrule(lr){6-7}\cmidrule(lr){8-9}\cmidrule(lr){10-11}
Task & $R$ & $C$ & $R$ & $C$ & $R$ & $C$ & $R$ & $C$ & $R$ & $C$ \\
\midrule
HalfCheetah & 2756\,\scriptsize$\pm$19 & \textbf{5\,\scriptsize$\pm$3} & 1037\,\scriptsize$\pm$663 & 42\,\scriptsize$\pm$45 & 1716\,\scriptsize$\pm$75 & \textbf{0\,\scriptsize$\pm$0} & 1991\,\scriptsize$\pm$74 & \textbf{0\,\scriptsize$\pm$0} & 2588\,\scriptsize$\pm$20 & \textbf{11\,\scriptsize$\pm$9} \\
Walker2d & 2712\,\scriptsize$\pm$13 & \textbf{1\,\scriptsize$\pm$1} & 33\,\scriptsize$\pm$94 & \textbf{6\,\scriptsize$\pm$9} & 518\,\scriptsize$\pm$98 & 42\,\scriptsize$\pm$15 & 2722\,\scriptsize$\pm$27 & \textbf{17\,\scriptsize$\pm$16} & 2103\,\scriptsize$\pm$251 & 33\,\scriptsize$\pm$13 \\
Ant & 2946\,\scriptsize$\pm$10 & \textbf{12\,\scriptsize$\pm$3} & -3000\,\scriptsize$\pm$0 & \textbf{0\,\scriptsize$\pm$0} & 2959\,\scriptsize$\pm$22 & 110\,\scriptsize$\pm$37 & 2673\,\scriptsize$\pm$61 & \textbf{6\,\scriptsize$\pm$4} & 2931\,\scriptsize$\pm$24 & 26\,\scriptsize$\pm$8 \\
Hopper & 1457\,\scriptsize$\pm$194 & \textbf{14\,\scriptsize$\pm$12} & 640\,\scriptsize$\pm$346 & 93\,\scriptsize$\pm$75 & 386\,\scriptsize$\pm$177 & 83\,\scriptsize$\pm$56 & 596\,\scriptsize$\pm$440 & 38\,\scriptsize$\pm$18 & 978\,\scriptsize$\pm$180 & \textbf{15\,\scriptsize$\pm$6} \\
Swimmer & 162\,\scriptsize$\pm$1 & 23\,\scriptsize$\pm$6 & 4\,\scriptsize$\pm$7 & \textbf{0\,\scriptsize$\pm$0} & 130\,\scriptsize$\pm$33 & 210\,\scriptsize$\pm$168 & 62\,\scriptsize$\pm$16 & 34\,\scriptsize$\pm$24 & 97\,\scriptsize$\pm$5 & 27\,\scriptsize$\pm$20 \\
CarGoal1 & 26.06\,\scriptsize$\pm$1.56 & 36.87\,\scriptsize$\pm$10.00 & 31.23\,\scriptsize$\pm$2.61 & 36.00\,\scriptsize$\pm$0.85 & 15.86\,\scriptsize$\pm$4.68 & \textbf{16.00\,\scriptsize$\pm$8.10} & 18.33\,\scriptsize$\pm$0.76 & \textbf{24.01\,\scriptsize$\pm$2.73} & 20.94\,\scriptsize$\pm$0.53 & 30.82\,\scriptsize$\pm$2.60 \\
CarGoal2 & 14.92\,\scriptsize$\pm$3.14 & 65.37\,\scriptsize$\pm$13.80 & 14.21\,\scriptsize$\pm$6.25 & 92.57\,\scriptsize$\pm$26.90 & 7.37\,\scriptsize$\pm$1.18 & 32.23\,\scriptsize$\pm$4.70 & 10.57\,\scriptsize$\pm$0.54 & 51.27\,\scriptsize$\pm$4.45 & 10.75\,\scriptsize$\pm$0.45 & 46.24\,\scriptsize$\pm$1.19 \\
PointGoal1 & 20.17\,\scriptsize$\pm$1.56 & 36.77\,\scriptsize$\pm$3.75 & 18.68\,\scriptsize$\pm$2.46 & \textbf{11.60\,\scriptsize$\pm$7.60} & 13.83\,\scriptsize$\pm$2.88 & 67.50\,\scriptsize$\pm$29.90 & 14.19\,\scriptsize$\pm$0.49 & 40.65\,\scriptsize$\pm$1.96 & 21.31\,\scriptsize$\pm$0.09 & 38.75\,\scriptsize$\pm$2.29 \\
PointGoal2 & 13.78\,\scriptsize$\pm$2.82 & 59.97\,\scriptsize$\pm$16.05 & 11.85\,\scriptsize$\pm$5.77 & 31.33\,\scriptsize$\pm$12.60 & 10.22\,\scriptsize$\pm$0.55 & 62.40\,\scriptsize$\pm$9.90 & 9.87\,\scriptsize$\pm$0.77 & 76.95\,\scriptsize$\pm$7.06 & 15.38\,\scriptsize$\pm$0.44 & 80.64\,\scriptsize$\pm$5.80 \\
PointButton1 & 23.13\,\scriptsize$\pm$1.02 & 116.50\,\scriptsize$\pm$23.20 & 27.93\,\scriptsize$\pm$0.37 & 114.93\,\scriptsize$\pm$14.05 & 2.63\,\scriptsize$\pm$1.16 & 34.90\,\scriptsize$\pm$11.15 & 4.06\,\scriptsize$\pm$1.11 & 49.23\,\scriptsize$\pm$10.05 & 8.23\,\scriptsize$\pm$0.83 & 52.97\,\scriptsize$\pm$2.04 \\
PointButton2 & 17.84\,\scriptsize$\pm$3.79 & 92.40\,\scriptsize$\pm$26.45 & 20.86\,\scriptsize$\pm$1.11 & 129.97\,\scriptsize$\pm$20.75 & 5.37\,\scriptsize$\pm$3.69 & 45.17\,\scriptsize$\pm$15.30 & 6.46\,\scriptsize$\pm$1.77 & 64.29\,\scriptsize$\pm$8.44 & 14.05\,\scriptsize$\pm$0.54 & 72.81\,\scriptsize$\pm$5.06 \\
CarButton1 & 7.69\,\scriptsize$\pm$1.26 & 104.63\,\scriptsize$\pm$44.35 & 15.20\,\scriptsize$\pm$4.80 & 389.97\,\scriptsize$\pm$68.65 & -7.65\,\scriptsize$\pm$4.26 & 28.17\,\scriptsize$\pm$13.55 & -2.64\,\scriptsize$\pm$1.12 & 61.65\,\scriptsize$\pm$1.91 & 4.11\,\scriptsize$\pm$2.66 & 80.08\,\scriptsize$\pm$8.99 \\
CarButton2 & 5.08\,\scriptsize$\pm$3.48 & 97.70\,\scriptsize$\pm$17.35 & 9.01\,\scriptsize$\pm$2.90 & 460.30\,\scriptsize$\pm$67.20 & -6.35\,\scriptsize$\pm$3.79 & 46.30\,\scriptsize$\pm$13.90 & -3.27\,\scriptsize$\pm$1.50 & 59.60\,\scriptsize$\pm$3.86 & -0.71\,\scriptsize$\pm$2.22 & 63.62\,\scriptsize$\pm$6.41 \\
PointCircle1 & 42.7\,\scriptsize$\pm$0.1 & \textbf{6.1\,\scriptsize$\pm$2.3} & 36.9\,\scriptsize$\pm$5.4 & 30.4\,\scriptsize$\pm$45.2 & 55.2\,\scriptsize$\pm$1.4 & 192.1\,\scriptsize$\pm$9.0 & 33.0\,\scriptsize$\pm$13.0 & 156.6\,\scriptsize$\pm$115.4 & 53.7\,\scriptsize$\pm$1.7 & 179.7\,\scriptsize$\pm$17.3 \\
PointCircle2 & 40.2\,\scriptsize$\pm$0.4 & 27.9\,\scriptsize$\pm$3.7 & 30.2\,\scriptsize$\pm$6.0 & 111.3\,\scriptsize$\pm$103.4 & 48.6\,\scriptsize$\pm$1.5 & 291.3\,\scriptsize$\pm$16.5 & 32.0\,\scriptsize$\pm$4.7 & 224.8\,\scriptsize$\pm$178.2 & 42.9\,\scriptsize$\pm$1.9 & 162.8\,\scriptsize$\pm$31.6 \\
 
\bottomrule
\end{tabular}}
\end{center}
\end{table}

\begin{figure}[h]
\begin{center}
\includegraphics[width=0.62\linewidth]{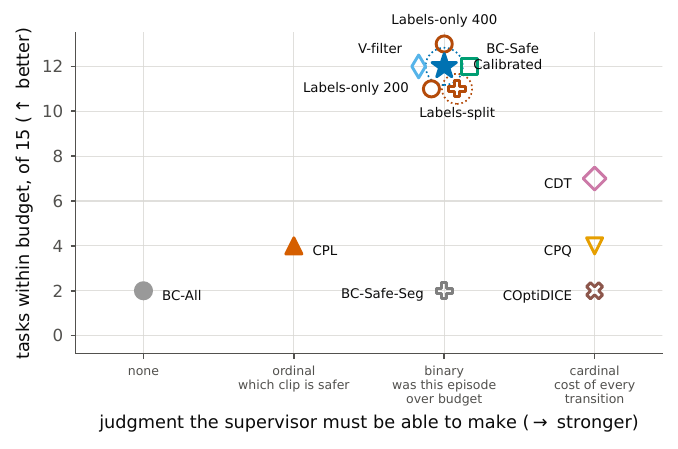}
\end{center}
\caption{Supervision against safety coverage. The horizontal axis is the strongest judgment a method's supervisor must be able to make, given by the tick labels, rather than how much of it is consumed; methods sit on their tier and spread sideways only to separate ties. The vertical axis counts the fifteen tasks kept within budget. \emph{Calibrated} is the procedure this paper proposes, V-filter its uncertified ablation, and labels-only and labels-split the controls that spend the 200 labels on training a classifier rather than on auditing one. The two ringed methods are the only ones shown that issue a certificate or an explicit refusal. CDT, CPQ and COptiDICE are external baselines at their best configuration, CDT at its safest target per task (Appendix~\ref{app:cdtsweep}); rewards appear in Tables~\ref{tab:main} and \ref{tab:baselines}.}
\label{fig:pareto}
\end{figure}

\section{Composability and contamination tolerance}
\label{app:composea40}
On the two navigation tasks CDT on the certified selection trails our own clone (7.5 against 8.5 on CarGoal1; 5.2 against 7.5 on PointGoal1), so the certified data costs it reward there, and at $\alpha = 0.40$ PointGoal1 moves from cost 5.4 to 41.5 as realized contamination rises. Table~\ref{tab:compose} reports the composability comparison at the deployed level $\alpha = 0.25$ that Section~\ref{sec:results} summarizes: full-data CDT with and without augmentation, CDT retrained on the certified selection, our calibrated clone, certified return-weighted cloning, and CPL on the same certified selection.

\begin{table}[t]
\caption{Composability at $\alpha = 0.25$: CDT and our certified operator on the deployed certified selection per task, three learner seeds each, with the two further distinct selections per task in Table~\ref{tab:a25draws}. The clone column is the ungated behavior clone of our deployed pipeline end to end over its five seeds, so it is scored on the selection each seed produces rather than on the single selection the CDT and operator columns share; it differs from Table~\ref{tab:main}'s calibrated column, which applies the certification-gated sub-selection on the certified tasks. CDT's frontier-fitting augmentation presumes a full cost spectrum and cannot run on curated data, so certified runs disable it; the no-augmentation column controls for that on full data. The operator column is certified return-weighted cloning at $\kappa = 2$ on each task's most
selective distinct certified selection, the $0.85$ quantile except on Walker2d, whose deepest
is $0.75$. CPL is given the same certified selection as CDT, with its preference budget redrawn inside it, so both external learners are scored on the same data. Bold cost is within budget.}
\label{tab:compose}
\begin{center}
\scriptsize
\setlength{\tabcolsep}{3.5pt}
\resizebox{\textwidth}{!}{%
\begin{tabular}{l rr rr rr rr rr rr}
\toprule
& \multicolumn{2}{c}{CDT (full data)} & \multicolumn{2}{c}{CDT (full, no aug.)} & \multicolumn{2}{c}{CDT (certified)} & \multicolumn{2}{c}{Ours (clone)} & \multicolumn{2}{c}{Ours (operator)} & \multicolumn{2}{c}{CPL (certified)} \\
\cmidrule(lr){2-3}\cmidrule(lr){4-5}\cmidrule(lr){6-7}\cmidrule(lr){8-9}\cmidrule(lr){10-11}\cmidrule(lr){12-13}
Task & $R$ & $C$ & $R$ & $C$ & $R$ & $C$ & $R$ & $C$ & $R$ & $C$ & $R$ & $C$ \\
\midrule
HalfCheetah & 2756\,\scriptsize$\pm$19 & \textbf{5\,\scriptsize$\pm$3} & 2794\,\scriptsize$\pm$5 & \textbf{0\,\scriptsize$\pm$0} & 2048\,\scriptsize$\pm$256 & \textbf{1\,\scriptsize$\pm$0} & 2104\,\scriptsize$\pm$69 & \textbf{3\,\scriptsize$\pm$3} & 2271\,\scriptsize$\pm$21 & \textbf{0\,\scriptsize$\pm$0} & 1537\,\scriptsize$\pm$220 & \textbf{0\,\scriptsize$\pm$0} \\
Walker2d & 2712\,\scriptsize$\pm$13 & \textbf{1\,\scriptsize$\pm$1} & 2635\,\scriptsize$\pm$118 & \textbf{1\,\scriptsize$\pm$2} & 2686\,\scriptsize$\pm$4 & \textbf{3\,\scriptsize$\pm$1} & 2652\,\scriptsize$\pm$51 & \textbf{2\,\scriptsize$\pm$2} & 2682\,\scriptsize$\pm$10 & \textbf{1\,\scriptsize$\pm$2} & 2037\,\scriptsize$\pm$286 & 51\,\scriptsize$\pm$21 \\
CarGoal1 & 26.06\,\scriptsize$\pm$1.56 & 36.87\,\scriptsize$\pm$10.00 & 28.12\,\scriptsize$\pm$0.78 & 52.10\,\scriptsize$\pm$19.55 & 7.46\,\scriptsize$\pm$0.88 & \textbf{6.77\,\scriptsize$\pm$3.30} & 8.54\,\scriptsize$\pm$0.49 & \textbf{8.41\,\scriptsize$\pm$0.73} & 8.73\,\scriptsize$\pm$0.32 & \textbf{7.40\,\scriptsize$\pm$1.03} & 9.51\,\scriptsize$\pm$0.52 & \textbf{6.84\,\scriptsize$\pm$1.76} \\
PointGoal1 & 20.17\,\scriptsize$\pm$1.56 & 36.77\,\scriptsize$\pm$3.75 & 22.68\,\scriptsize$\pm$0.39 & 47.03\,\scriptsize$\pm$7.65 & 5.22\,\scriptsize$\pm$2.04 & \textbf{5.40\,\scriptsize$\pm$3.90} & 7.47\,\scriptsize$\pm$0.68 & \textbf{8.80\,\scriptsize$\pm$1.74} & 5.31\,\scriptsize$\pm$0.66 & \textbf{4.82\,\scriptsize$\pm$1.98} & 7.35\,\scriptsize$\pm$0.17 & \textbf{9.78\,\scriptsize$\pm$2.31} \\
 
\bottomrule
\end{tabular}}
\end{center}
\end{table}

Table~\ref{tab:composea40} reports CDT, our calibrated clone and CPL at $\alpha = 0.40$ for the eight tasks that certify at that level (three seeds each; the pipeline's deployed certified selection per task, with the full three-selection grid in Table~\ref{tab:a40draws}; Hopper never certifies), alongside full-data CDT with and without augmentation for reference. Our clone is safe on seven of the eight, failing on Swimmer at cost $24$;
CDT is safe on six and violates on PointGoal1 and PointGoal2, the tasks whose selections carry the largest usable unsafe tails. The no-augmentation control seals the attribution across the table: full-data CDT without augmentation is safe on the velocity tasks but severely unsafe on all four navigation tasks (costs 47 to 109), so wherever certified-selection CDT is safe on navigation, the safety is due to the certified data and not to the disabled augmentation. Against full-data CDT, the $\alpha = 0.40$ selection lowers cost on six of the eight tasks and leaves Walker2d unchanged near zero; the exception is PointGoal1, where the loosely curated pool is worse than no curation at all (41.5 against 36.8), consistent with the selection concentrating high-return marginal violators that return conditioning then exploits. Read against full-data CDT, the two levels differ in what they achieve on the four shared tasks. Where the full pool already violates, the tighter level rescues both navigation tasks ($36.9$ to $6.8$ on CarGoal1, $36.8$ to $5.4$ on PointGoal1) while the looser one rescues only CarGoal1 and leaves PointGoal1 violating at $41.5$. Where the full pool is already safe, the tighter level keeps it so ($5.3$ to $0.9$ on HalfCheetah) while the looser one doubles its cost to $11.3$; on Walker2d both stay far within budget ($0.9$ on the pool, $2.7$ and $2.1$). The tighter certificate therefore rescues where the looser one cannot and takes no task outside its budget, which is the dose-response the main text summarizes.

\paragraph{Selection robustness across distinct certified selections.} Independent calibration draws often select the identical threshold. Across the twelve task-level settings the 500-draw resample covers, the single most probable distinct selection accounts for $21$ to $83$ percent of certified draws, and the three reported selections together for at least $53$ percent and on eight of the twelve for more than nine tenths. The concentration is itself a deployment-stability property, so the grids are built on the three most probable \emph{distinct} certified selections per task, ordered by realized contamination $\hat\alpha$ (Table~\ref{tab:a40draws}); the distribution's rare tail contains the certificate's worst outputs, and the grids deliberately include them. At $\alpha = 0.25$, all twelve distinct selections across the four tasks are CDT-safe, 36 of 36 runs (Table~\ref{tab:a25draws}), including three out-of-specification selections (HalfCheetah at realized $0.263$ and $0.337$, costs $2$ and $0$; CarGoal1 at $0.251$, cost $12.0$), with per-task mean costs spanning at most $6.8$. CPL is within budget on 28 of the same 36 runs, and every one of its eight failures is on Walker2d, the task whose full-data CPL already sits at the budget. At $\alpha = 0.40$ every learner degrades with realized contamination, at very different thresholds. CDT is safe on 14 of the 24 selections and violates on every navigation task on at least one, CarGoal1 moving from safe to unsafe as contamination rises ($0.29$, $0.32$, $0.36$ giving $17.5$, $18.9$, $36.2$). CPL, safe on 28 of 36 runs at $\alpha = 0.25$, is safe on 17 of 24 selections here, failing on all three Walker2d selections, Swimmer's most contaminated one and three navigation selections, so the dose-response that justifies the tighter level holds for the second external learner as well as the first. Our clone, an averaging learner, absorbs far more: it is safe on 22 of 24 task-selections (Table~\ref{tab:a40drawsbc}); its two failures are Swimmer's tail selections at realized $0.33$ and $0.45$, the latter above even the certified level, so the averaging learner's tolerance has a measured boundary rather than being unbounded, and on the $0.33$ selection the roles reverse, CDT safe at $18$ while the clone violates at $96$. One exception is honest to note. On PointGoal2 the ordering is not monotone in contamination ($27.0$, $14.2$, $33.3$ at contamination $0.26$, $0.32$, $0.37$), so selection size interacts with composition and the contamination reading is a strong tendency, not a law. The composite picture stands. At the tight level realized composition is pinned and every distinct selection composes for the return-conditioned learner, while at the loose level the certified band is wide enough that consumer safety depends on which selection the certificate outputs.

\begin{table}[h]
\caption{Distinct-selection grid at $\alpha = 0.25$: the three distinct certified selections per task for the four tasks that certify at this level, with the reward and mean cost of CDT and of CPL on each identical selection (three seeds each; bold marks mean cost within budget). All 36 CDT runs are within budget, including the three selections whose realized
contamination $\hat\alpha$ exceeds the certified level; CPL is within budget on 28 of its 36,
its eight failures falling on Walker2d's three selections.}
\label{tab:a25draws}
\begin{center}
\scriptsize
\setlength{\tabcolsep}{3pt}
\resizebox{\textwidth}{!}{%
\begin{tabular}{l rrrrr rrrrr rrrrr}
\toprule
& \multicolumn{5}{c}{selection 1} & \multicolumn{5}{c}{selection 2} & \multicolumn{5}{c}{selection 3} \\
\cmidrule(lr){2-6}\cmidrule(lr){7-11}\cmidrule(lr){12-16}
& & \multicolumn{2}{c}{CDT} & \multicolumn{2}{c}{CPL} & & \multicolumn{2}{c}{CDT} & \multicolumn{2}{c}{CPL} & & \multicolumn{2}{c}{CDT} & \multicolumn{2}{c}{CPL} \\
\cmidrule(lr){3-4}\cmidrule(lr){5-6}\cmidrule(lr){8-9}\cmidrule(lr){10-11}\cmidrule(lr){13-14}\cmidrule(lr){15-16}
Task & $\hat\alpha$ & $R$ & $C$ & $R$ & $C$ & $\hat\alpha$ & $R$ & $C$ & $R$ & $C$ & $\hat\alpha$ & $R$ & $C$ & $R$ & $C$ \\
\midrule
HalfCheetah & 0.192 & 2048\,\scriptsize$\pm$256 & \textbf{1\,\scriptsize$\pm$0} & 1537\,\scriptsize$\pm$220 & \textbf{0\,\scriptsize$\pm$0} & 0.263 & 2176\,\scriptsize$\pm$172 & \textbf{2\,\scriptsize$\pm$2} & 1764\,\scriptsize$\pm$159 & \textbf{0\,\scriptsize$\pm$0} & 0.337 & 2774\,\scriptsize$\pm$18 & \textbf{0\,\scriptsize$\pm$0} & 1551\,\scriptsize$\pm$192 & \textbf{0\,\scriptsize$\pm$0} \\
Walker2d & 0.139 & 2686\,\scriptsize$\pm$4 & \textbf{3\,\scriptsize$\pm$1} & 2037\,\scriptsize$\pm$286 & 51\,\scriptsize$\pm$21 & 0.177 & 2691\,\scriptsize$\pm$25 & \textbf{5\,\scriptsize$\pm$2} & 1827\,\scriptsize$\pm$151 & 63\,\scriptsize$\pm$5 & 0.248 & 2604\,\scriptsize$\pm$172 & \textbf{6\,\scriptsize$\pm$8} & 2212\,\scriptsize$\pm$146 & 35\,\scriptsize$\pm$24 \\
CarGoal1 & 0.175 & 7.5\,\scriptsize$\pm$0.9 & \textbf{6.77\,\scriptsize$\pm$3.30} & 9.5\,\scriptsize$\pm$0.5 & \textbf{6.84\,\scriptsize$\pm$1.76} & 0.206 & 10.5\,\scriptsize$\pm$1.6 & \textbf{5.23\,\scriptsize$\pm$1.80} & 9.5\,\scriptsize$\pm$0.6 & \textbf{8.44\,\scriptsize$\pm$1.19} & 0.251 & 9.3\,\scriptsize$\pm$2.3 & \textbf{12.03\,\scriptsize$\pm$3.30} & 11.1\,\scriptsize$\pm$0.5 & \textbf{10.36\,\scriptsize$\pm$2.54} \\
PointGoal1 & 0.122 & 5.2\,\scriptsize$\pm$2.0 & \textbf{5.40\,\scriptsize$\pm$3.90} & 7.3\,\scriptsize$\pm$0.2 & \textbf{9.78\,\scriptsize$\pm$2.31} & 0.141 & 4.1\,\scriptsize$\pm$0.9 & \textbf{7.33\,\scriptsize$\pm$4.15} & 6.7\,\scriptsize$\pm$0.9 & \textbf{8.73\,\scriptsize$\pm$1.05} & 0.238 & 2.6\,\scriptsize$\pm$1.8 & \textbf{6.80\,\scriptsize$\pm$2.40} & 9.0\,\scriptsize$\pm$2.1 & \textbf{12.30\,\scriptsize$\pm$1.71} \\
 
\bottomrule
\end{tabular}}
\end{center}
\end{table}

\begin{table}[h]
\caption{Composability at $\alpha = 0.40$ on the deployed certified selection per task (three seeds, mean $\pm$ 95\% CI; bold marks mean cost within budget). Column naming follows Table~\ref{tab:compose}: the clone column is our deployed pipeline end to end, so it is scored on the selection each seed produces rather than on the single selection the CDT and CPL columns share. The operator column of Table~\ref{tab:compose} has no counterpart here because certified return-weighted cloning is characterized at $\alpha = 0.25$ (Table~\ref{tab:operator}). Both external learners, CDT and CPL, are scored on the same selection as our own clone, with CPL's preference budget redrawn inside it.}
\label{tab:composea40}
\begin{center}
\scriptsize
\setlength{\tabcolsep}{3.5pt}
\resizebox{\textwidth}{!}{%
\begin{tabular}{l rr rr rr rr rr}
\toprule
& \multicolumn{2}{c}{CDT (full data)} & \multicolumn{2}{c}{CDT (full, no aug.)} & \multicolumn{2}{c}{CDT (certified)} & \multicolumn{2}{c}{Ours (calibrated)} & \multicolumn{2}{c}{CPL (certified)} \\
\cmidrule(lr){2-3}\cmidrule(lr){4-5}\cmidrule(lr){6-7}\cmidrule(lr){8-9}\cmidrule(lr){10-11}
Task & $R$ & $C$ & $R$ & $C$ & $R$ & $C$ & $R$ & $C$ & $R$ & $C$ \\
\midrule
HalfCheetah & 2756\,\scriptsize$\pm$19 & \textbf{5\,\scriptsize$\pm$3} & 2794\,\scriptsize$\pm$5 & \textbf{0\,\scriptsize$\pm$0} & 2069\,\scriptsize$\pm$75 & \textbf{11\,\scriptsize$\pm$3} & 2176\,\scriptsize$\pm$273 & \textbf{6\,\scriptsize$\pm$8} & 1618\,\scriptsize$\pm$174 & \textbf{0\,\scriptsize$\pm$0} \\
Walker2d & 2712\,\scriptsize$\pm$13 & \textbf{1\,\scriptsize$\pm$1} & 2635\,\scriptsize$\pm$118 & \textbf{1\,\scriptsize$\pm$2} & 2706\,\scriptsize$\pm$15 & \textbf{2\,\scriptsize$\pm$1} & 2674\,\scriptsize$\pm$54 & \textbf{1\,\scriptsize$\pm$1} & 1680\,\scriptsize$\pm$301 & 50\,\scriptsize$\pm$15 \\
Ant & 2946\,\scriptsize$\pm$10 & \textbf{12\,\scriptsize$\pm$3} & 2930\,\scriptsize$\pm$23 & \textbf{9\,\scriptsize$\pm$5} & 2923\,\scriptsize$\pm$4 & \textbf{7\,\scriptsize$\pm$2} & 2705\,\scriptsize$\pm$108 & \textbf{2\,\scriptsize$\pm$1} & 1226\,\scriptsize$\pm$217 & \textbf{0\,\scriptsize$\pm$0} \\
Swimmer & 162\,\scriptsize$\pm$1 & 23\,\scriptsize$\pm$6 & 161\,\scriptsize$\pm$1 & \textbf{20\,\scriptsize$\pm$2} & 156\,\scriptsize$\pm$1 & \textbf{9\,\scriptsize$\pm$5} & 121\,\scriptsize$\pm$13 & 24\,\scriptsize$\pm$14 & 32\,\scriptsize$\pm$24 & \textbf{0\,\scriptsize$\pm$0} \\
CarGoal1 & 26.06\,\scriptsize$\pm$1.56 & 36.87\,\scriptsize$\pm$10.00 & 28.12\,\scriptsize$\pm$0.78 & 52.10\,\scriptsize$\pm$19.55 & 8.80\,\scriptsize$\pm$4.18 & \textbf{3.60\,\scriptsize$\pm$3.10} & 8.76\,\scriptsize$\pm$1.62 & \textbf{7.47\,\scriptsize$\pm$1.83} & 9.32\,\scriptsize$\pm$0.76 & \textbf{7.81\,\scriptsize$\pm$1.68} \\
CarGoal2 & 14.92\,\scriptsize$\pm$3.14 & 65.37\,\scriptsize$\pm$13.80 & 16.32\,\scriptsize$\pm$1.18 & 108.93\,\scriptsize$\pm$22.65 & 4.94\,\scriptsize$\pm$0.56 & \textbf{20.63\,\scriptsize$\pm$4.25} & 3.15\,\scriptsize$\pm$0.24 & \textbf{17.83\,\scriptsize$\pm$1.52} & 4.95\,\scriptsize$\pm$0.31 & \textbf{23.93\,\scriptsize$\pm$2.07} \\
PointGoal1 & 20.17\,\scriptsize$\pm$1.56 & 36.77\,\scriptsize$\pm$3.75 & 22.68\,\scriptsize$\pm$0.39 & 47.03\,\scriptsize$\pm$7.65 & 18.60\,\scriptsize$\pm$1.62 & 41.53\,\scriptsize$\pm$19.70 & 9.85\,\scriptsize$\pm$2.39 & \textbf{12.28\,\scriptsize$\pm$1.75} & 12.61\,\scriptsize$\pm$1.62 & \textbf{19.99\,\scriptsize$\pm$3.44} \\
PointGoal2 & 13.78\,\scriptsize$\pm$2.82 & 59.97\,\scriptsize$\pm$16.05 & 16.88\,\scriptsize$\pm$2.15 & 88.67\,\scriptsize$\pm$14.20 & 7.47\,\scriptsize$\pm$1.14 & 26.53\,\scriptsize$\pm$12.30 & 6.42\,\scriptsize$\pm$0.80 & \textbf{20.98\,\scriptsize$\pm$1.28} & 7.32\,\scriptsize$\pm$2.50 & 27.34\,\scriptsize$\pm$13.58 \\
 
\bottomrule
\end{tabular}}
\end{center}
\end{table}

\begin{table}[h]
\caption{Distinct-selection grid at $\alpha = 0.40$: the three most probable distinct certified selections per task, ordered by increasing realized contamination $\hat\alpha$, with the reward and mean cost of CDT and of CPL on the identical selection, matching Table~\ref{tab:a25draws} (three seeds each; bold marks mean cost within budget). The pipeline's deployed selection per task, whose position in the ordering varies, is the one reported in Table~\ref{tab:composea40}. PointGoal2's selections are drawn at pipeline seed 1 because seed 0 does not certify at this level, and a certified selection cannot be reported from a draw that refused; the refused draw's clone would itself have been safe (cost $21.0$ against budget 25), so the substitution is conservative rather than favorable.}
\label{tab:a40draws}
\begin{center}
\scriptsize
\setlength{\tabcolsep}{3pt}
\resizebox{\textwidth}{!}{%
\begin{tabular}{l rrrrr rrrrr rrrrr}
\toprule
& \multicolumn{5}{c}{selection 1} & \multicolumn{5}{c}{selection 2} & \multicolumn{5}{c}{selection 3} \\
\cmidrule(lr){2-6}\cmidrule(lr){7-11}\cmidrule(lr){12-16}
& & \multicolumn{2}{c}{CDT} & \multicolumn{2}{c}{CPL} & & \multicolumn{2}{c}{CDT} & \multicolumn{2}{c}{CPL} & & \multicolumn{2}{c}{CDT} & \multicolumn{2}{c}{CPL} \\
\cmidrule(lr){3-4}\cmidrule(lr){5-6}\cmidrule(lr){8-9}\cmidrule(lr){10-11}\cmidrule(lr){13-14}\cmidrule(lr){15-16}
Task & $\hat\alpha$ & $R$ & $C$ & $R$ & $C$ & $\hat\alpha$ & $R$ & $C$ & $R$ & $C$ & $\hat\alpha$ & $R$ & $C$ & $R$ & $C$ \\
\midrule
HalfCheetah & 0.19 & 2069\,\scriptsize$\pm$75 & \textbf{11\,\scriptsize$\pm$3} & 1618\,\scriptsize$\pm$174 & \textbf{0\,\scriptsize$\pm$0} & 0.26 & 1831\,\scriptsize$\pm$196 & \textbf{0\,\scriptsize$\pm$0} & 1400\,\scriptsize$\pm$81 & \textbf{1\,\scriptsize$\pm$0} & 0.34 & 2766\,\scriptsize$\pm$20 & \textbf{0\,\scriptsize$\pm$1} & 1736\,\scriptsize$\pm$37 & \textbf{0\,\scriptsize$\pm$0} \\
Walker2d & 0.25 & 2706\,\scriptsize$\pm$15 & \textbf{2\,\scriptsize$\pm$1} & 1680\,\scriptsize$\pm$301 & 50\,\scriptsize$\pm$15 & 0.33 & 2647\,\scriptsize$\pm$99 & \textbf{2\,\scriptsize$\pm$3} & 1531\,\scriptsize$\pm$151 & 38\,\scriptsize$\pm$8 & 0.39 & 2730\,\scriptsize$\pm$30 & \textbf{1\,\scriptsize$\pm$1} & 1574\,\scriptsize$\pm$577 & 34\,\scriptsize$\pm$8 \\
Ant & 0.35 & 2564\,\scriptsize$\pm$19 & \textbf{1\,\scriptsize$\pm$0} & 1316\,\scriptsize$\pm$109 & \textbf{0\,\scriptsize$\pm$0} & 0.36 & 2617\,\scriptsize$\pm$101 & \textbf{5\,\scriptsize$\pm$4} & 1364\,\scriptsize$\pm$163 & \textbf{0\,\scriptsize$\pm$0} & 0.40 & 2923\,\scriptsize$\pm$4 & \textbf{7\,\scriptsize$\pm$2} & 1226\,\scriptsize$\pm$217 & \textbf{0\,\scriptsize$\pm$0} \\
Swimmer & 0.23 & 156\,\scriptsize$\pm$1 & \textbf{9\,\scriptsize$\pm$5} & 32\,\scriptsize$\pm$24 & \textbf{0\,\scriptsize$\pm$0} & 0.33 & 162\,\scriptsize$\pm$1 & \textbf{18\,\scriptsize$\pm$11} & 38\,\scriptsize$\pm$12 & \textbf{2\,\scriptsize$\pm$3} & 0.45 & 162\,\scriptsize$\pm$1 & 20\,\scriptsize$\pm$1 & 48\,\scriptsize$\pm$8 & 56\,\scriptsize$\pm$52 \\
CarGoal1 & 0.29 & 14.6\,\scriptsize$\pm$3.6 & \textbf{17.47\,\scriptsize$\pm$3.80} & 11.8\,\scriptsize$\pm$0.9 & \textbf{13.77\,\scriptsize$\pm$1.26} & 0.32 & 15.5\,\scriptsize$\pm$1.9 & \textbf{18.87\,\scriptsize$\pm$6.90} & 11.7\,\scriptsize$\pm$0.3 & \textbf{10.80\,\scriptsize$\pm$1.34} & 0.36 & 23.2\,\scriptsize$\pm$2.2 & 36.23\,\scriptsize$\pm$5.05 & 13.4\,\scriptsize$\pm$0.5 & \textbf{14.11\,\scriptsize$\pm$1.92} \\
CarGoal2 & 0.28 & 7.5\,\scriptsize$\pm$3.7 & 35.50\,\scriptsize$\pm$14.65 & 5.0\,\scriptsize$\pm$0.3 & \textbf{23.93\,\scriptsize$\pm$2.07} & 0.31 & 10.5\,\scriptsize$\pm$1.1 & 47.87\,\scriptsize$\pm$10.70 & 5.5\,\scriptsize$\pm$0.7 & \textbf{23.79\,\scriptsize$\pm$4.05} & 0.35 & 10.5\,\scriptsize$\pm$1.8 & 49.57\,\scriptsize$\pm$16.10 & 6.0\,\scriptsize$\pm$0.3 & \textbf{23.21\,\scriptsize$\pm$3.04} \\
PointGoal1 & 0.28 & 19.3\,\scriptsize$\pm$1.4 & 28.73\,\scriptsize$\pm$9.55 & 10.6\,\scriptsize$\pm$1.4 & \textbf{17.64\,\scriptsize$\pm$3.40} & 0.30 & 20.2\,\scriptsize$\pm$1.3 & 53.50\,\scriptsize$\pm$3.15 & 13.9\,\scriptsize$\pm$3.0 & 27.70\,\scriptsize$\pm$9.49 & 0.32 & 21.0\,\scriptsize$\pm$0.6 & 45.73\,\scriptsize$\pm$14.00 & 12.6\,\scriptsize$\pm$1.6 & \textbf{19.99\,\scriptsize$\pm$3.44} \\
PointGoal2 & 0.26 & 9.0\,\scriptsize$\pm$1.8 & 27.00\,\scriptsize$\pm$14.65 & 7.6\,\scriptsize$\pm$1.6 & \textbf{24.53\,\scriptsize$\pm$3.63} & 0.32 & 7.5\,\scriptsize$\pm$1.4 & \textbf{14.20\,\scriptsize$\pm$6.75} & 7.3\,\scriptsize$\pm$2.5 & 27.34\,\scriptsize$\pm$13.58 & 0.37 & 9.2\,\scriptsize$\pm$0.7 & 33.30\,\scriptsize$\pm$11.25 & 7.2\,\scriptsize$\pm$3.7 & 28.60\,\scriptsize$\pm$11.47 \\
 
\bottomrule
\end{tabular}}
\end{center}
\end{table}

\paragraph{Contamination tolerance of the clone.} The composability tables (Tables~\ref{tab:compose}, \ref{tab:a25draws}, \ref{tab:composea40} and~\ref{tab:a40draws}) report external learners on certified data. The clone is our own consumer, not an external one, so its behavior across the same selections is a learner-sensitivity control rather than a composability result, reported separately in Table~\ref{tab:a40drawsbc}: it measures how much contamination an averaging learner absorbs, which is what sets the level the external learners need. The comparison is against the tight level, where the clone is within budget on every certifying task (Table~\ref{tab:compose}); what Table~\ref{tab:a40drawsbc} measures is how far that tolerance extends once the certified band is widened. That table repeats neither CDT nor CPL, which Table~\ref{tab:a40draws} already gives on the identical selections, and its clone differs from the calibrated column of Table~\ref{tab:composea40}, which is the pipeline end to end on the selection each seed produces rather than cloning on a fixed given selection.

\begin{table}[h]
\caption{Contamination tolerance of the clone: behavior cloning on the identical $\alpha = 0.40$ selections of Table~\ref{tab:a40draws}, ordered by increasing realized contamination $\hat\alpha$ (three seeds each; bold marks mean cost within budget).}
\label{tab:a40drawsbc}
\begin{center}
\scriptsize
\setlength{\tabcolsep}{4pt}
\begin{tabular}{l rrr rrr rrr}
\toprule
& \multicolumn{3}{c}{selection 1} & \multicolumn{3}{c}{selection 2} & \multicolumn{3}{c}{selection 3} \\
\cmidrule(lr){2-4}\cmidrule(lr){5-7}\cmidrule(lr){8-10}
Task & $\hat\alpha$ & $R$ & $C$ & $\hat\alpha$ & $R$ & $C$ & $\hat\alpha$ & $R$ & $C$ \\
\midrule
HalfCheetah & 0.19 & 2118\,\scriptsize$\pm$23 & \textbf{0\,\scriptsize$\pm$0} & 0.26 & 2153\,\scriptsize$\pm$215 & \textbf{15\,\scriptsize$\pm$21} & 0.34 & 2291\,\scriptsize$\pm$55 & \textbf{5\,\scriptsize$\pm$7} \\
Walker2d & 0.25 & 2696\,\scriptsize$\pm$10 & \textbf{1\,\scriptsize$\pm$1} & 0.33 & 2717\,\scriptsize$\pm$7 & \textbf{1\,\scriptsize$\pm$0} & 0.39 & 2656\,\scriptsize$\pm$46 & \textbf{9\,\scriptsize$\pm$12} \\
Ant & 0.35 & 2423\,\scriptsize$\pm$9 & \textbf{1\,\scriptsize$\pm$0} & 0.36 & 2568\,\scriptsize$\pm$5 & \textbf{2\,\scriptsize$\pm$1} & 0.40 & 2811\,\scriptsize$\pm$21 & \textbf{5\,\scriptsize$\pm$1} \\
Swimmer & 0.23 & 108\,\scriptsize$\pm$8 & \textbf{12\,\scriptsize$\pm$6} & 0.33 & 128\,\scriptsize$\pm$21 & 96\,\scriptsize$\pm$111 & 0.45 & 125\,\scriptsize$\pm$10 & 66\,\scriptsize$\pm$61 \\
CarGoal1 & 0.29 & 9.4\,\scriptsize$\pm$0.7 & \textbf{9.54\,\scriptsize$\pm$0.57} & 0.32 & 9.8\,\scriptsize$\pm$0.7 & \textbf{8.73\,\scriptsize$\pm$1.47} & 0.36 & 10.5\,\scriptsize$\pm$0.5 & \textbf{8.71\,\scriptsize$\pm$1.99} \\
CarGoal2 & 0.28 & 3.6\,\scriptsize$\pm$0.7 & \textbf{19.69\,\scriptsize$\pm$3.77} & 0.31 & 3.8\,\scriptsize$\pm$0.1 & \textbf{16.00\,\scriptsize$\pm$2.30} & 0.35 & 3.9\,\scriptsize$\pm$0.3 & \textbf{18.75\,\scriptsize$\pm$2.47} \\
PointGoal1 & 0.28 & 9.4\,\scriptsize$\pm$0.7 & \textbf{11.10\,\scriptsize$\pm$2.98} & 0.30 & 10.5\,\scriptsize$\pm$0.4 & \textbf{13.83\,\scriptsize$\pm$0.43} & 0.32 & 11.3\,\scriptsize$\pm$0.8 & \textbf{16.10\,\scriptsize$\pm$0.50} \\
PointGoal2 & 0.26 & 5.7\,\scriptsize$\pm$0.4 & \textbf{15.63\,\scriptsize$\pm$1.73} & 0.32 & 6.0\,\scriptsize$\pm$0.2 & \textbf{16.37\,\scriptsize$\pm$2.98} & 0.37 & 7.1\,\scriptsize$\pm$1.0 & \textbf{21.96\,\scriptsize$\pm$1.08} \\
 
\bottomrule
\end{tabular}
\end{center}
\end{table}

\paragraph{The composition-to-cost mapping, measured.} Figure~\ref{fig:contamcost} plots clone cost over budget against realized contamination for 235 archived selection-clone pairs spanning certified selections, ground-truth and control selections, and both V-filter fractions across the fifteen main tasks, with a single archived CarRun pair included. Task-demeaned, contamination alone explains $R^2 = 0.41$ of clone cost, so the mapping the paper calls empirically favorable is a measured, monotone tendency rather than an assertion; the remaining variance is carried partly by the selection's score margin (adding margin statistics raises $R^2$ to $0.49$) and partly by learner effects the composability grids characterize.

\begin{figure}[h]
\begin{center}
\includegraphics[width=0.6\linewidth]{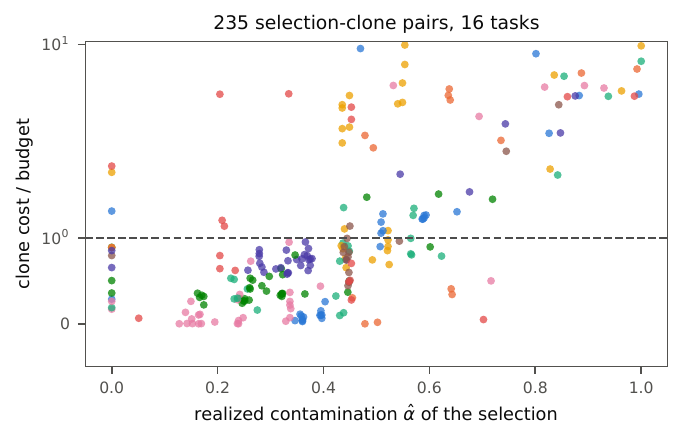}
\end{center}
\caption{Clone cost against realized selection contamination $\hat\alpha$, 235 selection-clone pairs, all but one from the fifteen DSRL tasks (colors: tasks; dashed line: budget; symlog cost axis).}
\label{fig:contamcost}
\end{figure}

\paragraph{The return-weighted operator across intensities.} Table~\ref{tab:operator} reports certified return-weighted cloning (Section~\ref{sec:gating}) on the distinct $\alpha = 0.25$ selections of Table~\ref{tab:a25draws}, at clips $\kappa \in \{1, 2, 3\}$ and the hard top-half rule, plus the CarRun certified selection from the transfer suite. The grid substantiates the two claims made in Section~\ref{sec:results}. Of the 40 within-specification cells, 39 are safe at every intensity, with reward on HalfCheetah rising from $\kappa = 1$ to $\kappa = 2$ on all three of its selections; the one exception is CarRun's hard top-half rule, which Appendix~\ref{app:bullet} reads as locating the boundary at truncation rather than at return-seeking. Of the twelve out-of-specification cells, 11 are safe: the two milder such selections ($0.251$ and $0.263$) remain safe in every arm, and the most contaminated one (HalfCheetah at $0.337$) is safe at $\kappa = 1$, at $\kappa = 2$ and under the hard rule, violating only at $\kappa = 3$ ($25.1$ against budget 20), with the strongest setting mining the contaminated return tail exactly as the gating rationale predicts. The failure therefore requires both a contaminated selection and the most aggressive intensity; the certificate removes the first factor and thereby makes the second far safer to tune.

\begin{table}[h]
\caption{Certified return-weighted cloning across intensities on the distinct $\alpha = 0.25$ certified selections and the CarRun echo selection (three seeds, mean $\pm$ 95\% CI; bold marks mean cost within budget; $\hat\alpha$ is the selection's realized contamination, computed from the selection's own membership as in Table~\ref{tab:a25draws}). The selections are those of Table~\ref{tab:a25draws} on HalfCheetah, Walker2d and CarGoal1. On PointGoal1, which lies outside the regenerated cohort, the operator sweep predates that grid and was run on the $q85$, $q70$ and $q65$ thresholds, of which the first is shared; no operator run exists at the grid's third threshold.}
\label{tab:operator}
\begin{center}
\scriptsize
\setlength{\tabcolsep}{3pt}
\resizebox{\textwidth}{!}{%
\begin{tabular}{l c rr rr rr rr}
\toprule
& & \multicolumn{2}{c}{$\kappa = 1$} & \multicolumn{2}{c}{$\kappa = 2$} & \multicolumn{2}{c}{$\kappa = 3$} & \multicolumn{2}{c}{Top half} \\
\cmidrule(lr){3-4}\cmidrule(lr){5-6}\cmidrule(lr){7-8}\cmidrule(lr){9-10}
Task & $\hat\alpha$ & $R$ & $C$ & $R$ & $C$ & $R$ & $C$ & $R$ & $C$ \\
\midrule
HalfCheetah & 0.192 & 2178\,\scriptsize$\pm$122 & \textbf{0.1\,\scriptsize$\pm$0.2} & 2271\,\scriptsize$\pm$21 & \textbf{0.0\,\scriptsize$\pm$0.0} & 2266\,\scriptsize$\pm$38 & \textbf{0.0\,\scriptsize$\pm$0.0} & 2230\,\scriptsize$\pm$31 & \textbf{0.2\,\scriptsize$\pm$0.2} \\
HalfCheetah & 0.263 & 2370\,\scriptsize$\pm$81 & \textbf{0.8\,\scriptsize$\pm$0.8} & 2396\,\scriptsize$\pm$10 & \textbf{0.6\,\scriptsize$\pm$0.8} & 2454\,\scriptsize$\pm$51 & \textbf{2.0\,\scriptsize$\pm$2.5} & 2399\,\scriptsize$\pm$17 & \textbf{2.5\,\scriptsize$\pm$1.9} \\
HalfCheetah & 0.337 & 2497\,\scriptsize$\pm$30 & \textbf{7.2\,\scriptsize$\pm$5.1} & 2545\,\scriptsize$\pm$20 & \textbf{16.1\,\scriptsize$\pm$11.8} & 2509\,\scriptsize$\pm$77 & 25.1\,\scriptsize$\pm$15.3 & 2566\,\scriptsize$\pm$37 & \textbf{7.2\,\scriptsize$\pm$1.3} \\
Walker2d & 0.139 & 2698\,\scriptsize$\pm$6 & \textbf{0.0\,\scriptsize$\pm$0.1} & 2682\,\scriptsize$\pm$10 & \textbf{1.3\,\scriptsize$\pm$1.9} & 2683\,\scriptsize$\pm$32 & \textbf{0.3\,\scriptsize$\pm$0.2} & 2558\,\scriptsize$\pm$203 & \textbf{18.7\,\scriptsize$\pm$27.5} \\
Walker2d & 0.177 & 2695\,\scriptsize$\pm$11 & \textbf{1.1\,\scriptsize$\pm$1.0} & 2702\,\scriptsize$\pm$24 & \textbf{0.8\,\scriptsize$\pm$0.8} & 2719\,\scriptsize$\pm$10 & \textbf{0.2\,\scriptsize$\pm$0.2} & 2721\,\scriptsize$\pm$21 & \textbf{1.3\,\scriptsize$\pm$1.3} \\
Walker2d & 0.248 & 2720\,\scriptsize$\pm$20 & \textbf{0.9\,\scriptsize$\pm$1.0} & 2741\,\scriptsize$\pm$20 & \textbf{2.6\,\scriptsize$\pm$2.9} & 2711\,\scriptsize$\pm$18 & \textbf{0.2\,\scriptsize$\pm$0.1} & 2712\,\scriptsize$\pm$34 & \textbf{1.4\,\scriptsize$\pm$1.5} \\
CarGoal1 & 0.175 & 7.4\,\scriptsize$\pm$0.6 & \textbf{7.8\,\scriptsize$\pm$0.8} & 8.7\,\scriptsize$\pm$0.3 & \textbf{7.4\,\scriptsize$\pm$1.0} & 7.7\,\scriptsize$\pm$0.3 & \textbf{6.9\,\scriptsize$\pm$1.9} & 8.1\,\scriptsize$\pm$0.3 & \textbf{6.6\,\scriptsize$\pm$1.3} \\
CarGoal1 & 0.206 & 8.8\,\scriptsize$\pm$0.6 & \textbf{6.1\,\scriptsize$\pm$1.1} & 8.5\,\scriptsize$\pm$1.0 & \textbf{7.2\,\scriptsize$\pm$0.9} & 7.9\,\scriptsize$\pm$0.4 & \textbf{6.1\,\scriptsize$\pm$0.5} & 9.0\,\scriptsize$\pm$1.0 & \textbf{8.1\,\scriptsize$\pm$1.0} \\
CarGoal1 & 0.251 & 9.1\,\scriptsize$\pm$0.7 & \textbf{8.4\,\scriptsize$\pm$0.6} & 8.8\,\scriptsize$\pm$0.6 & \textbf{8.1\,\scriptsize$\pm$2.3} & 8.7\,\scriptsize$\pm$0.6 & \textbf{8.2\,\scriptsize$\pm$1.6} & 9.3\,\scriptsize$\pm$0.3 & \textbf{9.4\,\scriptsize$\pm$0.4} \\
PointGoal1 & 0.122 & 4.8\,\scriptsize$\pm$0.3 & \textbf{6.0\,\scriptsize$\pm$1.7} & 5.3\,\scriptsize$\pm$0.7 & \textbf{4.8\,\scriptsize$\pm$2.0} & 4.8\,\scriptsize$\pm$0.2 & \textbf{4.0\,\scriptsize$\pm$0.8} & 5.8\,\scriptsize$\pm$0.5 & \textbf{5.1\,\scriptsize$\pm$1.3} \\
PointGoal1 & 0.183 & 8.5\,\scriptsize$\pm$0.2 & \textbf{10.4\,\scriptsize$\pm$0.5} & 7.8\,\scriptsize$\pm$0.4 & \textbf{9.0\,\scriptsize$\pm$1.6} & 8.1\,\scriptsize$\pm$0.6 & \textbf{9.5\,\scriptsize$\pm$1.1} & 8.1\,\scriptsize$\pm$0.5 & \textbf{9.3\,\scriptsize$\pm$1.3} \\
PointGoal1 & 0.208 & 7.3\,\scriptsize$\pm$0.1 & \textbf{10.9\,\scriptsize$\pm$1.6} & 7.6\,\scriptsize$\pm$0.2 & \textbf{8.8\,\scriptsize$\pm$0.3} & 8.1\,\scriptsize$\pm$0.5 & \textbf{10.2\,\scriptsize$\pm$1.4} & 8.9\,\scriptsize$\pm$0.4 & \textbf{10.7\,\scriptsize$\pm$1.6} \\
CarRun (echo) & 0.051 & 480\,\scriptsize$\pm$58 & \textbf{3.7\,\scriptsize$\pm$4.4} & 513\,\scriptsize$\pm$36 & \textbf{0.7\,\scriptsize$\pm$0.5} & 510\,\scriptsize$\pm$37 & \textbf{0.1\,\scriptsize$\pm$0.1} & 528\,\scriptsize$\pm$26 & 10.7\,\scriptsize$\pm$14.5 \\
 
\bottomrule
\end{tabular}}
\end{center}
\end{table}

\section{CDT cost-target sweep}
\label{app:cdtsweep}
CDT conditions on a target cost at deployment, so reporting it at a single target risks understating it. We retrain CDT on full data at each task's budget (three seeds, identical protocol) and evaluate every checkpoint across cost targets down to a fifth of the budget, 100 episodes per cell (Tables~\ref{tab:cdttargets} and~\ref{tab:cdttargetsb}).

The sweep separates two regimes. Where the pool admits it, conditioning below budget delivers real safety at high reward. All five velocity tasks, both Circle tasks, and, at target 2, all five BulletSafetyGym tasks, in each case with reward above every filter in the study. Where safety requires selectivity, the knob does not work at all: on the eight Goal and Button tasks CDT violates at every target, and realized cost is nearly flat in the target, moving from $40.9$ to $37.8$ on CarGoal1 and sitting near $100$ on PointButton1 across the whole range against budget 25. Figure~\ref{fig:paretotargets} shows the two families side by side. The entire CDT curve lies right of the budget on the Goal tasks, whichever target is chosen, while our $\alpha$ sweep traverses a frontier left of it. The gap between a conditioning target and realized cost is itself the diagnosis: CDT fits a return-cost frontier that these datasets do not contain, so asking for a lower cost cannot produce trajectories the data never demonstrated. This is the same reading the granularity analysis gives from the other direction, and it is why retraining CDT on a certified selection succeeds on exactly the tasks where its own conditioning fails (Section~\ref{sec:results}).

\begin{figure}[h]
\begin{center}
\includegraphics[width=0.98\linewidth]{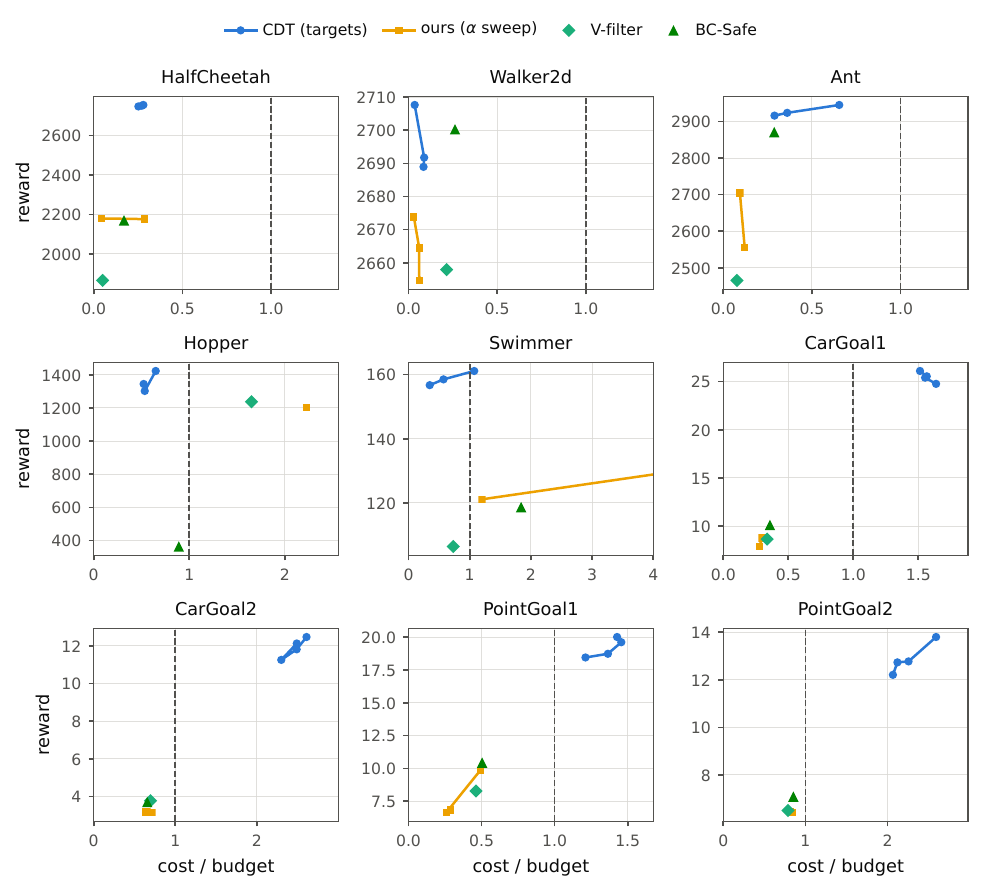}
\end{center}
\caption{Reward against cost over budget on the nine analysis tasks (dashed line: the budget). Blue: CDT on full data across its cost-target sweep, one point per target (values in Table~\ref{tab:cdttargets}). Orange: our calibrated filter across $\alpha \in \{0.10, 0.25, 0.40\}$, tighter $\alpha$ moving down-left. Teal: the V-filter at the calibration-estimated fraction. Green: BC-Safe with full labels. On the Goal tasks the entire CDT curve sits right of the budget line, so no target choice reaches the constraint, while the $\alpha$ sweep
traverses a genuine frontier left of it, within budget on every navigation task and at or inside BC-Safe's cost on all of them but CarGoal2 at $\alpha = 0.40$.}
\label{fig:paretotargets}
\end{figure}

\begin{table}[h]
\caption{CDT on full data across cost targets, fifteen DSRL tasks (reward/cost, three seeds, 100 episodes per cell; bold marks mean cost within the task budget of 20 or 25). The final column is the budget-matched target; velocity tasks have no target 15.}
\label{tab:cdttargets}
\begin{center}
\scriptsize
\setlength{\tabcolsep}{4pt}
\begin{tabular}{l rrrr}
\toprule
Task & target 5 & target 10 & target 15 & target = budget \\
\midrule
HalfCheetah & \textbf{2746/5.0} & \textbf{2750/5.4} & -- & \textbf{2754/5.6} \\
Walker2d & \textbf{2708/0.7} & \textbf{2692/1.8} & -- & \textbf{2689/1.7} \\
Ant & \textbf{2916/5.8} & \textbf{2924/7.2} & -- & \textbf{2945/13.1} \\
Hopper & \textbf{1344/10.5} & \textbf{1302/10.7} & -- & \textbf{1423/13.0} \\
Swimmer & \textbf{157/6.9} & \textbf{158/11.4} & -- & 161/21.5 \\
CarGoal1 & 24.76/40.9 & 25.40/38.8 & 25.54/39.1 & 26.09/37.8 \\
CarGoal2 & 12.12/62.2 & 11.25/57.4 & 11.81/62.2 & 12.46/65.2 \\
PointGoal1 & 18.45/30.3 & 18.73/34.1 & 19.61/36.4 & 20.01/35.7 \\
PointGoal2 & 12.21/51.7 & 12.73/53.0 & 12.78/56.4 & 13.80/64.8 \\
PointButton1 & 20.93/100.4 & 21.75/101.0 & 22.11/104.4 & 21.49/101.2 \\
PointButton2 & 19.22/103.3 & 19.21/101.4 & 18.93/97.6 & 18.65/97.0 \\
CarButton1 & 8.45/88.2 & 8.66/90.6 & 8.94/94.8 & 9.84/96.0 \\
CarButton2 & 7.07/114.9 & 6.83/103.8 & 8.03/100.2 & 7.52/112.9 \\
PointCircle1 & \textbf{42.0/4.2} & \textbf{42.2/5.4} & \textbf{42.5/5.1} & \textbf{42.7/13.2} \\
PointCircle2 & \textbf{40.2/10.2} & \textbf{40.1/14.9} & \textbf{40.4/18.8} & 40.9/32.2 \\
 
\bottomrule
\end{tabular}
\end{center}
\end{table}

\begin{table}[h]
\caption{CDT on full data across cost targets, five BulletSafetyGym tasks (reward/cost, three seeds; bold marks mean cost within budget 10).}
\label{tab:cdttargetsb}
\begin{center}
\scriptsize
\setlength{\tabcolsep}{4pt}
\begin{tabular}{l rrr}
\toprule
Task & target 2 & target 5 & target 10 \\
\midrule
BallRun & \textbf{446.9/7.4} & 447.7/13.3 & 439.1/14.3 \\
BallCircle & \textbf{598.2/9.4} & 611.4/11.0 & 632.2/14.1 \\
CarCircle & \textbf{376.8/6.6} & \textbf{376.8/7.3} & \textbf{377.8/8.9} \\
CarRun & \textbf{570.1/7.7} & \textbf{570.5/8.9} & 571.0/11.7 \\
DroneRun & \textbf{399.0/0.1} & \textbf{399.7/0.7} & \textbf{401.7/2.8} \\
 
\bottomrule
\end{tabular}
\end{center}
\end{table}

\section{Transfer to BulletSafetyGym}
\label{app:bullet}

Table~\ref{tab:bullet} reports the pipeline on five BulletSafetyGym tasks (BallRun, BallCircle, CarCircle, CarRun, DroneRun; 651 to 1990 trajectories; episode lengths 100 to 300; DroneRun is the one dataset in the study with true terminal states) at budget 10, the strictest of the DSRL protocol's three evaluation thresholds (10, 20, 40), with every hyperparameter unchanged from the main study and no per-task tuning. The slice is hard: ground-truth safe fractions are 9.0, 9.7, 11.4, 46.2, and 28.0 percent, and the 200 calibration labels are 10 to 31 percent of these smaller pools, extending the main study's 5-to-22-percent range. Short episodes are a different regime rather than a harder one for segment supervision, since a length-30 segment covers up to 30 percent of an episode; episode lengths are fixed except on DroneRun, where the score is essentially uncorrelated with length ($\rho = 0.13$). The score transfers. Its Spearman correlation with episodic cost is $-0.81$ to $-0.94$ on every task, and it is doing the work. Random selection at the identical matched fractions is unsafe on four of five tasks (costs 32 to 79), safe only on CarRun's 46-percent-safe pool. The strongest full-label baseline behaves differently here than on the main suite: full-data CDT at the budget-matched target 10 is unsafe on four of five tasks (costs 11.1 to 14.5), but conditioning its target down to 2 makes it safe on all five (costs 0.1 to 9.4) at reward far above every filter (Appendix~\ref{app:cdtsweep}). Unlike the Goal and Button tasks, where the same sweep leaves cost flat and far above budget, these pools are conditioning-tractable, so the transfer suite tests the pipeline's portability rather than exhibiting a baseline failure. Cloning everything is unsafe on four of five tasks (costs 24 to 65); the calibrated procedure satisfies the budget on all five, one of them (DroneRun) by the reward collapse just described, and the estimated-fraction V-filter on all five as well. Certification concentrates where safe mass supports it. CarRun certifies on five of five seeds, DroneRun on one, and the remaining three refuse on every seed, with the 2000-draw validation holding in every cell (worst unconditional false-certification rate $0.06$, 95 percent Clopper-Pearson $0.050$ to $0.071$). The operating curve transfers as well (Figure~\ref{fig:alphacurve}): resampling the calibration at looser levels, certification rates rise in safe-mass order, DroneRun from $0.20$ at $\alpha = 0.25$ to $0.87$ at $0.40$, BallRun from $0.00$ to $0.06$ over the same range, while CarRun certifies almost always ($0.995$ at $\alpha = 0.25$) and BallCircle and CarCircle never, so the $\alpha$ lever behaves on the new domain exactly as on the main suite.

One cell carries the section's real limitation. On DroneRun every filter, ours and the full-label one, drives cost to zero, but reward falls from $383$ under unfiltered cloning to $160$ for BC-Safe and $42$ for the calibrated filter. Safety here is bought by collapse rather than by selection, and the pipeline has no mechanism for noticing that its selection has become behaviorally degenerate while still satisfying the budget, a limitation this suite exposes for the first time. It is not caught by the machinery we have: a pre-registered check ran the certification machinery on the trivial selection $S = \Doff$, in both its composition and mean-cost forms, and it fires nowhere, because DroneRun's pool has mean cost 50.5 against budget 10 and 72 percent unsafe trajectories while its clone is safe by policy averaging. Neither composition-level test can detect that filtering is unwarranted; the detection appears to require a policy-level signal, which is the composition-to-policy gap of Section~\ref{sec:guarantee-validation} in its converse direction.

\paragraph{Composability is regime-dependent.} A pre-registered rule ran the composability echo on CarRun once full-data CDT violated there. The original certified selection for this echo was lost, so it was re-derived by the same certification procedure from the current value ensemble, which yields a purer and smaller selection than the one first used (realized contamination $0.051$ over 98 trajectories, $15$ percent of the pool, against $0.138$ over roughly half). On it, behavior cloning is safe at cost $0.65$ across three seeds. CDT was unsafe on all three seeds of the original selection ($11.4$ to $19.6$ against budget 10). On the regenerated selection it violates on one seed of three ($0.0$, $7.0$ and $11.1$), so the purer selection helps without taming it (Table~\ref{tab:operator}). Unlike the main suite, where certified selections tamed CDT in all 36 runs, curation is not sufficient here, and the reason is diagnostic: on CarRun even cloning the whole pool is safe (BC-All at cost $0.0$), so CDT's violation is caused by its operator rather than by the data, and curation cannot remove a violation the data did not cause. The certified return-weighted operator completes the diagnosis. On the same selection the three return-weighted intensities are safe on every seed (reward 480 to 513, cost at most $3.7$, Table~\ref{tab:operator}), so return-seeking as such is not what violates on CarRun; extrapolative return conditioning is. The hard top-half rule is the exception, violating on one seed of three (cost 30.0), which locates the boundary at truncation rather than at return-seeking: resampling preserves the selection's composition where hard truncation to the 49 highest-return trajectories does not. Across the five composability cases the pattern is consistent. Certified curation cures the full-label learner where the data drives the violation, on both
navigation tasks, where full-data CDT runs at nearly $1.5$ times budget; on HalfCheetah and Walker2d, where
full-data CDT is already within budget, there is nothing to cure and curation keeps it there; and it cannot help
on CarRun, where the pool is benign (BC-All at cost $0.0$) and the operator alone violates, the
boundary the granularity analysis predicts.

\begin{table}[h]
\caption{BulletSafetyGym transfer at budget 10 (mean $\pm$ 95\% CI where multiple seeds; bold marks mean cost within budget; BC-All deterministic; CDT trained on the full data and evaluated at the budget-matched target of 10, with its tighter targets in Table~\ref{tab:cdttargetsb}). CarRun certifies on five of five seeds and DroneRun on one of five; the other three refuse on every seed.}
\label{tab:bullet}
\begin{center}
\scriptsize
\setlength{\tabcolsep}{3pt}
\resizebox{\textwidth}{!}{%
\begin{tabular}{l rr rr rr rr rr}
\toprule
& \multicolumn{2}{c}{BC-All} & \multicolumn{2}{c}{BC-Safe} & \multicolumn{2}{c}{V-filter} & \multicolumn{2}{c}{Calibrated} & \multicolumn{2}{c}{CDT} \\
\cmidrule(lr){2-3}\cmidrule(lr){4-5}\cmidrule(lr){6-7}\cmidrule(lr){8-9}\cmidrule(lr){10-11}
Task & $R$ & $C$ & $R$ & $C$ & $R$ & $C$ & $R$ & $C$ & $R$ & $C$ \\
\midrule
BallRun & 825 & 52.4 & 276\,\scriptsize$\pm$66 & \textbf{7.6\,\scriptsize$\pm$8.5} & 287\,\scriptsize$\pm$65 & \textbf{5.3\,\scriptsize$\pm$6.8} & 329\,\scriptsize$\pm$48 & \textbf{7.0\,\scriptsize$\pm$9.6} & 439\,\scriptsize$\pm$8 & 14.5\,\scriptsize$\pm$7.4 \\
BallCircle & 601 & 24.3 & 455\,\scriptsize$\pm$21 & \textbf{5.3\,\scriptsize$\pm$2.7} & 298\,\scriptsize$\pm$13 & \textbf{0.0\,\scriptsize$\pm$0.0} & 307\,\scriptsize$\pm$14 & \textbf{0.2\,\scriptsize$\pm$0.3} & 632\,\scriptsize$\pm$18 & 13.2\,\scriptsize$\pm$1.3 \\
CarCircle & 364 & 64.6 & 210\,\scriptsize$\pm$32 & \textbf{7.2\,\scriptsize$\pm$6.7} & 121\,\scriptsize$\pm$17 & \textbf{2.4\,\scriptsize$\pm$1.4} & 152\,\scriptsize$\pm$17 & \textbf{2.9\,\scriptsize$\pm$1.6} & 384\,\scriptsize$\pm$1 & 11.1\,\scriptsize$\pm$1.5 \\
CarRun & 565 & \textbf{0.7} & 559\,\scriptsize$\pm$3 & \textbf{0.1\,\scriptsize$\pm$0.2} & 525\,\scriptsize$\pm$11 & \textbf{7.7\,\scriptsize$\pm$7.6} & 472\,\scriptsize$\pm$40 & \textbf{5.5\,\scriptsize$\pm$8.2} & 570\,\scriptsize$\pm$1 & 11.8\,\scriptsize$\pm$3.3 \\
DroneRun & 383 & 36.5 & 160\,\scriptsize$\pm$103 & \textbf{0.0\,\scriptsize$\pm$0.0} & 37\,\scriptsize$\pm$39 & \textbf{0.0\,\scriptsize$\pm$0.0} & 42\,\scriptsize$\pm$40 & \textbf{0.0\,\scriptsize$\pm$0.0} & 402\,\scriptsize$\pm$2 & \textbf{2.9\,\scriptsize$\pm$1.2} \\
 
\bottomrule
\end{tabular}}
\end{center}
\end{table}

\section{Properties of the certificate}
\label{app:certprops}

\begin{figure}[t]
\begin{center}
\includegraphics[width=0.8\linewidth]{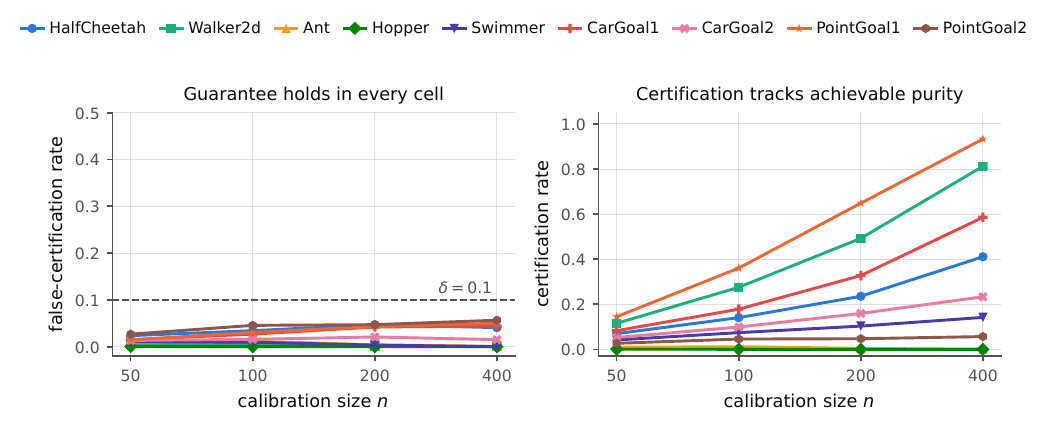}
\end{center}
\caption{Guarantee validation over 2000 calibration draws per cell; both panels draw the nine analysis tasks and plot means over the three value-ensemble seeds. Left: unconditional false-certification rate against the level $\delta = 0.1$ (dashed). Every individual task, seed and calibration-size cell is below $\delta$, the worst $0.078$, which Appendix~\ref{app:certprops} reports per cell rather than as the means drawn here. Right: certification rate as a function of calibration size $n$; tasks whose selections cannot meet $\alpha$ certify rarely, and the false-certification rate Proposition~\ref{prop:guarantee} bounds stays far below $\delta$ for them (at $n = 200$: Ant $0.004$, PointGoal2 $0.047$, Hopper $0.000$).}
\label{fig:coverage}
\end{figure}

\paragraph{Conditional false-certification.} Table~\ref{tab:condviol} reports, per task at $n = 200$ and $\alpha = 0.25$ from the 2000-draw resample, the certification rate, the unconditional false-certification rate the theorem bounds, and the conditional violation rate a certificate-holder faces. The unconditional rate is below $\delta$ everywhere, and over all tasks, seeds and calibration sizes the worst cell is CarGoal1 at seed 1 and $n = 400$, at $0.078$ with Clopper-Pearson interval $0.067$ to $0.091$, so the bound holds across the interval and not only at the point estimate; the conditional rate is governed by the certification rate and ranges from $0.07$ (PointGoal1) to near one on tasks certifying in under five percent of draws, where the handful of certificates that do fire are disproportionately the optimistic draws. The practical reading is in Section~\ref{sec:guarantee-validation}. Certification rate is itself a trust signal.

\begin{table}[h]
\caption{Certification rate, unconditional false-certification, and conditional violation rate at $n = 200$, $\alpha = 0.25$ (2000 draws per seed, seed means; final column: certified draws pooled over seeds).}
\label{tab:condviol}
\begin{center}
\small
\begin{tabular}{l rrrr}
\toprule
Task & cert.\ rate & $\Pr[\text{cert} \wedge \text{viol}]$ & $\Pr[\text{viol} \mid \text{cert}]$ & certified draws \\
\midrule
HalfCheetah & 0.24 & 0.047 & 0.20 & 1413 \\
Walker2d & 0.49 & 0.002 & 0.00 & 2955 \\
Ant & 0.00 & 0.004 & 1.00 & 25 \\
Hopper & 0.00 & 0.000 & -- & 0 \\
Swimmer & 0.10 & 0.004 & 0.03 & 619 \\
CarGoal1 & 0.33 & 0.042 & 0.13 & 1964 \\
CarGoal2 & 0.16 & 0.021 & 0.13 & 956 \\
PointGoal1 & 0.65 & 0.043 & 0.07 & 3891 \\
PointGoal2 & 0.05 & 0.047 & 1.00 & 283 \\
BallRun & 0.00 & 0.000 & 1.00 & 2 \\
BallCircle & 0.00 & 0.000 & -- & 0 \\
CarCircle & 0.00 & 0.000 & -- & 0 \\
CarRun & 1.00 & 0.003 & 0.00 & 5981 \\
DroneRun & 0.17 & 0.027 & 0.16 & 1029 \\
 
\bottomrule
\end{tabular}
\end{center}
\end{table}

\paragraph{The certificate as a triple.} The rate in Table~\ref{tab:condviol} is measured by resampling, which a deployer holding one calibration sample cannot do. Proposition~\ref{prop:rate} makes it estimable instead: the walk certifies exactly when it rejects at $\lambda_1$, so the rate is a hypergeometric function of $U_1$, and the $(m_1, k_1)$ counts that fell in $S(\lambda_1)$ give $U_1$ a posterior (Jeffreys $\mathrm{Beta}(k + \tfrac12, m - k + \tfrac12)$ on $u_1$, $200$ draws). Its mean is the estimated rate $\hat r$, and its lower decile $r_{\mathrm{lo}}$ converts the unconditional guarantee into a deployer-facing bound, since $\Pr[\text{certify and out of spec}] \le \delta$ gives 
\[
  \Pr[\text{out of spec} \mid \text{certify}] \;\le\; \delta / \Pr[\text{certify}].
\]
The procedure therefore returns a threshold, $\hat r$, and $\min(1, \delta / r_{\mathrm{lo}})$, all three from the labels already spent. Table~\ref{tab:certtriple} validates them against the resampled truth at $n = 200$ over 500 draws per task and seed. The estimate is close where certification is common and biased upward where it is rare (Walker2d $0.51$ against a realized $0.52$, PointGoal1 $0.59$ against $0.65$, but Ant $0.06$ against $0.01$), because a posterior over $U_1$ cannot concentrate on a pool whose top quantile almost never passes. That bias is exactly where the bound fails: across the $40$ cells with a defined conditional rate it covers the realized value in $25$, and it covers in \emph{all} $24$ cells that certify in at least a tenth of draws, failing only on cells certifying in at most $7$ percent, where $\hat r$ overstates the rate by three- to thirteenfold and $\delta / r_{\mathrm{lo}}$ is correspondingly too small. The honest reading is that the triple is informative in the regime where a practitioner would act on a certificate at all, and that the certification rate is the signal telling them whether they are in it. The last column of Table~\ref{tab:certtriple} shows why the obvious alternative is worse: a plug-in posterior for the returned selection's out-of-spec probability averages $0.04$ against a realized $0.43$, because a certificate fires precisely when the sample looks cleaner than the pool, so the plug-in is optimistic under selection and is reported only to be ruled out.

\begin{table}[h]
\caption{The certificate as a triple, validated at $n = 200$, $\alpha = 0.25$ over 500 fresh calibration draws per task and value-ensemble seed (seed means). Columns: realized certification rate; the estimated rate $\hat r$ the procedure returns; the realized conditional violation rate among certified draws; the returned bound $\min(1, \delta / r_{\mathrm{lo}})$ on it; the plug-in posterior reported only to show it is optimistic under selection; and the number of certified draws the conditional columns are computed over. Dashes mark tasks that never certify, where no conditional quantity is defined.}
\label{tab:certtriple}
\begin{center}
\small
\begin{tabular}{l rrrrrr}
\toprule
Task & cert.\ rate & $\hat r$ & $\Pr[\text{viol} \mid \text{cert}]$ & bound & plug-in & $n_{\text{cert}}$ \\
\midrule
HalfCheetah & 0.25 & 0.34 & 0.21 & 0.52 & 0.03 & 376 \\
Walker2d & 0.52 & 0.51 & 0.00 & 0.46 & 0.03 & 777 \\
Ant & 0.01 & 0.06 & 1.00 & 0.73 & 0.05 & 8 \\
Hopper & 0.00 & 0.00 & -- & -- & -- & 0 \\
Swimmer & 0.10 & 0.23 & 0.05 & 0.57 & 0.04 & 147 \\
CarGoal1 & 0.36 & 0.42 & 0.13 & 0.51 & 0.04 & 546 \\
CarGoal2 & 0.15 & 0.27 & 0.14 & 0.55 & 0.04 & 220 \\
PointGoal1 & 0.65 & 0.59 & 0.06 & 0.41 & 0.05 & 982 \\
PointGoal2 & 0.05 & 0.16 & 1.00 & 0.65 & 0.04 & 68 \\
PointButton1 & 0.01 & 0.09 & 1.00 & 0.62 & 0.05 & 19 \\
PointButton2 & 0.00 & 0.01 & -- & -- & -- & 0 \\
CarButton1 & 0.00 & 0.02 & -- & -- & -- & 0 \\
CarButton2 & 0.00 & 0.03 & 1.00 & 0.40 & 0.04 & 3 \\
PointCircle1 & 0.08 & 0.20 & 0.67 & 0.62 & 0.05 & 127 \\
PointCircle2 & 0.00 & 0.02 & -- & -- & -- & 0 \\
BallRun & 0.00 & 0.03 & 1.00 & 0.68 & 0.05 & 1 \\
BallCircle & 0.00 & 0.01 & -- & -- & -- & 0 \\
CarCircle & 0.00 & 0.00 & -- & -- & -- & 0 \\
CarRun & 1.00 & 0.92 & 0.00 & 0.15 & 0.02 & 1496 \\
DroneRun & 0.20 & 0.30 & 0.14 & 0.58 & 0.04 & 295 \\
 
\bottomrule
\end{tabular}
\end{center}
\end{table}

\paragraph{Multiple-testing procedures over the grid.} Fixed-sequence testing could in principle be beaten by procedures that reach deep thresholds without passing shallow ones. On the archived scores (2000 draws per cell), Bonferroni and Holm at level $\delta$ over the twelve-point grid certify far less everywhere, PointGoal1 at $0.295$ against fixed-sequence's $0.654$ and Walker2d at $0.295$ against $0.511$, with validity intact for all three (worst false-certification $0.063$ against $\delta = 0.1$). The $\delta/J$ penalty costs more than deep-threshold access returns at $n = 200$. Certification yield is limited by per-test power at this calibration size, not by the choice of procedure, and the deployed fixed-sequence walk is the power-optimal option among those tested.

\paragraph{Label complexity of certification.} Table~\ref{tab:labelcomplexity} reports the calibration budget at which each task first certifies in half of draws, obtained by extending the resampling sweep to $n \in \{800, 1600, 3200\}$ on the tasks with positive purity margin. Regressing the crossing point on the margin over the seven tasks that cross gives $\log n_{50} = -1.61 \log(\text{margin}) + 1.66$ with $R^2 = 0.997$, and a paired bootstrap over tasks puts the exponent at $1.51$ to $1.67$, below the $1/\text{margin}^2$ rate in $99.98$ percent of resamples. The one task with a small positive margin that does not reach half certification even at $n = 3200$ is Swimmer at $0.014$, consistent with the fit's prediction of a budget in the thousands; every task with a non-positive margin certifies at essentially zero for every budget, as Proposition~\ref{prop:dichotomy} requires.

\begin{table}[h]
\caption{Calibration budget $n_{50}$ at which certification reaches half of draws, against the purity margin (nine analysis tasks and five transfer tasks; resampling at $n$ up to $3200$).}
\label{tab:labelcomplexity}
\begin{center}
\small
\begin{tabular}{l rr}
\toprule
Task & purity margin & $n_{50}$ \\
\midrule
CarRun & $0.209$ & $64$ \\
PointGoal1 & $0.127$ & $140$ \\
Walker2d & $0.110$ & $203$ \\
CarGoal1 & $0.079$ & $318$ \\
HalfCheetah & $0.061$ & $501$ \\
DroneRun & $0.046$ & $728$ \\
CarGoal2 & $0.037$ & $1064$ \\
Swimmer & $0.014$ & $> 3200$ \\
\midrule
6 tasks with margin $\le 0$ & $-0.011$ to $-0.380$ & never \\
 
\bottomrule
\end{tabular}
\end{center}
\end{table}

\paragraph{The fallback's conservatism.} A lower confidence bound sits below the estimate it bounds, so on an uncertified task the fallback always keeps less than the V-filter at the calibration-estimated fraction: the same scores, a smaller and purer selection. Both consequences are visible in Table~\ref{tab:main}. On Hopper, whose pool is contaminated enough that the estimated fraction admits unsafe trajectories, the smaller selection is what brings the clone from $33$ to $3$. On Swimmer, whose pool is already clean at either fraction, the same conservatism removes data the clone needed and costs it the budget. The rule is therefore not uniformly better than the estimate; it is the one that keeps the guarantee's level, which is why it was fixed in advance.

\paragraph{A two-tier fallback.} When no threshold certifies, the procedure returns the fraction given by a Clopper-Pearson lower bound on the pool's safe mass. A natural alternative reruns the identical fixed-sequence test at even-odds evidence ($\delta = 0.5$) and takes the deepest passing threshold. We reject it on principle. Relaxing the evidence standard buys a deeper selection by abandoning the level at which the guarantee is stated, and the result is still returned as uncertified, so the procedure acquires the appearance of a test without its protection; the $\alpha$ lever relaxes the target instead and keeps a certificate. Its scores do not argue against it. Across the nine tasks it was run on, it changes two safety verdicts, converting Swimmer ($58.9$ to $18.0$) but breaking Hopper ($3.4$ to $43.8$), and it raises mean cost on five of the nine, including PointCircle2 ($87.2$ to $112.5$).

\paragraph{Mean-cost certificate.} The same machinery certifies a budget-native risk directly, by replacing the unsafe-fraction risk with the mean episodic cost of the selection (Hoeffding-without-replacement p-values, target, namely selection mean cost at or below the budget) yields a strictly more conservative variant that certifies only on PointGoal1 (5 to 8 percent of draws at $n = 200$) and records zero false certifications anywhere. The unsafe-fraction form remains the practical operating certificate; the mean-cost form documents that the framework extends to CMDP-native risks without modification. It also admits a loose bridge to policy cost: writing $\pi_S$ for the behavior policy of the selection and $\epsilon$ for the cloning error in average total variation, standard imitation arguments give $J_c(\pi) \le J_c(\pi_S) + 2T^2 c_{\max}\, \epsilon$, and $J_c(\pi_S)$ concentrates around the selection mean cost that this certificate bounds. The $T^2$ compounding makes the bound far too loose to certify in practice, but it separates
the remaining gap into exactly two quantities: the certified selection mean and the cloning error.

\section{Certificates on the deployed policy}
\label{app:policycert}

The certificate of Section~\ref{sec:calibration-method} bounds the composition of the
training set, not the cost of the policy learned from it, and Section~\ref{sec:limits}
states that boundary. The gap can be closed from the other side, without any
concentrability argument \citep{rashidinejad2022bridging,xie2021policy,zhan2022offline} and without
strengthening the supervisor. Once training ends
the clone is a fixed function, so its evaluation episodes are independent draws from the
task's initial-state distribution and the number exceeding the budget is binomial. Both
certificates below therefore consume exactly the supervision the pipeline already
assumes: one bit per episode, whether it went over budget. Each bound holds per cell,
with no joint claim across the seventy-five; counts of certifying tasks aggregate
separate bounds rather than asserting a simultaneous guarantee. They differ from the
composition certificate in one respect that matters. That certificate is computed from
the dataset alone and is available before the policy is trained or run; these are
computed from executed episodes, so they are deployment-time guarantees and require
the interaction an offline setting may not permit. We report them as a complement to
the offline certificate, not a replacement for it.

\paragraph{Violation probability.} Counting the episodes over budget among the $100$ evaluation episodes already
reported per cell and inverting the exact binomial gives a Clopper-Pearson upper limit on
$\Pr[\text{cost} > \text{budget}]$ at level $\delta$, reported in
Table~\ref{tab:policycert}, whose pooled values Section~\ref{sec:guarantee-validation}
quotes: curation cuts the deployment violation probability by roughly two thirds and
lands within two points of the full-label oracle, on a criterion stricter than the one the
benchmark scores. The same table records
a property of that criterion rather than of any method here: a quarter of episodes
exceed budget even for the oracle, because mean-cost safety constrains an average and
leaves the tail free.

\paragraph{Expected cost.} The tail bound is not the paper's safety criterion, which is
$\mathbb{E}[\text{cost}] \le$ budget, and bits carry no magnitudes. Bounding the mean
needs the cost values on the evaluation episodes, and it needs more of them: with
episodic cost in $[0, T]$ the range term of an empirical-Bernstein bound dominates at
$n = 100$ however favorable the observed costs are. We therefore rolled out $2000$
fresh episodes per cell on the same reported policies, with no retraining, and applied Maurer-Pontil with the range fixed
a priori at $T = 1000$ rather than taken from the observed
maximum. Table~\ref{tab:meancostcert} reports the result: $\mathbb{E}[\text{cost}]$ is
certified within budget on ten of the fifteen tasks, on every seed, at $\delta = 0.1$.
Two tasks are safe by point estimate but do not certify: CarButton2 misses by $1.1$ on
the bound while every one of its seeds is within budget, and PointCircle1 carries two
broken seeds at $59.1$ and $27.6$ against $4.7$ to $7.3$ for the other three. Such seeds
are not visible before deployment, and the selection statistics barely help: across the
$75$ calibrated cells, task-demeaned, the Spearman correlation of the $2000$-episode cost
with selection size, kept fraction and realized contamination is $0.32$, $0.36$ and
$0.30$. Each is significant at $p < 0.01$, so the association is real, but each accounts
for about a tenth of the variance, far too little to screen a seed before it is run.
Most failures are broad rather than rare: on eight of the fifteen failing cells the
median episode itself exceeds the budget, so the clone has learned an unsafe behavior
rather than an occasionally unlucky one. PointCircle1's two show both modes, medians of
$54$ and $14$ against a budget of $25$. That is the case for a
certificate on the deployed policy: the composition certificate audits the data, and
the seed that misuses it can only be caught by running it.

The larger evaluation also re-tests the headline. Comparing the $2000$-episode estimates
against the $100$-episode ones, all fifteen safe-or-unsafe verdicts are
unchanged and the safe-task count is twelve under both, so
Table~\ref{tab:main}'s result replicates at twenty times the evaluation budget.

\begin{table}[h]
\caption{Deployment violation probability: exact Clopper-Pearson upper limit on
$\Pr[\text{cost} > \text{budget}]$ at $\delta = 0.1$, from the $100$-episode evaluations
of Table~\ref{tab:main} (five seeds pooled per task; BC-All is a single training).
Pooling makes the bound a statement about a policy drawn uniformly from the five
training seeds, not about any one deployed clone; per-seed bounds are wider.}
\label{tab:policycert}
\begin{center}
\scriptsize
\begin{tabular}{l rrrr}
\toprule
Task & BC-All & BC-Safe & V-filter & Calibrated \\
\midrule
HalfCheetah & 0.449 & 0.076 & 0.008 & 0.060 \\
Walker2d & 0.138 & 0.091 & 0.104 & 0.026 \\
Ant & 1.000 & 0.056 & 0.005 & 0.005 \\
Hopper & 0.961 & 0.246 & 0.397 & 0.045 \\
Swimmer & 0.815 & 0.486 & 0.262 & 0.427 \\
CarGoal1 & 0.408 & 0.147 & 0.137 & 0.117 \\
CarGoal2 & 0.519 & 0.252 & 0.266 & 0.262 \\
PointGoal1 & 0.637 & 0.221 & 0.210 & 0.175 \\
PointGoal2 & 0.860 & 0.350 & 0.301 & 0.318 \\
PointButton1 & 0.588 & 0.270 & 0.227 & 0.166 \\
PointButton2 & 0.742 & 0.443 & 0.427 & 0.375 \\
CarButton1 & 0.479 & 0.272 & 0.191 & 0.206 \\
CarButton2 & 0.428 & 0.260 & 0.235 & 0.248 \\
PointCircle1 & 0.999 & 0.085 & 0.154 & 0.289 \\
PointCircle2 & 0.912 & 0.648 & 0.846 & 0.867 \\
 
\bottomrule
\end{tabular}
\end{center}
\end{table}

\begin{table}[h]
\caption{Expected-cost certificate from $2000$ fresh episodes per seed on the deployed
policies: mean cost, the Maurer-Pontil upper bound at $\delta = 0.1$ taken over the
worst seed, and the budget. Bold marks a bound within budget; the final column requires
every seed to certify.}
\label{tab:meancostcert}
\begin{center}
\scriptsize
\begin{tabular}{l rrrc}
\toprule
Task & mean cost & $\mathbb{E}[\text{cost}] \le$ & budget & certified \\
\midrule
HalfCheetah & 2.31 & \textbf{12.2} & 20 & yes \\
Walker2d & 1.11 & \textbf{8.4} & 20 & yes \\
Ant & 1.50 & \textbf{5.8} & 20 & yes \\
Hopper & 3.88 & \textbf{13.1} & 20 & yes \\
Swimmer & 59.69 & 221.3 & 20 & no \\
CarGoal1 & 8.07 & \textbf{12.9} & 25 & yes \\
CarGoal2 & 17.33 & \textbf{24.1} & 25 & yes \\
PointGoal1 & 9.36 & \textbf{16.2} & 25 & yes \\
PointGoal2 & 18.47 & \textbf{24.0} & 25 & yes \\
PointButton1 & 13.61 & \textbf{21.3} & 25 & yes \\
PointButton2 & 28.93 & 39.4 & 25 & no \\
CarButton1 & 14.81 & \textbf{22.3} & 25 & yes \\
CarButton2 & 18.40 & 26.1 & 25 & no \\
PointCircle1 & 20.96 & 65.8 & 25 & no \\
PointCircle2 & 89.03 & 130.8 & 25 & no \\
 
\bottomrule
\end{tabular}
\end{center}
\end{table}

\section{Extended ablation discussions}
\label{app:extended}

\subsection{Controls on the score and the selection}

\paragraph{Calibration-source and score controls.} Replacing the labeled calibration entirely with held-out preferences, thresholding at the $0.9$ quantile of dispreferred-parent scores, is safe on four of fifteen tasks and unsafe on the rest, including every low-safe-rate velocity task, so labeled calibration earns its keep exactly where safety is scarce. Flipping up to 30 percent of labels, a Boltzmann labeler that errs on borderline pairs, or a budget of 100 pairs degrades filter precision only mildly and never produces a false certification: validity flows from the clean calibration sample, not from score quality. The same machinery certifies the budget-native risk, the selection's mean episodic cost, in a strictly more conservative variant that certifies only on PointGoal1 with zero false certifications. Finally, $Q(s, a)$ under the identical Bradley-Terry loss leaves four of five representative tasks unchanged and is suggestively worse on Hopper while forfeiting the state-only value's policy-agnostic scoring, and oracle reward-to-go through the identical extraction raises reward while cost explodes, so the pipeline transmits whatever signal the value carries.

This appendix carries the selection-signal control table, the full versions of the ablation discussions summarized in Section~\ref{sec:ablations}, and per-cell detail for the gated variant of Section~\ref{sec:results}.

\paragraph{Selection-signal controls.} Table~\ref{tab:controls} reports both matched-size controls in full. Return selection buys reward on all fifteen tasks and pays for it with higher cost than random selection on twelve
of the fifteen, so ranking by return is not merely uninformative about safety but anti-aligned with it in these datasets; random selection's cost, by contrast, tracks each pool's base contamination and reads as a no-information reference for the selection operator. Table~\ref{tab:scorecorr} reports the score's Spearman correlations with episodic cost and return per task, and the cost-return correlation that makes some anti-alignment with return unavoidable on the navigation tasks. The V-filter's second fraction completes a small sweep over selection size: at a fixed top 25 percent it is safe on ten of fifteen, close to the estimated fraction everywhere except Hopper, whose ground-truth safe pool is 11 percent of its dataset, so a 25 percent selection necessarily dilutes it (cost 60 against 20 at the ground-truth-matched fraction of Table~\ref{tab:controls}). Safety tracks whether the fraction stays within the task's safe mass, which is exactly the quantity the calibrated threshold adapts to per task. The bottom-return column is the sharpest of the three controls. Its selected data is genuinely low-cost (on HalfCheetah, kept-set mean trajectory cost 39.9 against the pool's 115.4), yet its clone runs at cost 109: low-return trajectories are behaviorally incoherent, and cloning them reproduces neither their returns nor their costs, a second demonstration that training-set composition alone does not determine policy safety. Where bottom-return is safe, it is safe by collapse (Swimmer, reward 4 at cost 1) or with thinner margins than the V-filter at the same selection size (PointGoal1 cost 20.1 against 10.7), and it breaks on PointGoal2 (43.7) where the V-filter is safe at 19.8. On PointCircle1, where cost and return are nearly collinear ($+0.89$), the score itself approaches an anti-ranking of return ($-0.95$) and anti-return selection is accordingly safe there; the velocity tasks are where the two signals separate.

\begin{table}[h]
\caption{Spearman correlations of the trajectory score $g$ with episodic cost and episodic return (mean over V-ensemble seeds), and between cost and return, per task.}
\label{tab:scorecorr}
\begin{center}
\small
\begin{tabular}{l rrr}
\toprule
Task & $\rho(g, \text{cost})$ & $\rho(g, \text{return})$ & $\rho(\text{cost}, \text{return})$ \\
\midrule
HalfCheetah & -0.90 & -0.31 & +0.35 \\
Walker2d & -0.83 & +0.22 & +0.11 \\
Ant & -0.89 & -0.34 & +0.50 \\
Hopper & -0.72 & +0.20 & +0.29 \\
Swimmer & -0.90 & -0.36 & +0.35 \\
CarGoal1 & -0.65 & -0.65 & +0.31 \\
CarGoal2 & -0.67 & -0.56 & +0.42 \\
PointGoal1 & -0.66 & -0.79 & +0.40 \\
PointGoal2 & -0.75 & -0.59 & +0.36 \\
PointButton1 & -0.55 & -0.44 & +0.31 \\
PointButton2 & -0.61 & -0.47 & +0.28 \\
CarButton1 & -0.63 & -0.44 & +0.27 \\
CarButton2 & -0.64 & -0.53 & +0.44 \\
PointCircle1 & -0.93 & -0.95 & +0.89 \\
PointCircle2 & -0.70 & -0.90 & +0.84 \\
 
\bottomrule
\end{tabular}
\end{center}
\end{table}

\begin{table}[h]
\caption{Selection-signal controls at matched selection sizes (five seeds, mean $\pm$ 95\% CI; bold marks mean cost within budget). The V-score column is the V-filter at the ground-truth safe fraction, the labeled diagnostic of Section~\ref{sec:experiments}, not the calibration-estimated V-filter of Table~\ref{tab:main}. Ranking by return, the percentile-BC analogue, is unsafe on every task; random selection is unsafe on fourteen of fifteen; bottom-return selection is safe on eight of fifteen and catastrophically unsafe on the velocity tasks. The safety value is doing the work.}
\label{tab:controls}
\begin{center}
\scriptsize
\setlength{\tabcolsep}{3.5pt}
\resizebox{\textwidth}{!}{%
\begin{tabular}{l rr rr rr rr}
\toprule
& \multicolumn{2}{c}{Random selection} & \multicolumn{2}{c}{Return selection} & \multicolumn{2}{c}{Bottom-return} & \multicolumn{2}{c}{V-score (ours)} \\
\cmidrule(lr){2-3}\cmidrule(lr){4-5}\cmidrule(lr){6-7}\cmidrule(lr){8-9}
Task & $R$ & $C$ & $R$ & $C$ & $R$ & $C$ & $R$ & $C$ \\
\midrule
HalfCheetah & 2646\,\scriptsize$\pm$41 & 84\,\scriptsize$\pm$37 & 2805\,\scriptsize$\pm$18 & 105\,\scriptsize$\pm$40 & 1739\,\scriptsize$\pm$37 & 109\,\scriptsize$\pm$8 & 2103\,\scriptsize$\pm$190 & \textbf{3\,\scriptsize$\pm$3} \\
Walker2d & 2703\,\scriptsize$\pm$81 & 73\,\scriptsize$\pm$42 & 2849\,\scriptsize$\pm$80 & 69\,\scriptsize$\pm$21 & 1981\,\scriptsize$\pm$198 & 109\,\scriptsize$\pm$3 & 2696\,\scriptsize$\pm$12 & \textbf{1\,\scriptsize$\pm$1} \\
Ant & 2924\,\scriptsize$\pm$8 & 193\,\scriptsize$\pm$39 & 2984\,\scriptsize$\pm$14 & 96\,\scriptsize$\pm$19 & 2847\,\scriptsize$\pm$33 & 189\,\scriptsize$\pm$52 & 2540\,\scriptsize$\pm$41 & \textbf{1\,\scriptsize$\pm$1} \\
Hopper & 613\,\scriptsize$\pm$170 & 85\,\scriptsize$\pm$46 & 1404\,\scriptsize$\pm$283 & 139\,\scriptsize$\pm$37 & 218\,\scriptsize$\pm$10 & 48\,\scriptsize$\pm$7 & 1326\,\scriptsize$\pm$51 & 20\,\scriptsize$\pm$19 \\
Swimmer & 128\,\scriptsize$\pm$24 & 108\,\scriptsize$\pm$60 & 174\,\scriptsize$\pm$2 & 93\,\scriptsize$\pm$8 & 4\,\scriptsize$\pm$1 & \textbf{1\,\scriptsize$\pm$1} & 122\,\scriptsize$\pm$8 & 34\,\scriptsize$\pm$25 \\
CarGoal1 & 15.35\,\scriptsize$\pm$0.89 & \textbf{20.19\,\scriptsize$\pm$2.31} & 24.31\,\scriptsize$\pm$0.53 & 36.34\,\scriptsize$\pm$1.95 & 9.71\,\scriptsize$\pm$0.89 & \textbf{9.22\,\scriptsize$\pm$1.42} & 9.36\,\scriptsize$\pm$0.39 & \textbf{9.66\,\scriptsize$\pm$1.90} \\
CarGoal2 & 8.50\,\scriptsize$\pm$0.44 & 39.92\,\scriptsize$\pm$2.59 & 13.78\,\scriptsize$\pm$0.41 & 66.92\,\scriptsize$\pm$2.90 & 4.11\,\scriptsize$\pm$0.24 & \textbf{18.69\,\scriptsize$\pm$1.76} & 3.84\,\scriptsize$\pm$0.23 & \textbf{16.83\,\scriptsize$\pm$2.42} \\
PointGoal1 & 18.72\,\scriptsize$\pm$0.45 & 36.54\,\scriptsize$\pm$3.85 & 21.16\,\scriptsize$\pm$0.25 & 37.85\,\scriptsize$\pm$1.17 & 9.44\,\scriptsize$\pm$0.47 & \textbf{20.07\,\scriptsize$\pm$2.60} & 8.61\,\scriptsize$\pm$0.28 & \textbf{10.72\,\scriptsize$\pm$1.89} \\
PointGoal2 & 14.17\,\scriptsize$\pm$0.56 & 73.73\,\scriptsize$\pm$4.96 & 19.82\,\scriptsize$\pm$0.25 & 116.21\,\scriptsize$\pm$1.21 & 7.74\,\scriptsize$\pm$0.72 & 43.65\,\scriptsize$\pm$6.70 & 6.84\,\scriptsize$\pm$0.35 & \textbf{19.75\,\scriptsize$\pm$1.96} \\
PointButton1 & 5.70\,\scriptsize$\pm$0.59 & 43.43\,\scriptsize$\pm$3.48 & 17.09\,\scriptsize$\pm$0.59 & 101.90\,\scriptsize$\pm$5.92 & -0.00\,\scriptsize$\pm$0.75 & \textbf{24.12\,\scriptsize$\pm$2.15} & 0.81\,\scriptsize$\pm$0.28 & \textbf{20.52\,\scriptsize$\pm$4.81} \\
PointButton2 & 10.62\,\scriptsize$\pm$0.58 & 59.62\,\scriptsize$\pm$3.35 & 22.58\,\scriptsize$\pm$0.15 & 116.88\,\scriptsize$\pm$5.17 & 3.80\,\scriptsize$\pm$0.35 & 32.69\,\scriptsize$\pm$2.73 & 4.41\,\scriptsize$\pm$0.26 & 27.57\,\scriptsize$\pm$2.64 \\
CarButton1 & 2.84\,\scriptsize$\pm$1.58 & 54.80\,\scriptsize$\pm$7.30 & 3.76\,\scriptsize$\pm$1.25 & 115.58\,\scriptsize$\pm$9.41 & -1.78\,\scriptsize$\pm$0.95 & \textbf{19.75\,\scriptsize$\pm$1.34} & 1.06\,\scriptsize$\pm$0.52 & \textbf{18.50\,\scriptsize$\pm$4.25} \\
CarButton2 & -1.02\,\scriptsize$\pm$1.32 & 52.40\,\scriptsize$\pm$4.27 & -0.78\,\scriptsize$\pm$0.86 & 125.26\,\scriptsize$\pm$10.11 & -1.95\,\scriptsize$\pm$0.31 & \textbf{18.65\,\scriptsize$\pm$2.43} & -0.82\,\scriptsize$\pm$0.54 & \textbf{21.56\,\scriptsize$\pm$3.14} \\
PointCircle1 & 47.9\,\scriptsize$\pm$4.5 & 122.9\,\scriptsize$\pm$36.0 & 57.6\,\scriptsize$\pm$0.9 & 195.1\,\scriptsize$\pm$6.0 & 32.3\,\scriptsize$\pm$3.9 & \textbf{4.0\,\scriptsize$\pm$1.2} & 37.7\,\scriptsize$\pm$1.5 & \textbf{10.2\,\scriptsize$\pm$2.3} \\
PointCircle2 & 44.8\,\scriptsize$\pm$2.5 & 214.6\,\scriptsize$\pm$47.2 & 48.1\,\scriptsize$\pm$0.7 & 245.5\,\scriptsize$\pm$6.2 & 36.1\,\scriptsize$\pm$1.0 & 71.3\,\scriptsize$\pm$34.1 & 36.5\,\scriptsize$\pm$1.0 & 89.3\,\scriptsize$\pm$18.6 \\
 
\bottomrule
\end{tabular}}
\end{center}
\end{table}

\paragraph{Gated reward-aware variant, per cell.} On the four certified tasks the gated sub-selection keeps cost within budget in every cell of the distinct-selection grid (three learner seeds per selection, so these differ from the five-seed deployed means of Table~\ref{tab:main}): HalfCheetah $0.2$ to $7.2$, Walker2d $1.3$ to $18.7$, CarGoal1 $6.6$ to $9.4$ and PointGoal1 $5.1$ to $6.6$, with reward held within confidence intervals throughout. Walker2d's widest cell is the one to read with care: at $18.7$ against budget $20$ it is within specification but with little margin, on the selection whose realized contamination is lowest, so the intensity rather than the composition is what consumes the headroom there. Ungated, the identical top-half sub-selection applied to Hopper's contaminated selection multiplies mean cost to 82, reaching 289 on one seed; the certificate's refusal is what prevents this, the decision-signal role of Section~\ref{sec:gating} in action.

\paragraph{Pipeline sanity and the action-conditioned oracle.} Two further controls close the analysis. Feeding oracle reward-to-go through the identical extraction machinery, a single-task control on HalfCheetah over three seeds, raises reward above behavior cloning on every seed while cost explodes, confirming the pipeline transmits whatever signal the value carries; its gentleness, a few percent of movement, is the compositional mechanism of the obstacle. And granting the per-transition route everything it lacks at once, an action-conditioned, long-horizon, generalizing value distilled from full ground-truth cost labels, scored counterfactually through the dynamics ensemble, still does not reach safety at its documented configuration (PointGoal1 cost $31.5$, CarGoal2 $39.5$, PointGoal2 $49.6$, against budget 25), and it buys reward rather than safety (on PointGoal2, reward $12.2$ against the selection route's $6.2$). Even the maximally endowed per-transition operator does not produce trajectory-level safety on these tasks, which is the obstacle in its strongest form.

\paragraph{Action-conditioned value.} Replacing $V(s)$ with $Q(s, a)$ trained by the identical Bradley-Terry segment-sum loss, changing only the input to the concatenation of state and action, leaves four of five representative tasks unchanged (HalfCheetah, PointGoal1, PointGoal2, CarButton1: safe, comparable reward). On Hopper, the task with the scarcest safe fraction, $Q$ exceeds the budget on two of three seeds (costs 35.8, 37.2, 14.2; mean 29.1) where $V$ is safe at $18 \pm 7$ over five seeds; given Hopper's seed variance this is suggestive rather than conclusive, but the direction is consistent with the mechanism we would expect, since conditioning on actions ties the score to the behavior policy's action choices at exactly the states where selection is hardest. Action-conditioning also costs the properties that the state-only value provides elsewhere in the paper. The safety landscape and constraint-saliency analyses of Section~\ref{sec:analysis} read $V$ directly as a function of the state, while $Q$ admits no such reading without first marginalizing over actions under some policy. And although $V$ is shaped by the behavior distribution it was trained on, it does not condition on actions, so the same scorer applies to any trajectory regardless of which policy produced it, whereas $Q$'s scores are meaningful only on the support of the behavior policy's actions.

\paragraph{Segment length.} Rerunning the entire pipeline at $H \in \{10, 50\}$ on the nine analysis tasks (re-segmentation, preference regeneration, value ensembles, calibrated filter, three seeds each, against the main arm's first three seeds) moves the deployed filter's safe count only at the shortest length: 8 of 9 at $H = 50$ and at $H = 30$, 7 of 9 at $H = 10$. Swimmer fails at every length and Hopper only at $H = 10$. Certification is the more sensitive quantity, but three seeds estimate it far too coarsely to compare arms, so we resample $2000$ calibration draws per task and seed and report rates. Mean certification rate is $0.204$ at $H = 50$, $0.224$ at $H = 30$ and $0.102$ at $H = 10$, tracking the mean purity margin over the same arms ($-0.002$, $+0.003$, $-0.072$) rather than the segment length itself: the two longer arms leave almost the same purity attainable and certify at almost the same rate, while the shortest leaves materially less and certifies at half. It is worth stating that the realized three-seed counts suggest a sharper collapse than there is: PointGoal1 happened to certify on none of its three $H = 10$ seeds despite a true rate of $0.69$, which is the kind of artifact a seed count produces and a resampled rate does not. Across all $27$ task-length cells the margin predicts the rate almost exactly (Spearman $0.994$, $p < 10^{-24}$), an independent replication of the law of Section~\ref{sec:guarantee-validation} under a design change rather than across tasks: cells with negative margin certify at mean rate $0.015$ and never above $0.057$, while positive-margin cells average $0.327$. The guarantee holds in every arm, with worst unconditional false-certification $0.044$, $0.061$, and $0.054$ against $\delta = 0.1$. Segment length is therefore not a load-bearing choice for whether the resulting policy is safe, and it moves certification only through the purity it leaves attainable.

\begin{figure}[h]
\begin{center}
\includegraphics[width=0.98\linewidth]{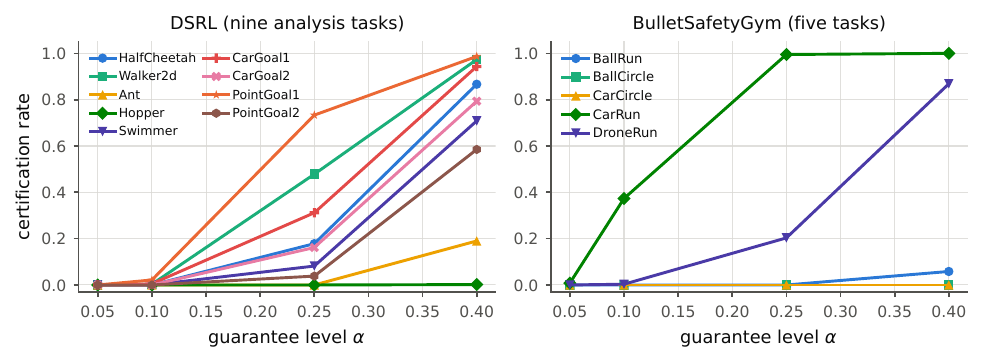}
\end{center}
\caption{The certification operating curve on both suites. Certification rate against the guarantee level $\alpha$ ($n = 200$; resampled calibration draws; mean over V-ensemble seeds). Rates rise with $\alpha$ in the order of each pool's attainable purity, the margin of Proposition~\ref{prop:rate}; where a pool cannot meet a level, the rate there stays at or near zero and within the $\delta$ the guarantee prices.}
\label{fig:alphacurve}
\end{figure}

\subsection{What the supervision buys}

\paragraph{Preference noise and budget.} Flipping preference labels degrades the score and never breaks the certificate. Across noise levels up to 30 percent and budgets down to 100 pairs the worst false-certification rate never exceeds $\delta$, reaching it exactly once at 30 percent noise on PointGoal1, because validity flows from the clean calibration sample rather than from score quality. Precision itself does fall: at 30 percent noise the navigation tasks lose $0.05$ to $0.14$ and Swimmer loses $0.38$, so the score degrades while the guarantee does not, which is the separation the design is for. Halving the preference budget to 100 pairs costs at most $0.12$ of precision (Hopper) and nothing measurable on several tasks. One effect deserves note: on Walker2d, mild label noise \emph{improves} the score (precision $0.834$ to $0.855$, certification rate $0.46$ to $0.77$ at five percent noise), acting as a regularizer on precisely the score tail whose pathology the threshold grid works around. A Boltzmann labeler with temperature $T_{\mathrm{lab}}$ (distinct from the horizon $T$), which errs preferentially on borderline pairs ($P(\text{prefer safer}) = \mathrm{logistic}(\text{cost gap}/T_{\mathrm{lab}})$, $T_{\mathrm{lab}} = 3$ on every task, realized error rates $5.3$ to $14.6$ percent), leaves navigation-task precision unchanged within $\pm 0.02$ and costs the velocity tasks more, at worst $-0.33$ on Swimmer, with false certification still inside $\delta$ at $0.08$: structured, boundary-concentrated labeler error behaves no worse for the certificate than uniform flips.

\paragraph{Labels-only filtering.} This control puts the calibration sample to that alternative use, training a supervised state-level safety classifier on only the $n$ labeled trajectories, using the same architecture, trajectory scoring rule, and matched selection fractions as the preference route, and clones its selection. Across label budgets $n \in \{50, 100, 200, 400\}$ it is safe on 11, 12, 11, and 13 of fifteen tasks, competitive on average, and evidence that filter-then-clone, not the preference machinery alone, carries much of the safety on this benchmark. Two properties separate the routes. First, which tasks come out safe under labels-only depends on the label draw. Its safe-task set overlaps across independent draws at a mean Jaccard of 0.85, dropping to 0.61 at $n = 100$, with Swimmer's cost ranging from 21 to 174 across budgets, Hopper safe only at $n = 400$, and PointButton2 and PointCircle2 alternating without pattern. The variation is attributable to the draw rather than the learner, retrained with three initializations on one fixed draw, the labels-only safe-task set overlaps at 0.95, above the preference route's 0.85 across value-ensemble seeds, so safe-set stability is not an advantage of the preference route. What the draw sensitivity does mark is that the labels-only deliverable depends on which $n$ trajectories were labeled, and a practitioner holding one draw cannot tell which tasks failed; the preference route's score does not consume the labels at all. Second, at every budget the labels-only route consumes its labels in training, so no valid certificate or refusal can be issued from the same sample; matching our deliverable would require a second labeled sample of equal size. The preference route separates the two supervision roles by construction. Preferences carry the scoring, and the labels remain held out to certify or refuse. Splitting the label budget between the two roles completes the comparison, with our exact total budget divided 100 for training and 100 for calibration, the labels-only route certifies at a rate comparable to ours (9 against 11 percent of seed-runs) and is equally safe, and at double the budget it certifies more often (24 percent) and again matches on safety. The certifications concentrate on disjoint task sets, ours on HalfCheetah, CarGoal1, and PointGoal1 and the split's on Walker2d and PointCircle1, so the two scorers are complementary rather than redundant. Each concentrates certifiable purity where the other's score has a weak tail. At matched label budget, preferences therefore buy one further safe task and no additional certification on this benchmark; their measured advantage is the graded selection on Hopper, and their structural advantage is that the budget judgments remain held out, so the sample that scores can also certify. Three kinds of judgment appear in this study: \emph{ordinal}, which of two clips is safer; \emph{binary}, whether one episode exceeded its budget; and \emph{cardinal}, the numeric cost of a transition. Labels-only and our calibration rest on the identical binary judgment, and the routes differ in how each spends it: labels-only consumes its labels in training, so a certificate needs a second sample or a split of the first, while the preference route carries scoring on ordinal comparisons and leaves every budget judgment held out. BC-Safe needs that same binary judgment on every trajectory rather than on 200; the offline safe-RL baselines go further and presume a cardinal cost on every transition. Which route is available is a property of the application, not of the budget. Appendix~\ref{app:labelsonly} reports the full sweep.

\paragraph{Cardinal supervision at matched budget.} The three judgments are ordered, and
the labels-only route occupies the middle rung, consuming the same binary over-budget
bit the calibration sample does and leaving the cardinal rung untested. This control
tests it. Everything is held fixed against the labels-only arm, the same script, the
same $200$ trajectories, the same per-task selection fraction, architecture, seeds,
clone and evaluation; only the target changes, from the bit to the standardised numeric
episodic cost, with a regression loss in place of the classification one. The registered
prediction, that magnitudes buy nothing at matched budget if safety is decodable from
the ordering alone, holds. Table~\ref{tab:cardinalctrl} reports both arms: the cardinal
selection is \emph{less} pure on eleven of the fifteen tasks, by $0.024$ on average,
and its clone is safe on twelve tasks against the bit route's eleven, a difference of
one task with flips in both directions and well inside the draw-to-draw variation of
Appendix~\ref{app:labelsonly}. Supervision richness is therefore not the binding
choice either, which is what makes the readout scale of Section~\ref{sec:analysis} the
one that is.

\begin{table}[h]
\caption{Binary against cardinal supervision at matched budget: the same filter trained
on the over-budget bit and on the numeric episodic cost of the same $200$ trajectories
(three seeds; kept-unsafe is the true unsafe fraction of the selection, lower better;
bold marks clone cost within budget).}
\label{tab:cardinalctrl}
\begin{center}
\scriptsize
\begin{tabular}{l rr rr r}
\toprule
& \multicolumn{2}{c}{binary (bit)} & \multicolumn{2}{c}{cardinal (cost)} & \\
\cmidrule(lr){2-3}\cmidrule(lr){4-5}
Task & kept-unsafe & clone $C$ & kept-unsafe & clone $C$ & budget \\
\midrule
HalfCheetah & 0.190 & \textbf{6.6} & 0.238 & 31.0 & 20 \\
Walker2d & 0.123 & \textbf{0.5} & 0.119 & \textbf{0.0} & 20 \\
Ant & 0.358 & \textbf{1.3} & 0.348 & \textbf{1.9} & 20 \\
Hopper & 0.360 & 49.5 & 0.398 & \textbf{7.4} & 20 \\
Swimmer & 0.150 & 21.0 & 0.224 & \textbf{19.6} & 20 \\
CarGoal1 & 0.288 & \textbf{7.3} & 0.281 & \textbf{7.5} & 25 \\
CarGoal2 & 0.325 & \textbf{18.7} & 0.341 & \textbf{18.7} & 25 \\
PointGoal1 & 0.195 & \textbf{10.9} & 0.197 & \textbf{9.9} & 25 \\
PointGoal2 & 0.358 & \textbf{18.8} & 0.356 & \textbf{17.6} & 25 \\
PointButton1 & 0.360 & \textbf{16.0} & 0.375 & \textbf{19.0} & 25 \\
PointButton2 & 0.430 & 27.1 & 0.463 & 30.1 & 25 \\
CarButton1 & 0.422 & \textbf{16.9} & 0.461 & \textbf{20.6} & 25 \\
CarButton2 & 0.409 & \textbf{19.3} & 0.414 & \textbf{22.8} & 25 \\
PointCircle1 & 0.137 & \textbf{1.2} & 0.162 & \textbf{2.8} & 25 \\
PointCircle2 & 0.208 & 64.5 & 0.299 & 48.7 & 25 \\
 
\bottomrule
\end{tabular}
\end{center}
\end{table}

\paragraph{Why labels-only works.} An episodic label is not one bit. It weakly labels every state of a thousand-step trajectory, so 200 trajectory labels yield roughly 200{,}000 labeled states. A subsampling test shows this multiplication is not the mechanism. Shrinking the states seen per labeled trajectory a hundredfold, to roughly 2{,}000 states in total, degrades selection purity only from 0.29 to 0.35 kept-unsafe and leaves the safe-task count essentially unchanged (13 and 10 of fifteen at tenfold and hundredfold reduction, against 11 at full). Labels-only is strong because trajectory-level safety in these environments is carried by low-dimensional, strongly separable state features, the same structure the saliency and landscape analyses of Section~\ref{sec:analysis} identify. This answers the paper's opening question in its sharpest form. Safety in these tasks is decodable from state alone, easily enough that the discriminating questions are not whether a scorer can be learned but whether its selections can be certified and how little supervision, of what kind, the certification requires.

\section{Labels-only control across label budgets}
\label{app:labelsonly}
Table~\ref{tab:labelsonly} reports the labels-only filtering control of Section~\ref{sec:ablations} at every label budget (three seeds per cell; the classifier, scoring rule, and matched selection fractions are identical to the preference route). Safe-task counts are stable in aggregate while the membership of the safe set shifts across budgets, which is the instability the main text describes. In the split variants the certified population is the dataset minus the classifier's training trajectories, over which the calibration half is exactly a uniform without-replacement sample, so Proposition~\ref{prop:guarantee} applies unchanged; training trajectories that are individually labeled safe are then added to the cloned selection, which can only lower its unsafe fraction below the certified level.

\begin{table}[h]
\caption{Labels-only filtering control across calibration budgets $n$ (three seeds, mean $\pm$ 95\% CI; bold marks mean cost within budget).}
\label{tab:labelsonly}
\begin{center}
\scriptsize
\setlength{\tabcolsep}{2.5pt}
\resizebox{\textwidth}{!}{%
\begin{tabular}{l rr rr rr rr}
\toprule
& \multicolumn{2}{c}{$n = 50$} & \multicolumn{2}{c}{$n = 100$} & \multicolumn{2}{c}{$n = 200$} & \multicolumn{2}{c}{$n = 400$} \\
\cmidrule(lr){2-3}\cmidrule(lr){4-5}\cmidrule(lr){6-7}\cmidrule(lr){8-9}
Task & $R$ & $C$ & $R$ & $C$ & $R$ & $C$ & $R$ & $C$ \\
\midrule
HalfCheetah & 2362\,\scriptsize$\pm$107 & \textbf{7\,\scriptsize$\pm$10} & 2158\,\scriptsize$\pm$194 & \textbf{0\,\scriptsize$\pm$0} & 2337\,\scriptsize$\pm$81 & \textbf{7\,\scriptsize$\pm$9} & 2299\,\scriptsize$\pm$46 & \textbf{2\,\scriptsize$\pm$2} \\
Walker2d & 2722\,\scriptsize$\pm$17 & \textbf{4\,\scriptsize$\pm$2} & 2694\,\scriptsize$\pm$26 & \textbf{0\,\scriptsize$\pm$0} & 2689\,\scriptsize$\pm$16 & \textbf{1\,\scriptsize$\pm$1} & 2702\,\scriptsize$\pm$9 & \textbf{0\,\scriptsize$\pm$0} \\
Ant & 2535\,\scriptsize$\pm$64 & \textbf{1\,\scriptsize$\pm$0} & 2498\,\scriptsize$\pm$25 & \textbf{1\,\scriptsize$\pm$0} & 2585\,\scriptsize$\pm$55 & \textbf{1\,\scriptsize$\pm$0} & 2577\,\scriptsize$\pm$108 & \textbf{1\,\scriptsize$\pm$1} \\
Hopper & 787\,\scriptsize$\pm$519 & 30\,\scriptsize$\pm$36 & 611\,\scriptsize$\pm$654 & 27\,\scriptsize$\pm$25 & 841\,\scriptsize$\pm$175 & 50\,\scriptsize$\pm$43 & 686\,\scriptsize$\pm$88 & \textbf{11\,\scriptsize$\pm$5} \\
Swimmer & 105\,\scriptsize$\pm$6 & 28\,\scriptsize$\pm$23 & 137\,\scriptsize$\pm$26 & 174\,\scriptsize$\pm$137 & 120\,\scriptsize$\pm$13 & 21\,\scriptsize$\pm$23 & 129\,\scriptsize$\pm$12 & 82\,\scriptsize$\pm$108 \\
CarGoal1 & 11.09\,\scriptsize$\pm$0.57 & \textbf{12.20\,\scriptsize$\pm$2.96} & 10.65\,\scriptsize$\pm$0.42 & \textbf{8.78\,\scriptsize$\pm$1.88} & 9.55\,\scriptsize$\pm$1.83 & \textbf{7.30\,\scriptsize$\pm$2.36} & 10.57\,\scriptsize$\pm$0.36 & \textbf{8.57\,\scriptsize$\pm$0.96} \\
CarGoal2 & 4.51\,\scriptsize$\pm$0.72 & \textbf{22.55\,\scriptsize$\pm$6.01} & 4.44\,\scriptsize$\pm$0.83 & \textbf{18.47\,\scriptsize$\pm$3.68} & 4.35\,\scriptsize$\pm$0.51 & \textbf{18.70\,\scriptsize$\pm$0.19} & 4.29\,\scriptsize$\pm$0.30 & \textbf{17.46\,\scriptsize$\pm$2.96} \\
PointGoal1 & 11.86\,\scriptsize$\pm$1.97 & \textbf{17.84\,\scriptsize$\pm$5.62} & 9.28\,\scriptsize$\pm$0.80 & \textbf{12.32\,\scriptsize$\pm$2.00} & 8.96\,\scriptsize$\pm$1.23 & \textbf{10.86\,\scriptsize$\pm$2.06} & 7.65\,\scriptsize$\pm$0.36 & \textbf{8.18\,\scriptsize$\pm$0.65} \\
PointGoal2 & 6.76\,\scriptsize$\pm$0.48 & \textbf{20.88\,\scriptsize$\pm$2.56} & 6.68\,\scriptsize$\pm$1.07 & \textbf{23.81\,\scriptsize$\pm$1.45} & 6.27\,\scriptsize$\pm$0.62 & \textbf{18.82\,\scriptsize$\pm$0.13} & 6.10\,\scriptsize$\pm$0.72 & \textbf{17.76\,\scriptsize$\pm$1.67} \\
PointButton1 & 1.69\,\scriptsize$\pm$1.23 & \textbf{19.33\,\scriptsize$\pm$4.79} & 2.01\,\scriptsize$\pm$1.45 & 27.79\,\scriptsize$\pm$5.82 & 0.77\,\scriptsize$\pm$0.83 & \textbf{16.02\,\scriptsize$\pm$2.41} & 0.75\,\scriptsize$\pm$1.03 & \textbf{14.30\,\scriptsize$\pm$3.19} \\
PointButton2 & 4.87\,\scriptsize$\pm$0.76 & 30.93\,\scriptsize$\pm$2.02 & 4.21\,\scriptsize$\pm$0.31 & \textbf{24.51\,\scriptsize$\pm$3.13} & 4.29\,\scriptsize$\pm$0.52 & 27.15\,\scriptsize$\pm$3.11 & 4.36\,\scriptsize$\pm$0.53 & 27.61\,\scriptsize$\pm$2.12 \\
CarButton1 & 0.76\,\scriptsize$\pm$0.71 & \textbf{18.72\,\scriptsize$\pm$2.22} & 1.21\,\scriptsize$\pm$0.72 & \textbf{21.98\,\scriptsize$\pm$6.18} & 1.39\,\scriptsize$\pm$0.20 & \textbf{16.86\,\scriptsize$\pm$1.29} & 1.27\,\scriptsize$\pm$0.18 & \textbf{16.36\,\scriptsize$\pm$1.67} \\
CarButton2 & -0.88\,\scriptsize$\pm$0.79 & \textbf{16.49\,\scriptsize$\pm$4.56} & -1.31\,\scriptsize$\pm$0.55 & \textbf{20.34\,\scriptsize$\pm$5.00} & -1.68\,\scriptsize$\pm$1.04 & \textbf{19.28\,\scriptsize$\pm$4.92} & -1.14\,\scriptsize$\pm$0.35 & \textbf{21.06\,\scriptsize$\pm$5.46} \\
PointCircle1 & 33.9\,\scriptsize$\pm$0.9 & \textbf{2.3\,\scriptsize$\pm$1.2} & 35.0\,\scriptsize$\pm$3.3 & \textbf{5.6\,\scriptsize$\pm$4.0} & 34.8\,\scriptsize$\pm$3.8 & \textbf{1.2\,\scriptsize$\pm$0.5} & 33.6\,\scriptsize$\pm$2.2 & \textbf{1.9\,\scriptsize$\pm$1.0} \\
PointCircle2 & 35.3\,\scriptsize$\pm$0.6 & 41.3\,\scriptsize$\pm$27.9 & 33.0\,\scriptsize$\pm$1.8 & \textbf{21.5\,\scriptsize$\pm$17.5} & 36.1\,\scriptsize$\pm$0.9 & 64.5\,\scriptsize$\pm$27.4 & 32.6\,\scriptsize$\pm$2.7 & \textbf{12.8\,\scriptsize$\pm$10.8} \\
 
\bottomrule
\end{tabular}}
\end{center}
\end{table}

\section{Hyperparameters}
\label{app:hyper}
Value ensembles: $K = 3$, two-layer MLPs of width 256, 300 epochs, batch 512, Adam $3 \times 10^{-4}$, 1000 preference pairs of length 30 per task. Preference pairs are sampled from parent trajectories in the bottom and top quartiles of episodic cost, and labeled by the two segments' own summed costs, ties skipped. Behavior cloning: 100 epochs, batch 512. CPL baseline and CPL composability arms: behavior-cloning pretraining for 50 epochs then 300 epochs of CPL, batch 512, Adam $3 \times 10^{-4}$, BC regularization weight $0.1$, temperature $1$, conservative bias $0.5$, on the same 1000 preference pairs per task the value ensemble consumes. On a certified selection the same budget of 1000 pairs is redrawn from within the selection, since pairs are sampled to contrast episodic cost and almost none of the pool's pairs have both endpoints inside a selection. Calibration: $\alpha = 0.25$, $\delta = 0.1$, $n = 200$, threshold grid at score quantiles $\{0.85, 0.80, \dots, 0.30\}$, uncertified fallback fraction from the one-sided Clopper-Pearson lower bound (level $\delta$) on pool safe mass, minimum selection 50 trajectories. The grid starts at $0.85$ rather than deeper in the tail because above it too few calibration trajectories clear the threshold for the exact test to reject at $\delta$. All $200$ labels are spent before the threshold is known, and the uncertified fallback uses every one, but the accepted test rests on the $m$ that land above the chosen threshold: $m = 28$, $42$, $31$ and $28$ on the deployed certified draws of HalfCheetah, Walker2d, CarGoal1 and PointGoal1, giving $p = 0.014$, $0.027$, $0.068$ and $0.044$ against $\delta = 0.1$. Walker2d's walk is the one that does not stop at the first grid threshold, accepting at the third. Reward-aware sub-selection: top 50 percent by episodic return, applied only under certification. Certified return-weighted cloning, meaning trajectory resampling with probability proportional to $\exp(\mathrm{clip}(z(R), \pm \kappa))$, return standardized within the selection, selection size preserved, $\kappa \in \{1, 2, 3\}$ reported with $\kappa = 2$ in Table~\ref{tab:compose}, applied only under certification. Advantage-weighted extraction (analysis section): weights $\exp(A/\beta)$ with $\beta = 0.1$, weight clip 20, 50 epochs at batch 1024, every extraction arm starting from the same BC-All policy trained on the active segments (Table~\ref{tab:main}'s BC-All row instead clones whole trajectories at the protocol above), advantage window $h = 1$ unless stated. Evaluation: 100 episodes, deterministic policy mean. Sensitivity: averaged over its members, a single network matches the $K = 3$ filter precision within $0.006$ on all nine analysis tasks, though individual members vary by up to $0.075$, so the ensemble size is not load-bearing for selection and serves the aggregation diagnostics; a segment-length sweep appears in Appendix~\ref{app:extended}.

\end{document}